\documentclass[10pt,twocolumn,letterpaper]{article}

\usepackage{wacv}
\usepackage{times}
\usepackage{epsfig}
\usepackage{graphicx}
\usepackage{amsmath}
\usepackage{amssymb}
\usepackage{booktabs}

\usepackage[accsupp]{axessibility}  

\usepackage{subcaption,amsfonts,dcolumn}

\usepackage[square,sort,comma,numbers]{natbib}
\usepackage{url}

\usepackage{bbold}
\usepackage{microtype}
\usepackage{paralist}
\usepackage{amsthm}
\usepackage{arydshln}
\usepackage{pifont}
\newcommand{\cmark}{\ding{51}}%
\newcommand{\xmark}{\ding{55}}%

\usepackage[skip=4pt]{caption}
\usepackage{wrapfig}

\def\prob{\mathbf{prob}}
\renewcommand\prob[1]{\mathbf{prob}(\!\!(#1)\!\!)}

\newcommand{\Mleft}{\mathrm{M}^0}
\newcommand{\Mright}{\mathrm{M}^1}

\newcommand\tempDeleted[1]{\textcolor{red}{[temporally deleted]}}

\newcommand{\our}{SONG}
\newcommand{\sbdg}{SBDG}

\DeclareMathOperator*{\argmax}{arg\,max}

\def\R{\mathbb{R}}

\def\G{\mathcal{G}}
\def\V{\mathcal{V}}
\def\E{\mathcal{E}}
\def\MM{\mathrm{M}}

\def\acc{\mathrm{acc}}

\newtheorem{theorem}{Theorem}[section]

\def\wacvPaperID{235} 

\wacvalgorithmstrack   

\wacvfinalcopy 

\ifwacvfinal
\usepackage[breaklinks=true,bookmarks=false]{hyperref}
\else
\usepackage[pagebackref=true,breaklinks=true,colorlinks,bookmarks=false]{hyperref}
\fi

\begin{document}

\title{\our{}s: Self-Organizing Neural Graphs}

\author{\L{}ukasz Struski$^1$
\and 
Tomasz Danel$^1$
\and
Marek \'Smieja$^1$
\and
Jacek Tabor$^1$
\and
Bartosz Zieli\'nski$^{1, 2}$
\\
\\
$^1$ Faculty of Mathematics and Computer Science, Jagiellonian University, Krak\'ow, Poland\\
$^2$ IDEAS NCBR, Warsaw, Poland \\
{\tt\small \{lukasz.struski;marek.smieja;jacek.tabor;bartosz.zielinski\}@uj.edu.pl}\\
{\tt\small \{tomasz.danel\}@doctoral.uj.edu.pl}
}


\maketitle
\thispagestyle{empty}

\begin{abstract}
Recent years have seen a surge in research on combining deep neural networks with other methods, including decision trees and graphs. There are at least three advantages of incorporating decision trees and graphs: they are easy to interpret since they are based on sequential decisions, they can make decisions faster, and they provide a hierarchy of classes. However, one of the well-known drawbacks of decision trees, as compared to decision graphs, is that decision trees cannot reuse the decision nodes. Nevertheless, decision graphs were not commonly used in deep learning due to the lack of efficient gradient-based training techniques. In this paper, we fill this gap and provide a general paradigm based on Markov processes, which allows for efficient training of the special type of decision graphs, which we call Self-Organizing Neural Graphs (SONG). We provide a theoretical study on SONG, complemented by experiments conducted on Letter, Connect4, MNIST, CIFAR, and TinyImageNet datasets, showing that our method performs on par or better than existing decision models.
\end{abstract}

\begin{figure*}[bth]
    \centering
    \begin{subfigure}[b]{.27\textwidth}
        \centering
        \hspace{-6mm}
        \includegraphics[width=0.8\textwidth]{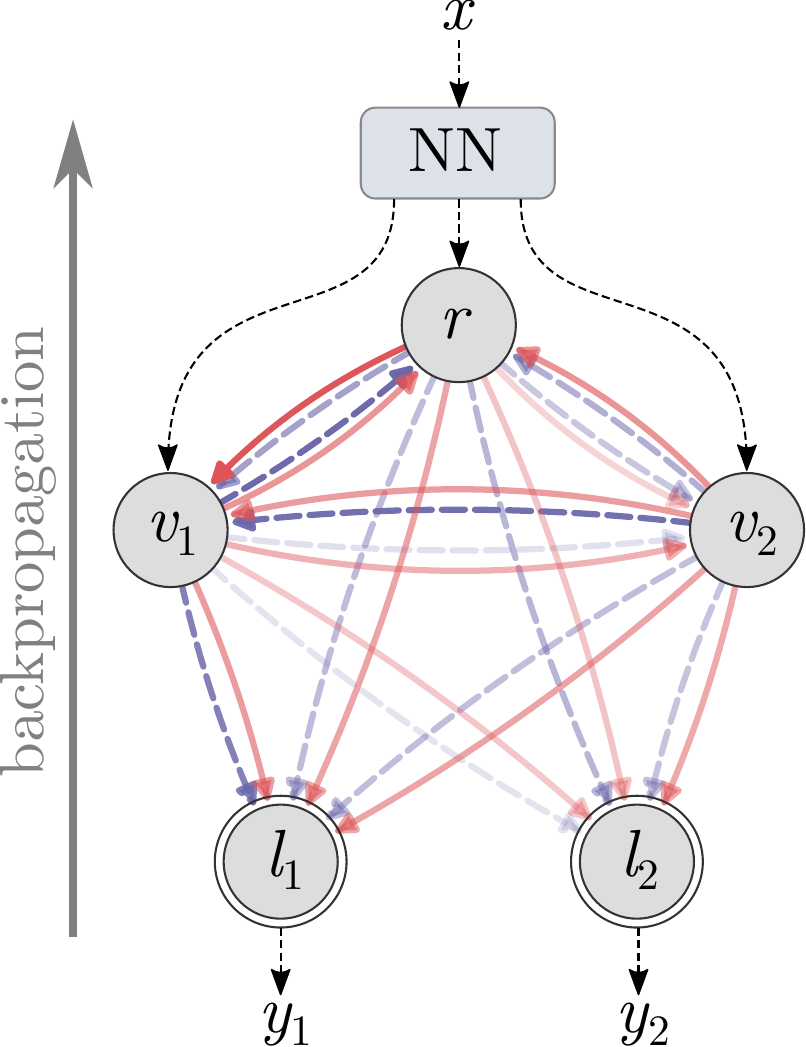}
        \caption{random initialization}
        \label{fig:abstract1}
    \end{subfigure}
    \hfill
    \begin{subfigure}[b]{.27\textwidth}
        \centering
       \hspace{-6mm}
       \includegraphics[width=0.8\textwidth]{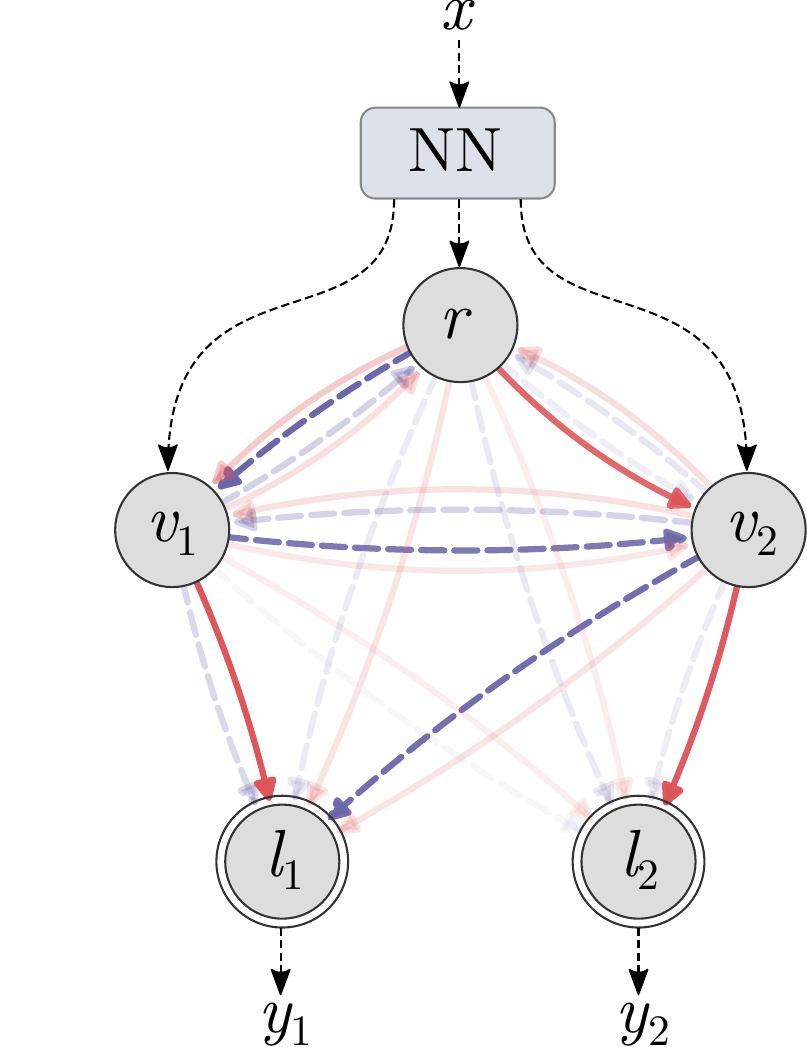}
        \caption{intermediate phase}
        \label{fig:abstract2}
    \end{subfigure}
    \hfill
    \begin{subfigure}[b]{.27\textwidth}
        \centering
       \hspace{-6mm}
       \includegraphics[width=0.8\textwidth]{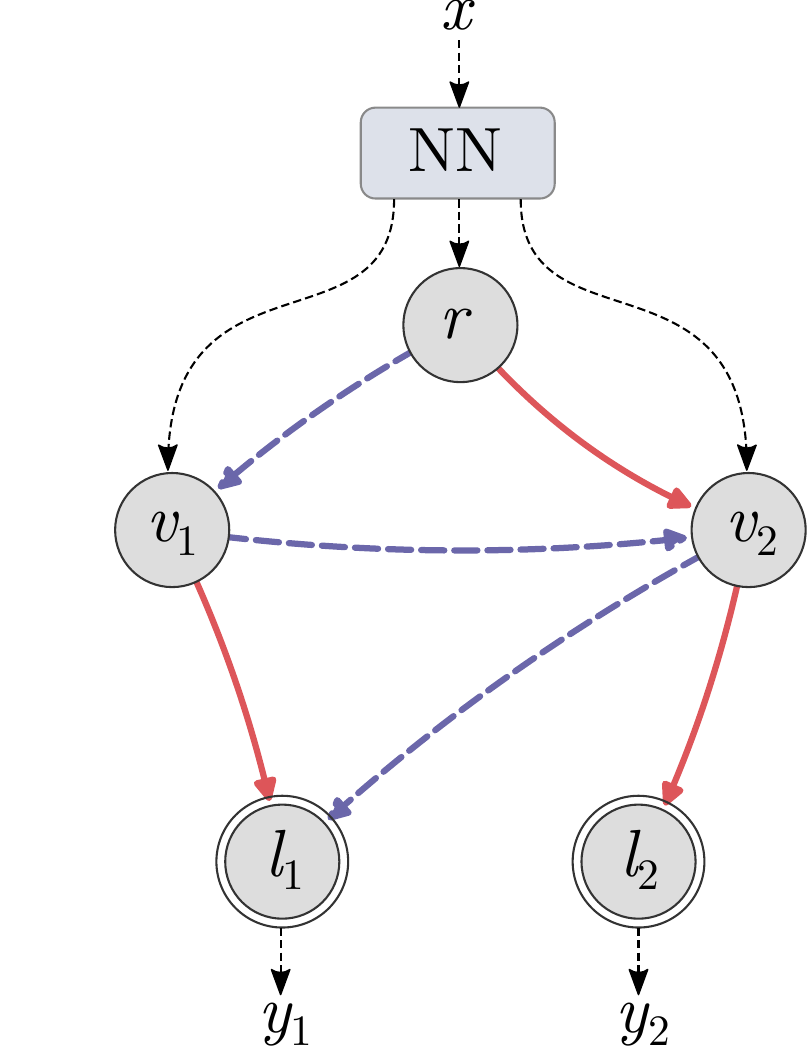}
        \caption{trained SONG}
        \label{fig:abstract3}
    \end{subfigure}
    \caption{Training stages of SONG that uses gradient descent to modify the graph structure and transition probabilities. Based on an input $x$, the backbone neural network (NN) extracts a vector representation, which is passed to SONG to obtain a prediction for each class ($y_1$ and $y_2$). At the beginning of training, a graph has root $r$, nodes $v_1$ and $v_2$, leaves $l_1$ and $l_2$, and randomly initialized edges (a). In the successive training iterations, the entropy of edge weights grows (b), finally resulting in a sparse binary graph, with two strong edges outgoing from each node (c). Notice that SONG contains two alternative sets of edges between the nodes (dashed blue arrows and solid red arrows, respectively) that are combined based on the input (see Figure~\ref{fig:transitions} for details).}
    \label{fig:abstract}
\end{figure*}

\section{Introduction}

Neural networks (NNs) and decision trees (DTs) are two exceptionally powerful machine learning models with a rich and successful history in machine learning.
However, they typically come with mutually exclusive benefits and limitations. NNs outperform conventional pipelines by jointly learning to represent and classify data~\cite{krizhevsky2012imagenet}. However, they are widely opaque and suffer from a lack of transparency and explainability~\cite{rudin2019stop}. On the other hand, it is easy to explain predictions of DTs because they depend on a relatively short sequence of decisions~\cite{wan2020nbdt}. However, they usually do not generalize as well as deep neural networks~\cite{frosst2017distilling}. As a result, a strong focus is recently put on joining the positive aspects of both models~\cite{frosst2017distilling,suarez1999globally,kontschieder2015deep,nauta2020neural,murthy2016deep,tanno2019adaptive,alaniz2021learning,wan2020nbdt}.
There are methods that combine NNs and soft decision trees with partial membership in each node~\cite{frosst2017distilling,suarez1999globally,kontschieder2015deep,nauta2020neural}. Others use trees to explain NNs~\cite{zhang2019interpreting,bastani2017interpreting} or to obtain their optimal hierarchical structure~\cite{murthy2016deep,tanno2019adaptive,alaniz2021learning}. Finally, some models replace the final softmax layer of a neural network with a hierarchical binary decision tree~\cite{wan2020nbdt,morin2005hierarchical,mikolov2011extensions}.

While decision trees can increase the performance and interpretability of NNs, they usually suffer from exponential growth with depth~\cite{shotton2012efficient}, repeating nodes~\cite{frosst2017distilling}, and suboptimal structure, often selected manually before training~\cite{wan2020nbdt}. Hence, more and more attention is put on combining NNs with decision graphs instead of trees~\cite{baek2017deep,veit2018convolutional,he2016deep,mullapudi2018hydranets,keskin2018splinenets}. Decision graphs have a few advantages when compared to decision trees. They have a flexible structure that allows multiple paths from the root to each leaf. As a result, nodes are reused, resulting in simpler and smaller models, which solves the replication problem~\cite{oliver1993decision}. Moreover, decision graphs require substantially less memory while considerably improving generalization~\cite{shotton2016decision}. Nevertheless, decision graphs are not commonly used in deep learning due to a lack of efficient gradient-based training techniques.

In this paper, we introduce Self-Organizing Neural Graphs (SONGs), a special type of decision graphs that generalize methods like Soft Decision Tree (SDT)~\cite{frosst2017distilling} and Neural-Backed Decision Trees (NBDT)~\cite{wan2020nbdt}, and as a differentiable solution are applicable to any deep learning pipeline. Moreover, in contrast to the fixed structure of the existing methods~\cite{wan2020nbdt,frosst2017distilling}, SONGs can strengthen or weaken an edge between any pair of nodes during training to optimize their structure. We illustrate this process in Figure~\ref{fig:abstract}. In the beginning, the edges have random weights. However, in successive steps of training, the structure is corrected with backpropagation, and it gets sparse and converges to the binary directed acyclic graphs~\cite{platt1999large}.


Our contributions can be summarized as follows:
\begin{compactitem}
    \item We introduce Self-Organizing Neural Graphs (SONGs), a new paradigm of end-to-end training based on Markov processes that simultaneously learn the optimal graph structure and transition probabilities.
    \item Our model is fully differentiable and thus suitable for combined training with other deep learning models.
    \item We prove empirically and theoretically that SONGs during training converge to sparse binary acyclic graphs.
    \item Our method performs on par or outperforms decision trees trained in a similar setup and does not require the graph/tree structure to be predefined before training.
\end{compactitem}

\section{Related works}

\subsection{Decision trees}
Numerous Decision Tree (DT) algorithms have been developed over the years~\cite{quinlan2014c4,loh2011classification,shafer1996sprint,mehta1996sliq} and after the success of deep learning, much research relates to combining DTs with neural networks. As a result, Soft Decision Tree (SDT) was introduced, allowing for the partial membership of a sample in the nodes that make up the tree structure~\cite{suarez1999globally}, also trained in distillation setup~\cite{frosst2017distilling}. This idea was also used in~\cite{kontschieder2015deep} that trains a set of classification trees and a backbone network in an end-to-end manner. Moreover, it was recently used in~\cite{nauta2020neural} to faithfully visualize the model using nodes with prototypes~\cite{chen2018looks} instead of classifiers. Trees were also used to explain the previously trained black box models~\cite{zhang2019interpreting,bastani2017interpreting}.
More advanced methods automatically generate deep networks with a tree structure in a multi-step or an end-to-end manner~\cite{ murthy2016deep,alaniz2021learning,wan2020nbdt,ahmed2016network}. Many previous works were generalized by the ANT framework \cite{tanno2019adaptive}, which additionally enriched the tree structure with transformer, a nonlinear function that maps samples from the previous module to the next one, and allows for training the tree topology.
In contrast to the these methods, we design a simple and effective model for training decision graphs together with the parameters of the base neural network in the end-to-end manner. 

\subsection{Decision graphs} A decision graph is a well-studied classifier and has been used to solve many real-world problems~\cite{sudo2018efficient}. When implemented as Directed Acyclic Graphs (DAG), it leads to accurate predictions while having lower model complexity, subtree replication, and training data fragmentation compared to decision trees~\cite{shotton2016decision}. However, most of the existing algorithms for learning DAGs involve training a conventional tree that is later manipulated into a DAG~\cite{kohavi1995oblivious,oliveira1996using,oliver1992decision,chou1991optimal} and, as such, are difficult to be directly adopted into neural networks. Hence, alternative approaches were proposed, like~\cite{baek2017deep}, which maintains the structure of the standard convolutional neural networks (CNNs) but uses additional routing losses at each layer to maximize the class-wise purity (like in growing decision trees) using data activation according to the class label distribution. Another method~\cite{veit2018convolutional} introduces identity skip-connections similar to ResNets~\cite{he2016deep} that are executed or skipped depending on the gate response for an input. A similar gate mechanism was used in~\cite{mullapudi2018hydranets} to choose branches specialized for different inputs, whose outputs are combined to make the final predictions. Finally, \cite{keskin2018splinenets} embeds infinitely many filters into low dimensional manifolds parameterized by compact B-splines and maximizes the mutual information between spline positions and class labels to specialize for classification tasks optimally. Such a mechanism significantly reduces runtime complexity. In contrast to existing methods, SONG is a directed graph that can be adapted to any deep architecture and trained in an efficient gradient-based manner.



\begin{figure*}
    \centering
    \begin{subfigure}[b]{.45\textwidth}
        \centering
        \includegraphics[height=5cm]{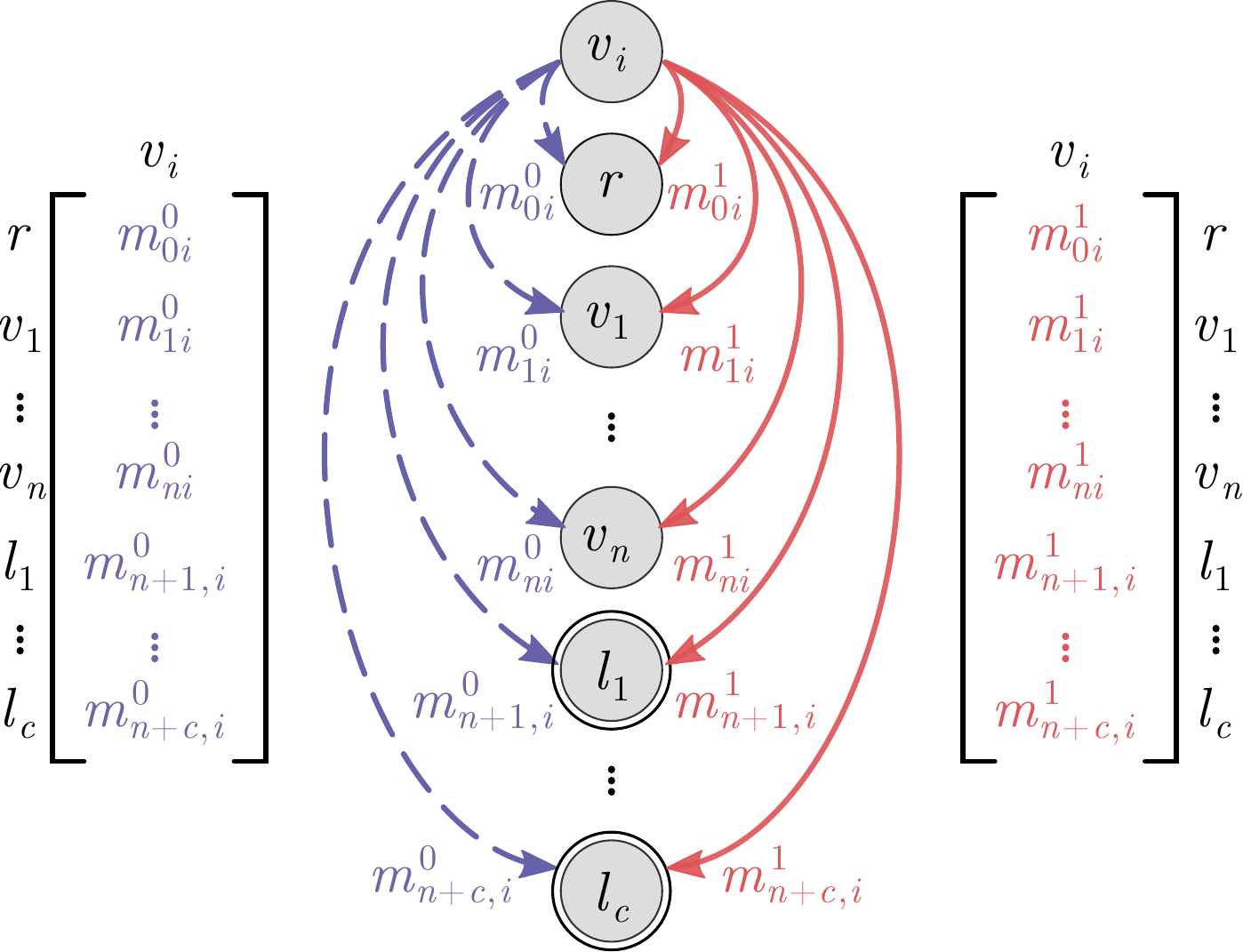}
        \caption{transition vectors}
        \label{fig:transition_vectors}
    \end{subfigure}
    \hfill
    \begin{subfigure}[b]{.4\textwidth}
        \centering
        \includegraphics[height=5cm]{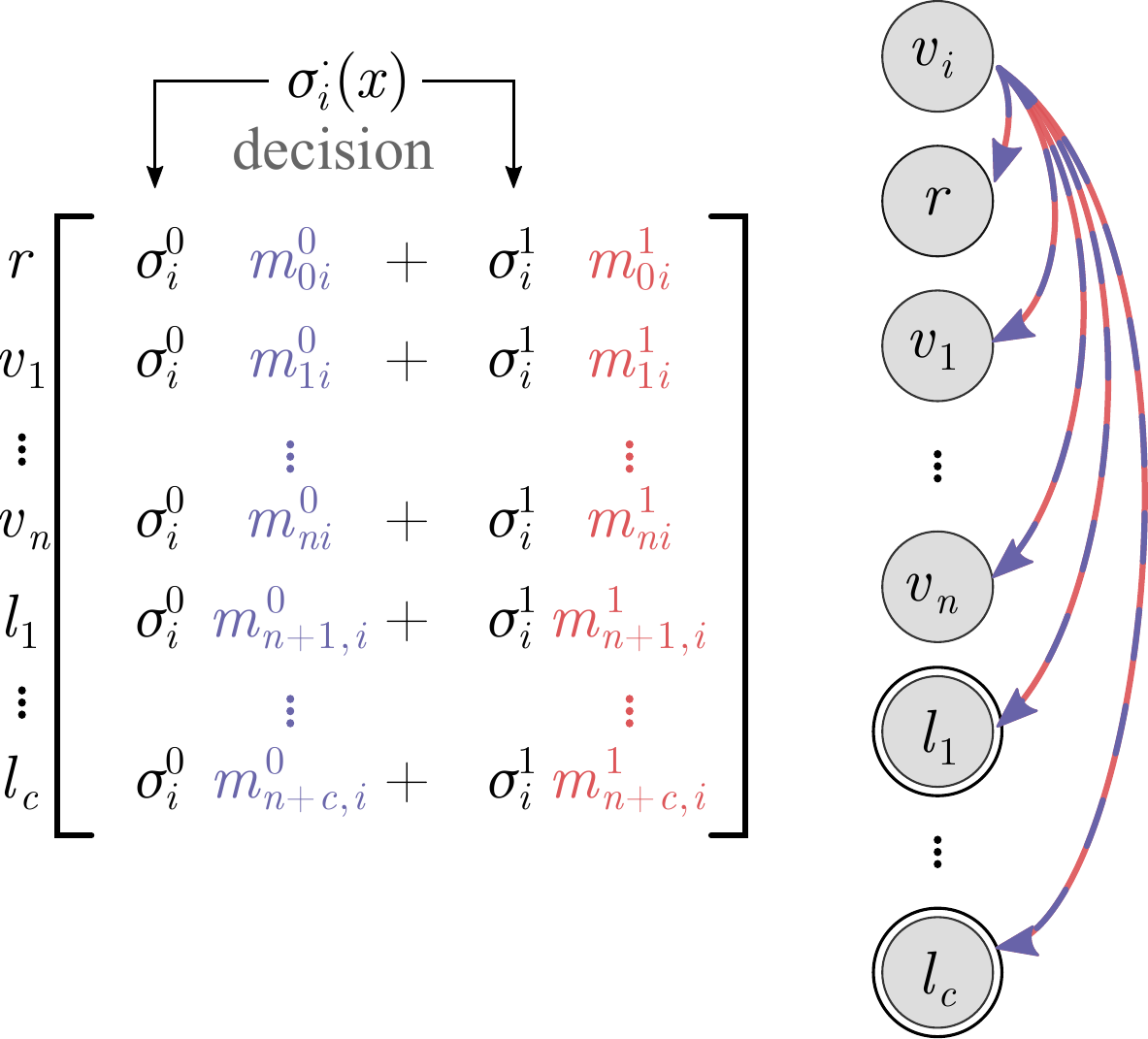}
        \caption{combination of transition vectors}
        \label{fig:node_rooting}
    \end{subfigure}
    \caption{SONG contains two alternative transition vectors $m_{\cdot i}^0$ and $m_{\cdot i}^1$ that aggregate the probability of moving from a particular node $v_i$ to all other nodes. In (a), they are represented as dashed blue and solid red arrows, respectively. Each node obtains input data $x$ and makes a binary decision with probabilities $\sigma_i^0$ and $\sigma_i^1$ of using one transition or another. As $\sigma_i^0 + \sigma_i^1 = 1$, SONG can be transformed to a standard directed graph by combining $m_{\cdot{}i}^0$ and $m_{\cdot{}i}^1$, as presented in (b). During training, both $\sigma_i^{\cdot}$ and $m_{\cdot i}^{\cdot}$ are trained to obtain the optimal decision graph as presented in Figure 1 of the paper.}
    \label{fig:transitions}
\end{figure*}

\section{Self-organizing neural graphs}
\label{sec:song}

To adequately describe the Self-Organizing Neural Graph (SONG), we first define a more abstract structure that we call Soft Binary Directed Graph (\sbdg{}). \sbdg{} is considered binary because there are two alternative sets of edges, and soft because those sets are combined into one target set of edges depending on the input. Then, based on \sbdg{}, we define SONG and describe how to use them as a decision model. Finally, we present method limitations and show how to overcome them with additional regularizers. The below definitions correspond to single-label classification for the clarity of description. However, they could be easily extended to other tasks, like multi-label classification or regression.

\subsection{Method}

\subsubsection{Soft binary directed graphs}
Soft Binary Directed Graph (\sbdg{}) is a directed graph, which can be viewed as a probabilistic model. It is defined as graph $G=(V,E^0,E^1)$, with $V$ corresponding to a set of nodes and $E^0$, $E^1$ corresponding to two alternative sets of edges, where:
\begin{compactitem}
  \item Set $V$ contains two types of nodes:
  \begin{compactitem}
    \item internal nodes $v_0,\dots,v_n$, with $v_0$ specified as root $r$,
    \item leaves $l_1,\dots,l_c$, each exclusively associated with one class from set $\{1,\dots,c\}$,
  \end{compactitem}
  \item Set $E^d$, for $d\in \{0,1\}$, contains all possible edges with weights $m_{ji}^d$ corresponding to the probability of moving from node $u_i$ to $u_j\in V$, as presented in Figure~\ref{fig:transition_vectors}. In the following, the aggregated probabilities of moving from node $u_i$ to other nodes will be called a transition vector and denoted as $m_{\cdot{}i}^d$.
  \item If $u_i$ is a leaf, then $m^d_{ji}=\delta_{ji}$ (Kronecker delta), which means that it is impossible to move out from the leaves.
  \item Each internal node $u_i$ makes binary decisions $d\in \{0, 1\}$ with probabilities $\sigma_i^d$ of using edges from set $E^d$.
  \item $\sigma_i^0 + \sigma_i^1 = 1$ and $G$ can be transformed to a standard directed graph by combining $m_{\cdot{}i}^0$ and $m_{\cdot{}i}^1$ using the following formula for each node $u_i$: $\sigma_i^0 m_{\cdot{}i}^0 + \sigma_i^1 m_{\cdot{}i}^1$. This process is presented in Figure~\ref{fig:node_rooting}.
\end{compactitem}
Notice that if all transition vectors are binary, then after removing the edges with zero probability, \sbdg{} becomes a binary directed graphs~\cite{platt1999large}.

\begin{figure*}
    \centering
    \includegraphics[width=.9\textwidth]{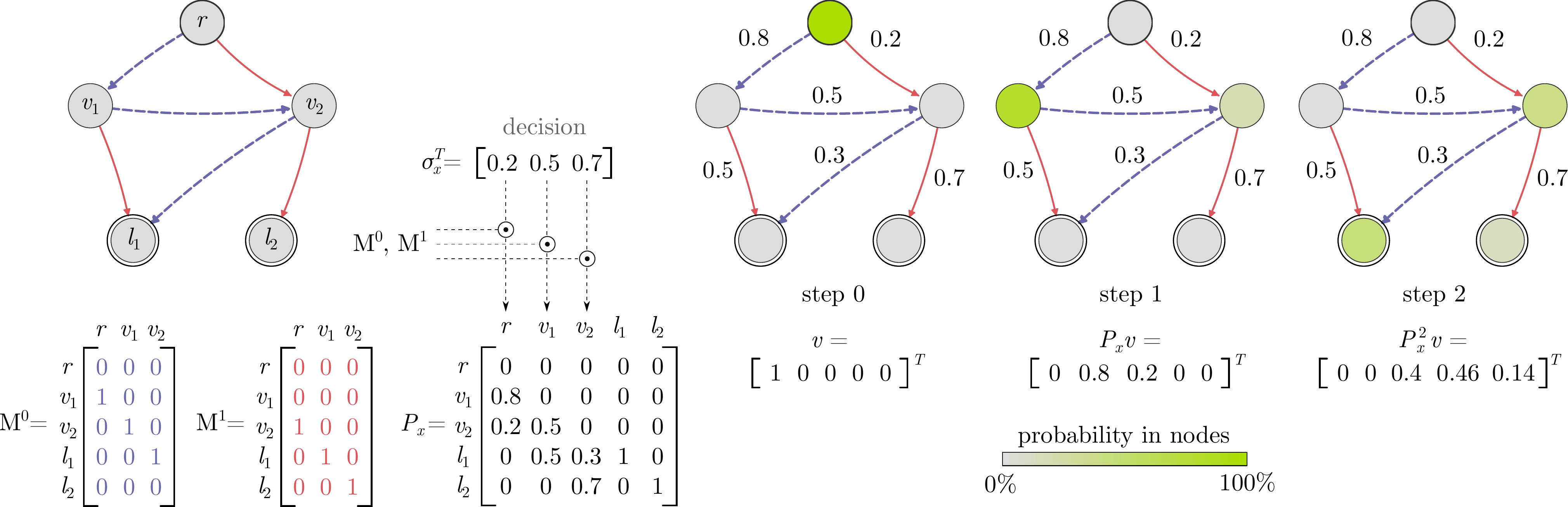}
    \caption{Construction of the transition matrix and successive steps of our Markov process. On the left, a graph with its matrices $\Mleft$ and $\Mright$ is presented, followed by an exemplary decision vector $\sigma_x$ and the resulting matrix $P_x$. On the right, the flow in a graph is depicted for 3 consecutive steps. At first, the probability is entirely placed in the root. However, in the next steps, the distribution splits between nodes according to the transition probabilities, reaching leaves in step 3. The probabilities in the leaves after all steps are class probabilities inferred by the model (the number of steps is considered as a method hyperparameter).}
    \label{fig:song}
\end{figure*}

\subsubsection{Self-organizing neural graphs}
 Self-Organizing Neural Graph (SONG) is a fully differentiable adaptation of \sbdg{} that can be combined with various deep architectures. SONG is defined as $\G=(\V,\E^0,\E^1)$, where $\V$, $\E^0$, $\E^1$ implement $V$, $E^0$, and $E^1$ of \sbdg{}, and are obtained for input point $x$ in the following way:
\begin{compactitem}
  \item The probability of decision $d=1$ in node $u_i$ is obtained as $\sigma_i^1(x)=\sigma(x w_i + b_i)$, where $\sigma$ is the sigmoid logistic function, $w_i$ is a filter function, and $b_i$ is a bias\footnote{In practice, this probability could also be obtained with any NN that ends with a sigmoid function.}.
  \item The probability of decision $d=0$ equals $\sigma_i^0(x) = 1 - \sigma_i^1(x)$.
  \item The probability of moving from internal nodes is defined by two matrices $\MM^d = [m^d_{ji}] \in\R^{(n+c)\times n}$, for $d=\{0, 1\}$, with positive values and columns summing up to $1$. In our implementation, we obtain such matrices by applying softmax to each of their columns.
\end{compactitem}
Notice that $\{w_i\}_{i=1,\ldots,n}$, $\{b_i\}_{i=1,\ldots,n}$, $\MM^0$, and $\MM^1$ are trainable parameters of the model.

Finally, we define a directed graph $\G_x = (\V, \E)$ generated for input $x$ where $\E$ corresponds to the combination of matrices $\MM^0$ and $\MM^1$:
\begin{equation}\label{eq:matrix_Mx}
\setlength{\abovedisplayskip}{5pt} \setlength{\abovedisplayshortskip}{0pt}
\setlength{\belowdisplayskip}{5pt} \setlength{\belowdisplayshortskip}{0pt}
\MM_x = \mathbb{1}\sigma_x^T\odot\Mright + \mathbb{1}(\mathbb{1}-\sigma_x)^T\odot\Mleft,
\end{equation}
where $\sigma_x = [\sigma^1_0(x),\ldots, \sigma^1_n(x)]^T$, symbol $\odot$ denotes the Hadamard product, and $\mathbb{1}$ is the all-ones vector of dimension $n$.

\subsubsection{Decision model}
Matrix $\MM_x$ contains the probability of moving from internal nodes to all nodes of the graph. However, to apply the theory of the Markov processes, it needs to be extended by columns corresponding to the leaves (as presented on the left side of Figure~\ref{fig:song}):
\begin{equation}\label{eq:matrix_Px}
\abovedisplayskip=2pt
\abovedisplayshortskip=0pt
\belowdisplayskip=0pt
\belowdisplayshortskip=1pt
P_x = 
\left[
\begin{array}{c;{1pt/1pt}r}
    \mathrm{M}_x & 
    \begin{matrix} \mathrm{0} \\\hdashline[1pt/1pt] \mathrm{I} \end{matrix}
\end{array}
\right]\in\R^{(n+c)\times(n+c)},
\end{equation}
where $\mathrm{0}\in\R^{n\times c}$ is zero matrix and $\mathrm{I}\in\R^{c\times c}$ is an identity matrix. As a result, we obtain a square stochastic (transition) matrix used to describe the transitions of a Markov chain. While $P_x$ contains the probability of moving from $u_i$ to $u_j$ in one time step, it can be easily used to obtain a similar probability for $N$ steps by calculating the $N$-th power of $P_x$. Finally, the resulting matrix can be multiplied by vector $v=[1, 0, ..., 0]^T$ to obtain the probability of moving from the root to any node of the graph, including leaves, whose probability is the output of the model. We present a simple example illustrating this process on the right side of Figure~\ref{fig:song}. More examples are provided in the Supplementary Materials.

\subsubsection{Algorithm complexity}

The memory scales quadratically with the number of nodes $n$ due to the necessity of storing transition matrices $\mathrm{M}^0$ and $\mathrm{M}^1$.
For a single image $x$, the computational complexity is bounded by the Hadamard product used to produce matrix $\mathrm{M}_x$ (and $P_x$ consequently), which is $O(n^2)$. Considering also the number of steps $N$, vector $v$ is $N$ times multiplied by matrix $P_x$. Thus, the complexity is $O(n^2+N\cdot (n + c)^2)$, where $c$ is the number of leaves.

\begin{figure*}[t]
    \centering
    
    \hspace*{\fill}
    \begin{subfigure}[b]{.35\textwidth}
        \centering
        \includegraphics[width=.95\textwidth]{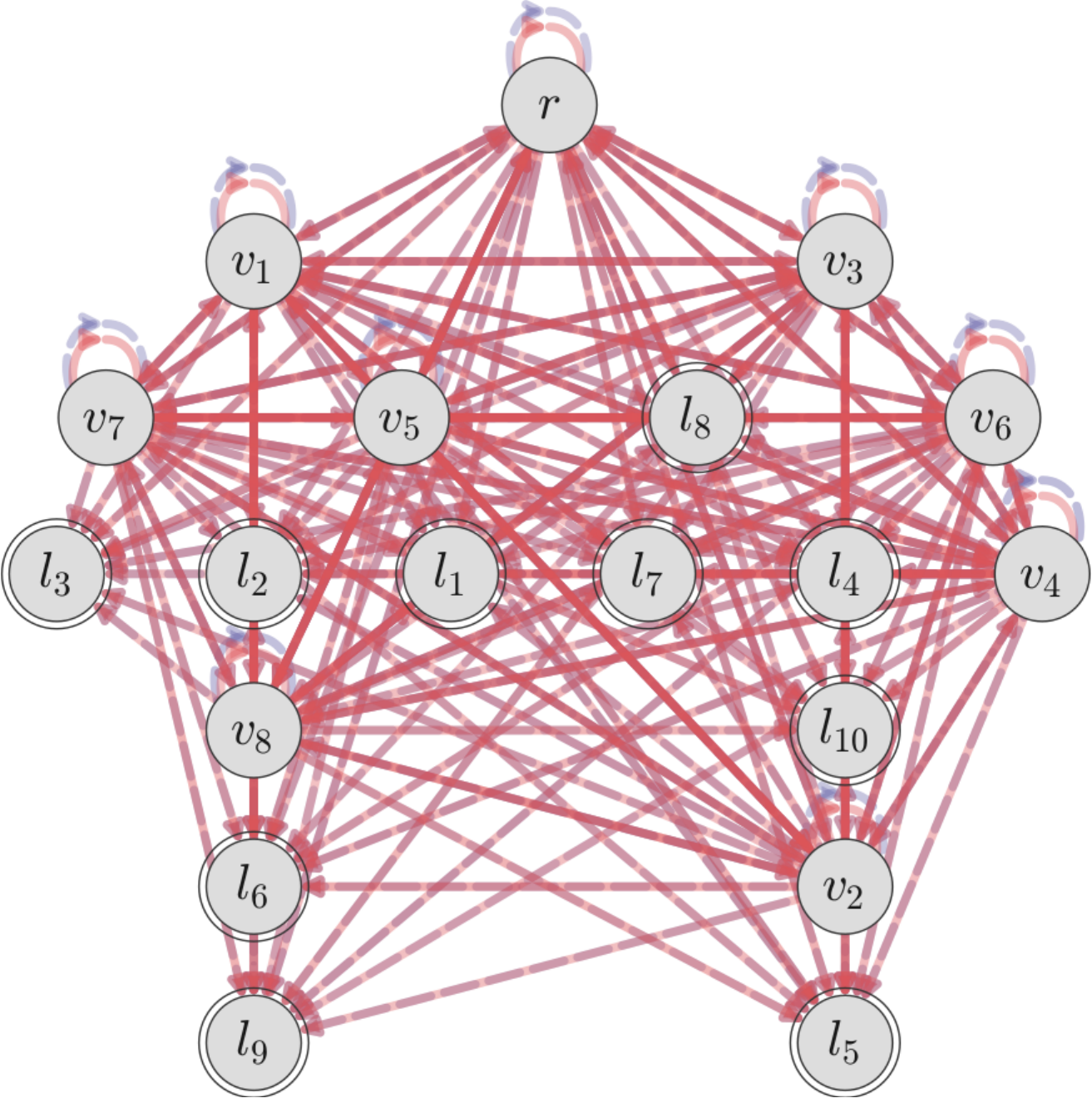}
        \caption{Graph corresponding to the initial $M^0$, $M^1$.}
    \end{subfigure}
    \hfill
    \begin{subfigure}[b]{.35\textwidth}
        \centering
        \includegraphics[width=.95\textwidth]{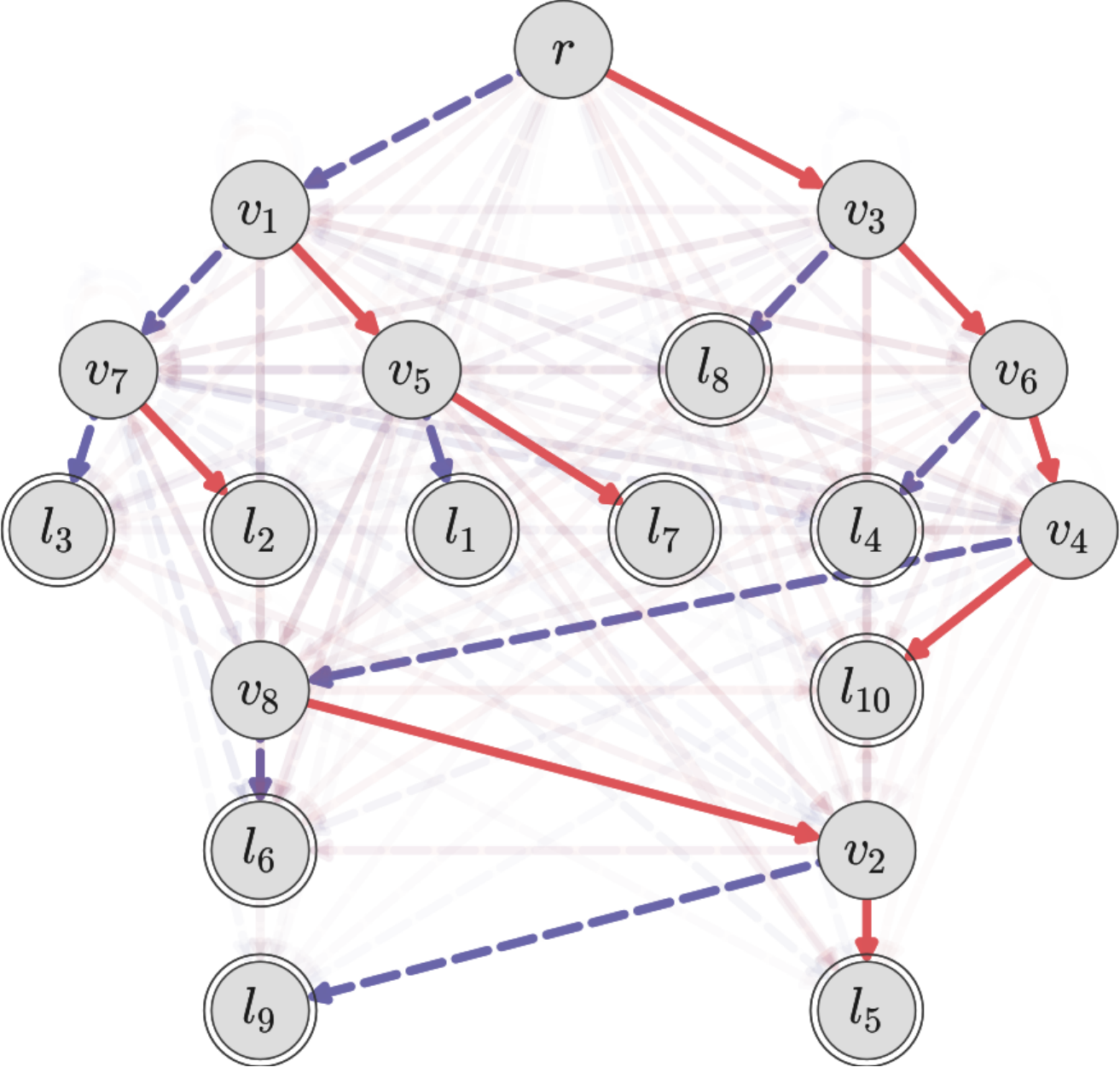}
        \caption{Graph corresponding to the trained $M^0$, $M^1$.}
    \end{subfigure}
    \hspace{1ex}
    \hspace*{\fill}
    \\  
    
    \hspace*{\fill}
    \begin{subfigure}[b]{.4\textwidth}
        \centering
        \includegraphics[width=\textwidth]{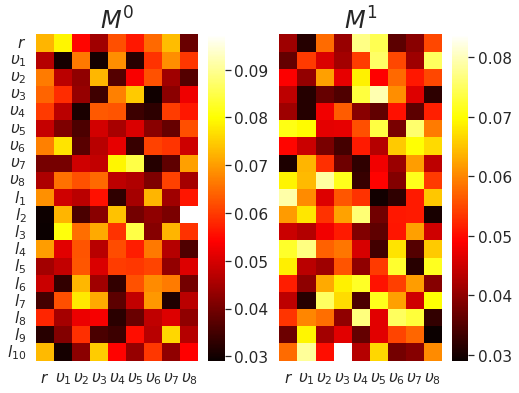}
        \caption{$M^0$, $M^1$ with initial values.}
        \label{fig:init_M}
    \end{subfigure}
    \hfill
    \begin{subfigure}[b]{.4\textwidth}
        \centering
        \includegraphics[width=\textwidth]{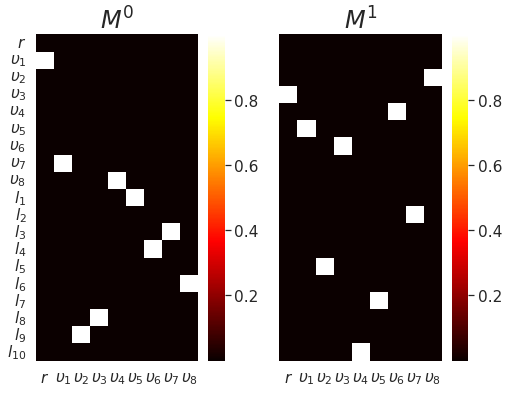}
        \caption{Trained $M^0$, $M^1$.}
        \label{fig:trained_M}
    \end{subfigure}
    \hspace*{\fill}
    
    \caption{Visualisation of the graphs with edges corresponding to the transition matrices $M^0$ (blue edges) and $M^1$ (red edges) of the SONG before and after training on the MNIST dataset with $9$ internal nodes. One can observe that SONG models binarize the connections during gradient training.}
    \label{fig:M_visualization}
\end{figure*}

\subsection{Regularizations}
Similarly as in Soft Decision Trees (SDT)~\cite{frosst2017distilling}, we observe that our graphs require additional training regularizers. The reasons for that are threefold.
First, SONG may get stuck on plateaus in which one or more $\sigma^d_i(x)$ is 0 for all input samples $x$, and the gradient of the sigmoid logistic function for this decision is always very close to zero.
Second, if SONG is uncertain of its predictions, it can safely hold the probabilities in internal nodes instead of moving them to leaves, which results in a small accumulated probability in the latter.
Third, SONG tends to binarize what is positive in general, but if this binarization appears too early, the model can get stuck in a local minimum.
Therefore, to prevent model degeneration, we introduce three types of regularization.

\subsubsection{Node regularization} The node regularization is a direct adaptation of the approach proposed by~\cite{frosst2017distilling}. It is used to avoid getting stuck at poor solutions by encouraging each internal node to make equal use of both left and right subtrees. In our approach, this regularization encourages each internal node to make equal use of both sets of edges $\E^0$ and $\E^1$. I.e., to send half of the training samples to one direction (using $M_0$) and half of them to the other direction (with $M_1$). For this purpose, we calculate the cross entropy between the desired average distribution $0.5$, $0.5$ for those two sets and the actual average distribution $\alpha_{i,s}$, $\beta_{i,s}$ in node $v_i$ at step $s$
\[
\setlength{\abovedisplayskip}{3pt} \setlength{\abovedisplayshortskip}{0pt}
\setlength{\belowdisplayskip}{3pt} \setlength{\belowdisplayshortskip}{0pt}
L_{nodes}=-\frac{\lambda}{2}\sum_{i=1}^n\log(\alpha_{i,s}) + \log(\beta_{i,s}),
\]
where
\[
\setlength{\abovedisplayskip}{3pt} \setlength{\abovedisplayshortskip}{0pt}
\setlength{\belowdisplayskip}{3pt} \setlength{\belowdisplayshortskip}{0pt}
\alpha_{i,s}=\frac{\sum_{x\in B} (P_x^s r)_i\cdot (\sigma_i^1(x))^\gamma}{\sum_{x\in B} (P_x^s r)_i},
\]
\[
\setlength{\abovedisplayskip}{3pt} \setlength{\abovedisplayshortskip}{0pt}
\setlength{\belowdisplayskip}{3pt} \setlength{\belowdisplayshortskip}{0pt}
\beta_{i,s}=\frac{\sum_{x\in B} (P_x^s r)_i\cdot(\sigma_i^0(x))^\gamma}{\sum_{x\in B} (P_x^s r)_i},
\]
$B$ is a batch of samples used in an iteration, $\gamma \in [1, 2]$, and $(P_x^s r)_i$ corresponds to $i$th coordinate of vector $(P_x^s r)$.
One can observe that our node regularizer is calculated per node and step, and it is different from~\cite{frosst2017distilling}, where additional loss is computed once for each node. Moreover, we penalize model for making  uncertain decisions ($\sigma_{i,s}(x)\approx 0.5$) using the parameter $\gamma$.

\subsubsection{Leaves regularization}

The leaves regularization, enforcing the summary probabilities in leaves to be close to $1$, is defined as
\begin{equation}\label{eq:loss_leaves}
\setlength{\abovedisplayskip}{3pt} \setlength{\abovedisplayshortskip}{0pt}
\setlength{\belowdisplayskip}{3pt} \setlength{\belowdisplayshortskip}{0pt}
L_{leaves} = -\log\left(\sum_{i=n}^{n + c} (P_x^N r)_i\right),
\end{equation}
where $n$ is the number of nodes (excluding root indexed with $0$), $c$ is the number of leaves (classes), and $N$ is the number of steps.

\subsubsection{Gumbel-softmax} We use Gumbel-softmax~\cite{jang2016categorical} instead of softmax to each column of matrices $\MM_0$ and $\MM_1$ to explore the trajectories of the graph better. In other words, Gumbel-softmax introduces randomness, which results in a wider exploration of the graph structure in the optimization process.

\section{Theoretical analysis}
\label{sec:theory}

In the supplementary materials, we provide theoretical results, showing that the graph structure generated by \our{} is binarized during training, as presented in Figure~\ref{fig:M_visualization}.

\begin{table}[t]\small
    \centering
    \addtolength{\tabcolsep}{-3pt}
    \begin{tabular}{lcccrrrr}
    \toprule
    Method & Ex & SO & EE & MNIST & C10 & C100 & TIN \\
    \midrule
    DDN (NiN) & \xmark & \cmark & \xmark & - & 90.32 & 68.35 & -\\
    DCDJ (NiN) & \xmark & \cmark & \cmark & - & - & 69.00 & -\\
    ANT-A* (n/a) & \cmark & \cmark & \xmark & \textbf{99.36} & 93.28 & - & -\\
    \midrule
    ResNet18 & \xmark & \xmark & \xmark & 98.91 & 94.93 & 75.82 & 63.05\\
    DNDF & \xmark & \xmark & \xmark & 97.20 & 94.32 & 67.18 & 44.56\\
    DT & \cmark & \xmark & \xmark & - & 93.97 & 64.45 & 52.09\\
    NBDT & \cmark & \xmark & \xmark & - & 94.82 & \textbf{77.09} & \textbf{64.23}\\
    NBDT w/o h. & \cmark & \cmark & \xmark & - & 94.52 & 74.97 & -\\
    RDT & \cmark & \cmark & \cmark & - & 93.12 & - & -\\
    SONG (ours) & \cmark & \cmark & \cmark & 98.81 & \textbf{95.62} & 76.26 & 61.99\\
    \bottomrule
  \end{tabular}
    \caption{Comparison of models with deep architecture in terms of model features and accuracy on MNIST, CIFAR10 (C10), CIFAR100 (C100), and TinyImageNet (TIN). ResNet18 was used to extract the vector representation of input images for DNDF~\cite{kontschieder2015deep}, DT, NBDT (with and without hierarchy)~\cite{wan2020nbdt}, RDT~\cite{alaniz2021learning}, and SONG. For DDN~\cite{murthy2016deep}, DCDJ~\cite{baek2017deep}, and ANT-A~\cite{tanno2019adaptive}, the backbone models are provided in the brackets. ``Ex'' indicates if the method retains properties such as pure leaves, sequential decisions, and non-ensemble. ``SO'' indicates if the model is self-organized (does not require a predefined structure). ``EE'' indicates if the structure and weights of model are trained in an end-to-end continuous manner.}
    \label{tab:compare2nbdt}
\end{table}

\section{Experiments}
\label{sec:experiments}

In this section, we analyze the accuracy of the SONGs trained on Letter~\cite{asuncion2007uci}, Connect4~\cite{asuncion2007uci}, MNIST~\cite{lecun2010mnist}, CIFAR10~\cite{krizhevsky2009learning}, CIFAR100~\cite{krizhevsky2009learning}, and TinyImageNet~\cite{le2015tiny} datasets and compare it with the state of the art methods~\cite{wan2020nbdt,kontschieder2015deep,murthy2016deep,tanno2019adaptive,alaniz2021learning,baek2017deep}. We examine how the number of nodes and steps influence the structure of graphs, the number of internal nodes used by the model, the number of back edges, and the distance from the root to leaves. Moreover, we explain how the probability of back and cross edge changes in the successive training steps. Finally, we provide a detailed comparison with SDT~\cite{frosst2017distilling} and present sample graphs obtained for the MNIST dataset. In all experiments, we use leaves normalization and Gumbel-softmax, and we treat node regularization as a hyperparameter of the model. While this section presents only the most important findings for the sake of clarity, the experimental setup and detailed results can be found in the Supplementary Materials.

\subsection{SONG in deep learning setup}

In the first experiment, we apply SONG~\footnote{We made the code available at: \href{https://github.com/gmum/SONGs}{https://github.com/gmum/SONGs}} at the top of the backbone Convolutional Neural Network (CNN) without the final linear layer. CNN takes the input image and generates the representation, which is passed to the SONG. SONG processes the representation and returns the predictions for each class, which are then used with target labels to calculated Binary Cross-Entropy (BCE) loss. As a backbone network, we use ResNet18 for all datasets except MNIST, for which we employ a smaller network (see Supplementary Materials for details).

As presented in Table~\ref{tab:compare2nbdt}, our method matches or outperforms most of the recent state-of-the-art methods. On CIFAR10, SONG accuracy outperforms all baseline by almost 1 percentage point. On MNIST, it is worse than ANT~\cite{tanno2019adaptive} by around $0.5\%$, and on CIFAR100 and TinyImagNet, NBDT~\cite{wan2020nbdt} achieves better results. However, both ANT and NBDT are not trained in an end-to-end continuous manner. Moreover, NBDT requires a hierarchy provided before training, and without such a hierarchy, it obtains accuracy more than $1\%$ lower than SONG on CIFAR100.

\begin{table}[t]\small
    \centering
    \addtolength{\tabcolsep}{-3pt}
    \begin{tabular}{lllllll}
    \toprule
    Method & Letter & Connect4 & MNIST \\
    \midrule
    SDT w/o distillation~\cite{frosst2017distilling} & 78.00 (511) & 78.63 (255) & 94.45 (255)\\
    SDT~\cite{frosst2017distilling} & 81.00 (511) & 80.60 (255) & \textbf{96.76} (255)\\
    SONG-S-large (ours) & \textbf{86.25} (511) & \textbf{82.82} (255) & 95.74 (255) \\
    SONG-S-small (ours) &  82.95 (64) & 80.27 (8) & 94.66 (64)\\
    \bottomrule
  \end{tabular}
    \caption{Comparison of SDT~\cite{frosst2017distilling} and shallow SONG (SONG-S) on three datasets, where shallow corresponds to direct flattened inputs (no backbone network used). The accuracy of each model is reported along with the number of internal nodes specified in the parentheses. SONG-S-small contains the minimal number of nodes necessary to match the accuracy of SDT. SONG-S-large uses the same number of internal nodes as SDT. Please notice that SONG models are trained without a distillation mechanism, and they always obtain better results than SDT without distillation.}
    \label{tab:hinton}
\end{table}

\subsection{SONG as shallow model}
Although SONG can be successfully used in a deep learning setup, it can also be treated as a shallow model. In this case, SONG directly processes an input sample and returns the predictions passed with target labels to BCE loss. This setup is similar to the one presented in experiments on SDTs~\cite{frosst2017distilling}. Hence, we compare to SDT on all datasets considered in~\cite{frosst2017distilling}.

Table~\ref{tab:hinton} shows that SONG obtains better results than SDT without distillation on all datasets. Moreover, on Letter and Connect4, SONG outperforms even SDT with distillation. We also observe that SONG requires fewer nodes than SDT and obtains on par results on the Connect4 dataset with $30$ times fewer nodes. For Letter and MNIST, similarly good results can be obtained with $30$ times fewer nodes. This finding is in line with~\cite{shotton2016decision} which shows that decision graphs require dramatically less memory while considerably improving generalization.

\subsection{SONG structure}

\begin{figure}
    \centering
    \begin{subfigure}[b]{.3\textwidth}
        \centering
        \includegraphics[width=\textwidth]{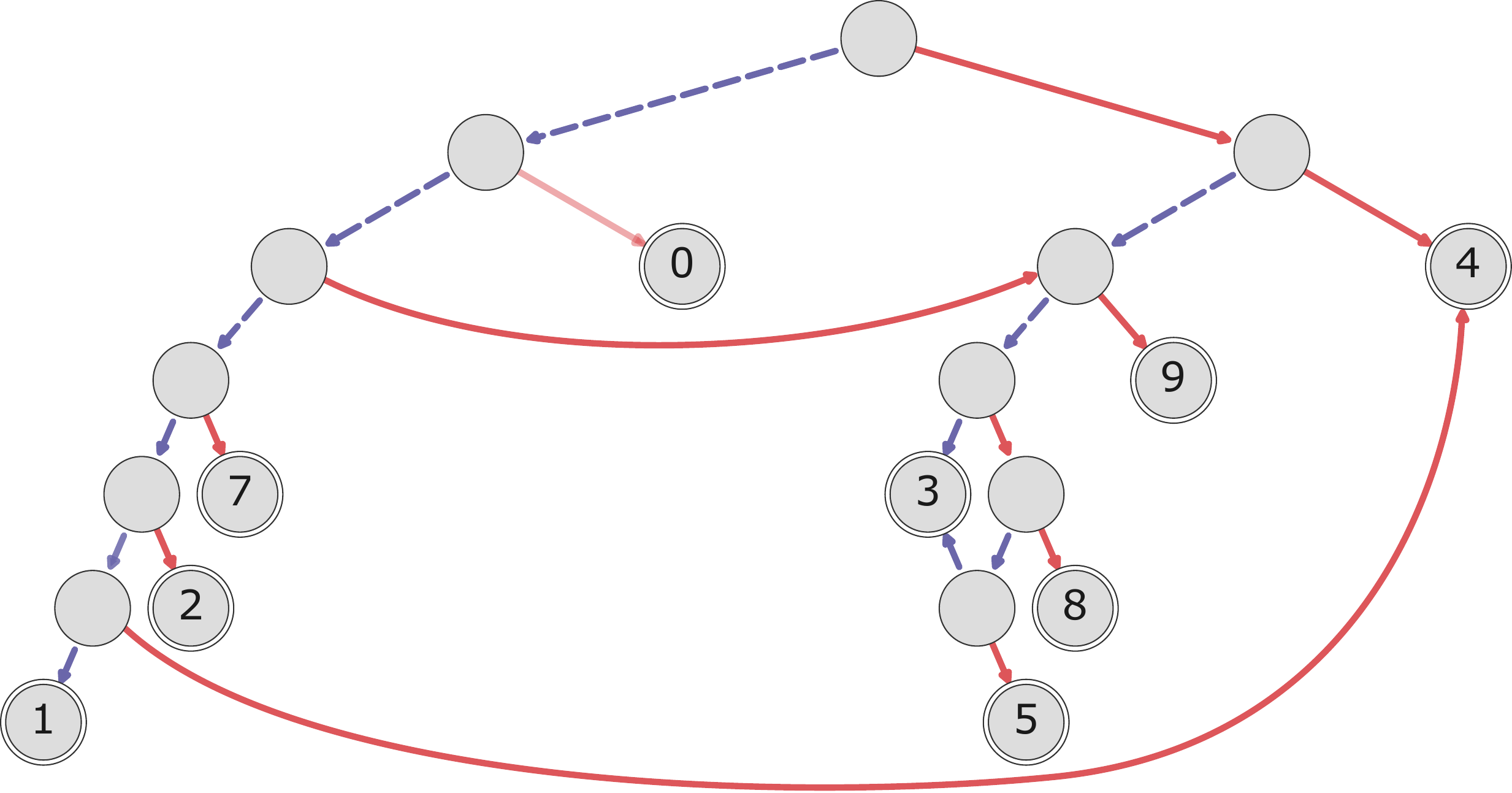}
        \caption{16 internal nodes and 8 steps.}
        \label{fig:mnist_graph_1}
    \end{subfigure}
    \qquad
    \begin{subfigure}[b]{.3\textwidth}
        \centering
        \includegraphics[width=\textwidth]{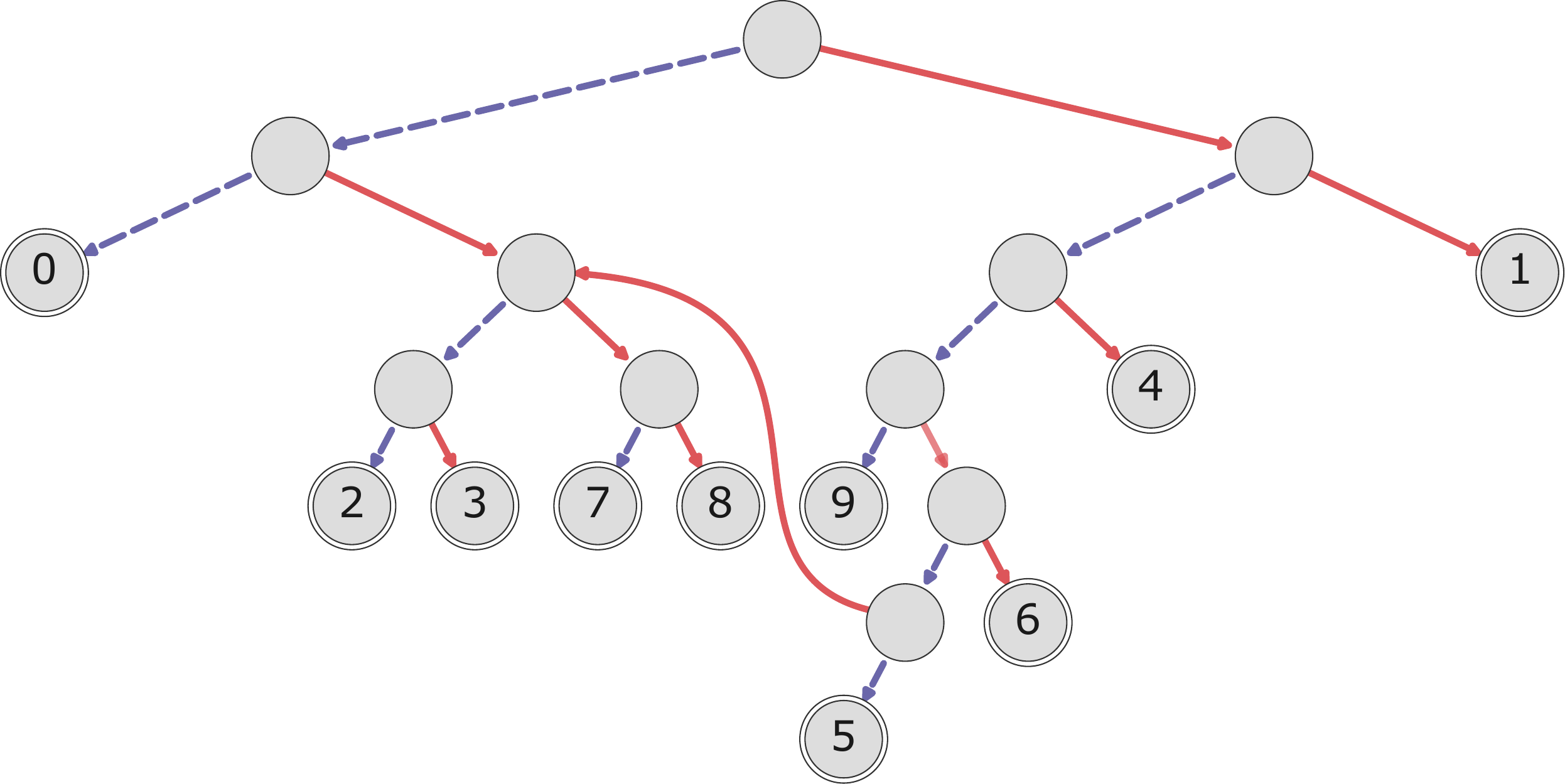}
        \caption{32 internal nodes and 8 steps.}
        \label{fig:mnist_graph_2}
    \end{subfigure}
    
    \caption{Examples of the graph structures obtained by training \our{} on the MNIST dataset. The root is the top-most node in each graph, and double node borders denote the leaves with numbers of the MNIST classes. For each node $v_i$, we present two edges corresponding to the highest probability from two transition vectors $m_{\cdot{}i}^0$ and $m_{\cdot{}i}^1$ (represented as dashed blue and solid red arrows, respectively).}
    \label{fig:mnist_graph}
\end{figure}

As a fully differentiable model, SONG strengthens or weakens an edge between any pair of nodes during training to constantly optimize the graph's structure (see Figure~\ref{fig:M_visualization}). Consequently, it can generate any structure that uses all available nodes, or only some of them. In particular, the final structure may be a binary tree or contain back edges. Moreover, the distance from the root to leaves can vary. This variability is visualized in Figure~\ref{fig:mnist_graph}, where we present two graphs obtained for MNIST using a different number of internal nodes and steps.

In Figure~\ref{fig:cifar10_stat}, we provide statistics on multiple SONGs generated for the CIFAR10 dataset. We observe a significant difference in SONG structure depending on the number of internal nodes and steps. First, we note that the number of internal nodes used by the model increases with the increasing number of steps $N$, and it does not depend on the total number of internal nodes $n$. As a natural consequence, a similar trend is observed for the distance from the root to the leaves. When it comes to back edges, their number is relatively small, and they appear only for a larger number of steps. At the same time, the cross edges are more often and increase with the increased number of internal nodes.

\begin{figure}[t]
    \centering
    \includegraphics[width=0.5\textwidth]{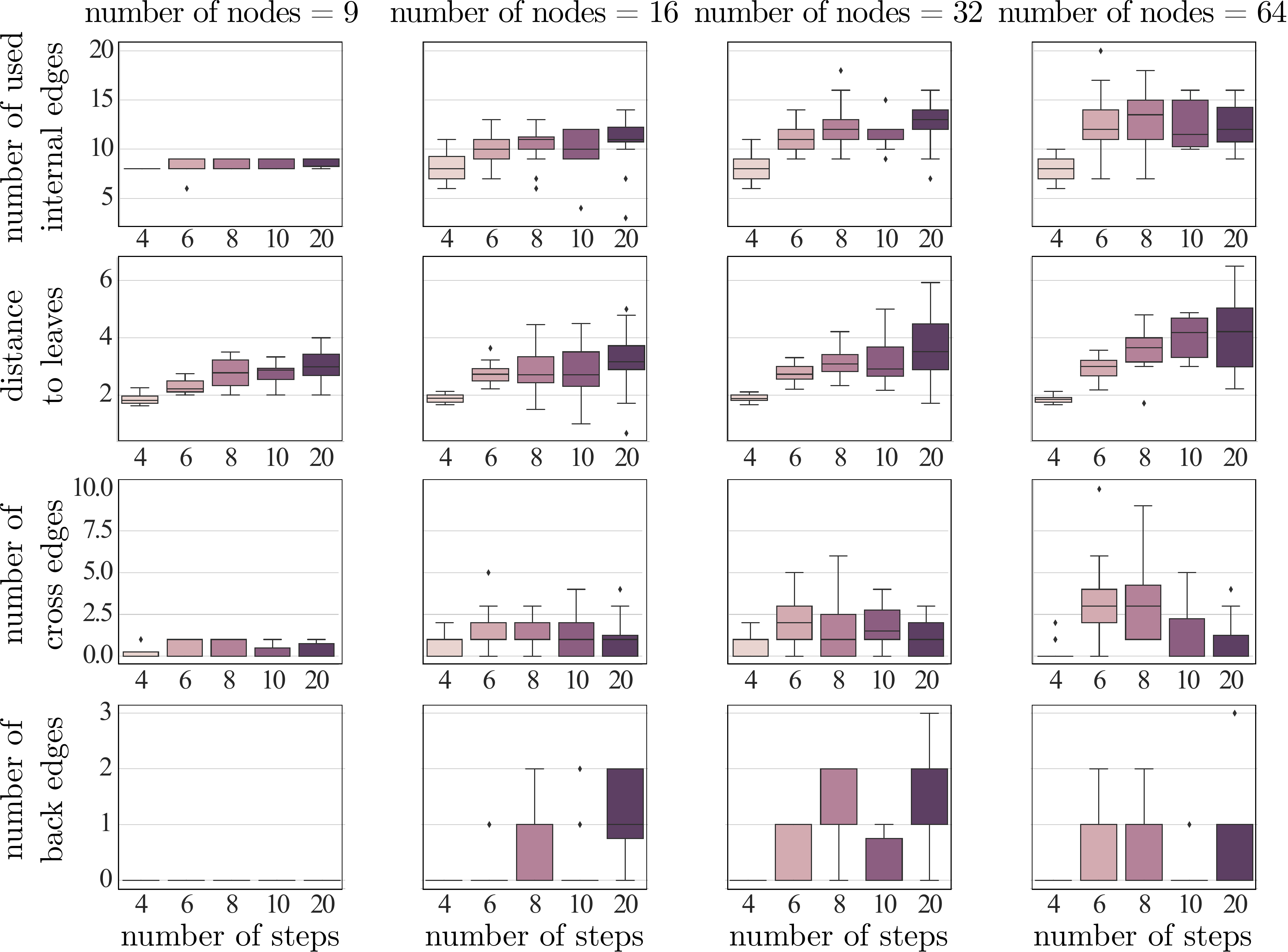}
    \caption{Nodes and edges statistics calculated for SONGs trained on the CIFAR10 dataset. For each combination of the number of internal nodes and steps, 20 graphs are trained and used to plot the distributions of four statistics.}
    \label{fig:cifar10_stat}
\end{figure}

We also analyze the relationship between the number of nodes and steps and prediction accuracy \our{}. As can be seen in Table~\ref{tab:deep}, the performance on MNIST constantly increases as the number of nodes and steps grow. This is not the case of CIFAR-10, where the accuracy remains similar for all combinations of the parameters. It can be caused by the smaller dimension of the representation vector in MNIST ($50$) than in CIFAR10 ($512$).


\begin{table}[ht]
    \centering
    \begin{subtable}{\linewidth}
        \centering
        \begin{tabular}{lrrrrr}
        \toprule
        & \multicolumn{5}{c}{steps}\\
        \cmidrule{2-6}
        nodes & 4 & 6 & 8 & 10 & 20 \\
        \midrule
        9 & 95.66 & 97.29 & 97.25 & 97.95 & 97.56 \\
        16 & 97.31 & 97.83 & 98.23 & 98.43 & 98.56 \\
        32 & 96.82 & 97.74 & 98.35 & 98.65 & 98.62 \\
        64 & 96.29 & 98.12 & 98.12 & 98.47 & 98.68 \\
        \bottomrule
      \end{tabular}
        \caption{MNIST.}
        \label{tab:deep_cifar}
        \vspace{-10pt}
    \end{subtable}\\
    \vspace{4mm}
    \begin{subtable}{\linewidth}
        \centering
        \begin{tabular}{lrrrrr}
        \toprule
        & \multicolumn{5}{c}{steps}\\
        \cmidrule{2-6}
        nodes & 4 & 6 & 8 & 10 & 20 \\
        \midrule
        9 & 94.48 & 94.86 & 94.92 & 94.94 & 94.93 \\
        16 & 94.88 & 94.95 & 94.86 & 94.87 & 94.89 \\
        32 & 94.99 & 94.95 & 94.95 & 94.90 & 94.98 \\
        64 & 94.90 & 94.87 & 94.88 & 94.94 & 94.93 \\
        \bottomrule
      \end{tabular}
        \caption{CIFAR10.}
        \label{tab:deep_cifar10}
    \end{subtable}%
    \caption{Results of SONG in a deep learning setup. One can observe that for the MNIST dataset (a), the performance increases with the increasing number of nodes and steps. In contrast to CIFAR10 (b), where the performance is relatively similar for all combinations of the parameters. }
    \label{tab:deep}
\end{table}

\subsection{SONG structure during training} We analyze the relationship between BCE loss and the probability of back and cross edges in the successive epochs of the training. 

\begin{figure}
    \centering
    \begin{subfigure}[b]{.5\textwidth}
        \centering
        \includegraphics[width=.93\textwidth]{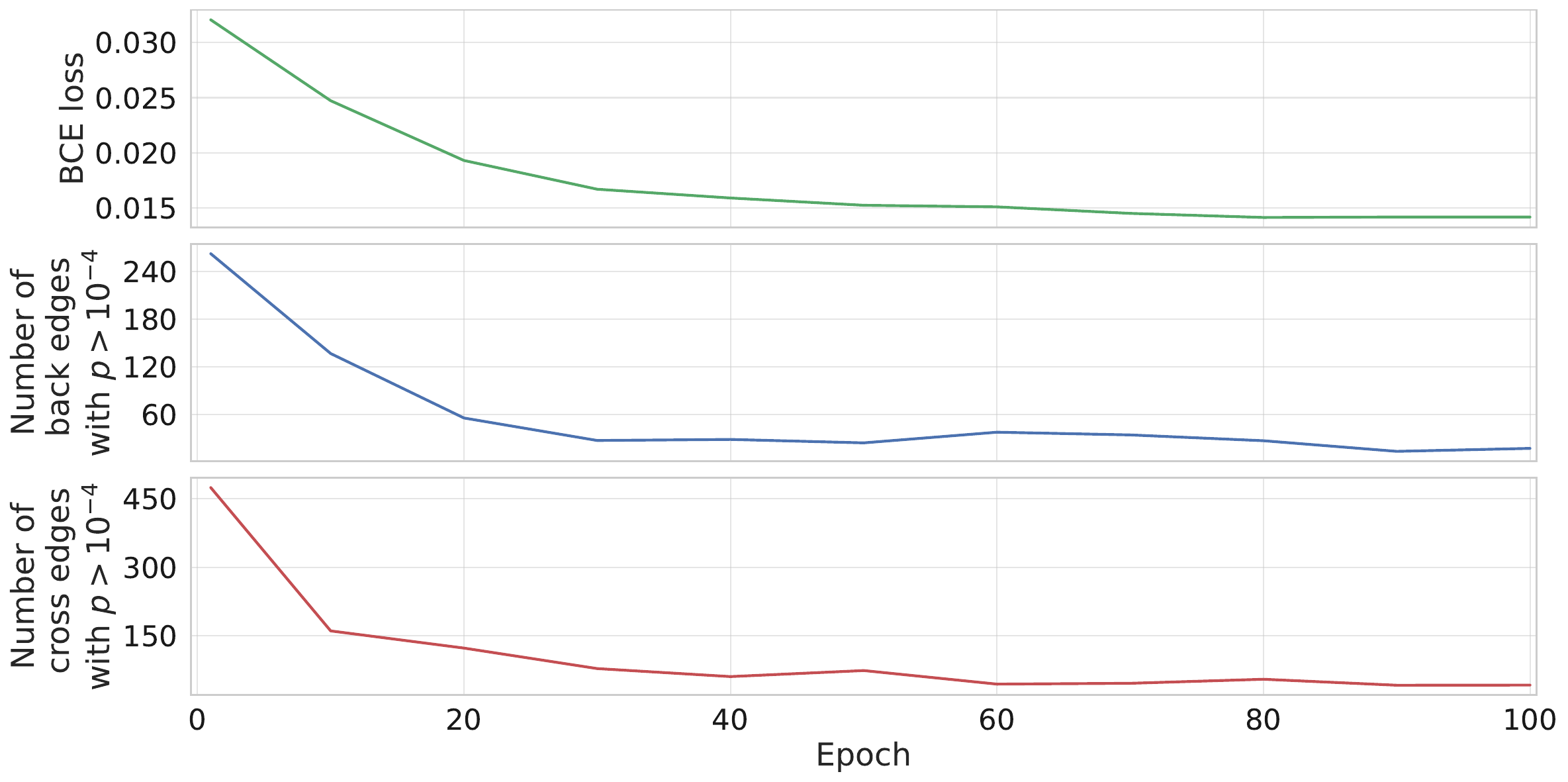}
        \caption{SONG with 256 internal nodes and 10 steps trained for CIFAR100.}
        \label{fig:abstract1}
    \end{subfigure}
    \qquad
    \begin{subfigure}[b]{.5\textwidth}
        \centering
        \includegraphics[width=.93\textwidth]{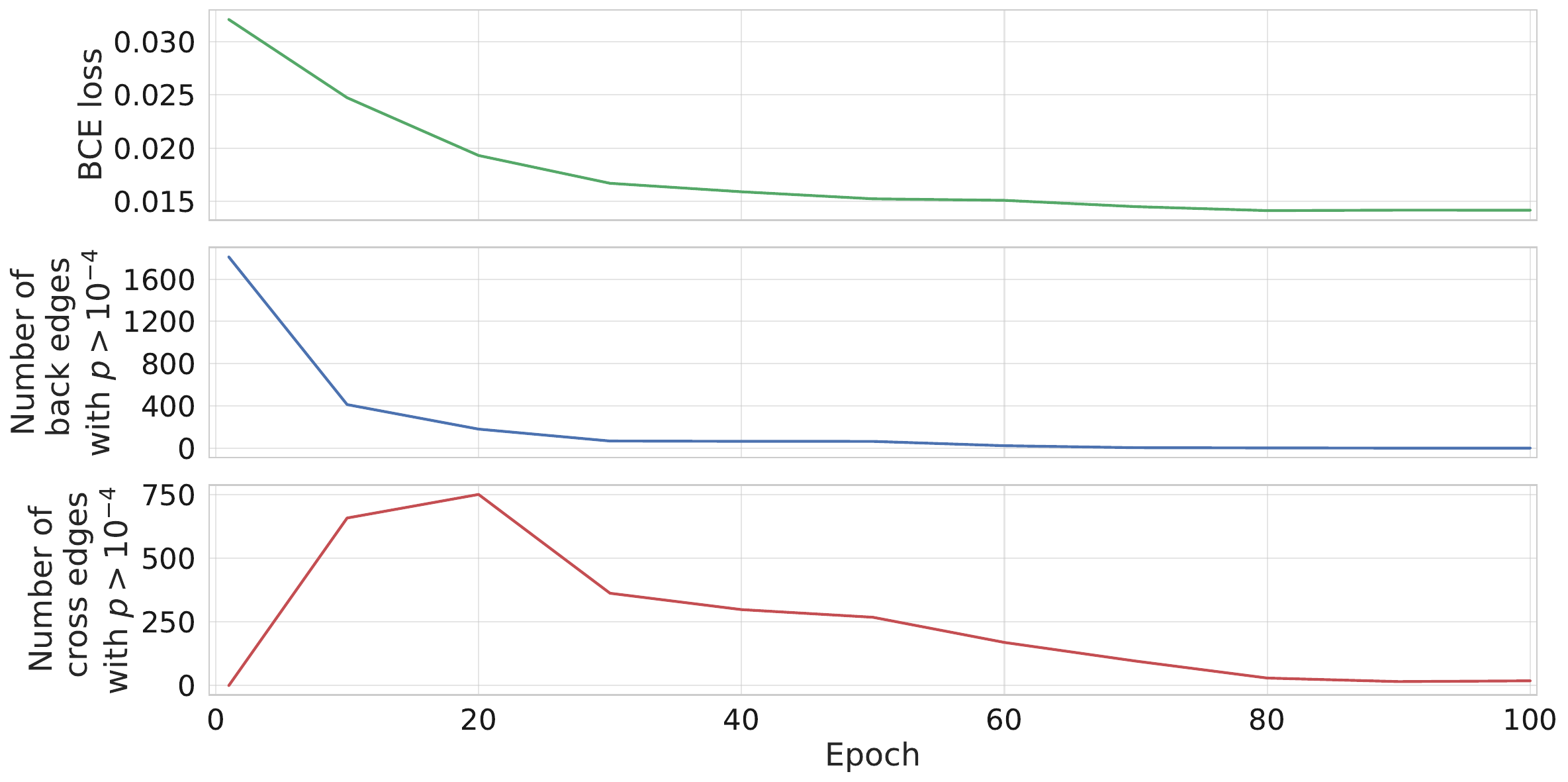}
        \caption{SONG with 64 internal nodes and 10 steps trained for MNIST.}
        \label{fig:abstract1}
    \end{subfigure}
    \caption{BCE loss as well as the number of back and cross edges in the successive training epochs of SONG. One can observe that number of back edges decrease together with decreasing BCE loss.}
    \label{fig.loss_vs_edges_cifar}
\end{figure}

We present the mean over multiple models and all test samples (as each test sample $x$ has its graph represented by matrix $P_x$). We observe that the probability of back edges decreases together with decreasing BCE loss, both for simple MNIST and more complicated CIFAR100 dataset (see Figure~\ref{fig.loss_vs_edges_cifar}).

Moreover, in Figure~\ref{fig:mean_dist_M}, we present the mean distances between transition matrices ($P_x$) obtained for samples of the same and different classes. One can observe, among others, that the diagonal is visibly darker than the rest of the matrix, which means that inputs from the same class have more similar transition matrices. This confirms that we obtain similar transition matrices for similar inputs.

\begin{figure}
  \centering
  \begin{subfigure}[b]{.2\textwidth}
    \centering
    \includegraphics[width=\textwidth]{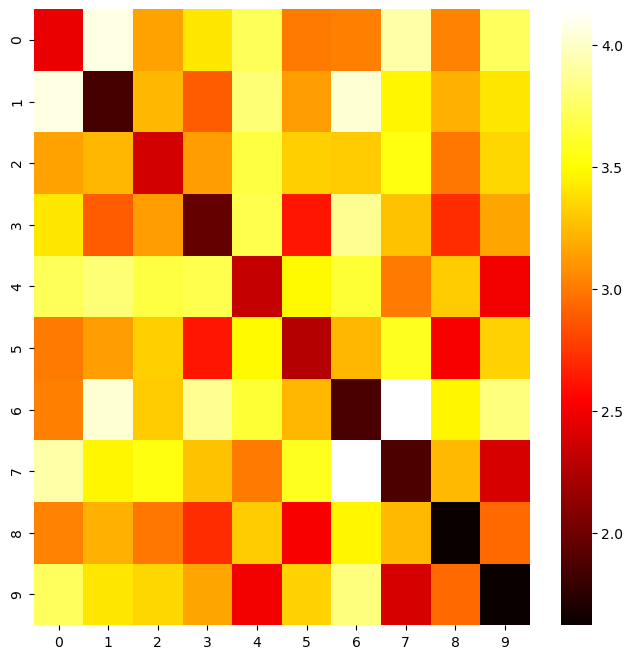}
    \caption{16 internal nodes.}
    \label{fig:loss}
  \end{subfigure}
  \quad
  \begin{subfigure}[b]{.2\textwidth}
    \centering
    \includegraphics[width=\textwidth]{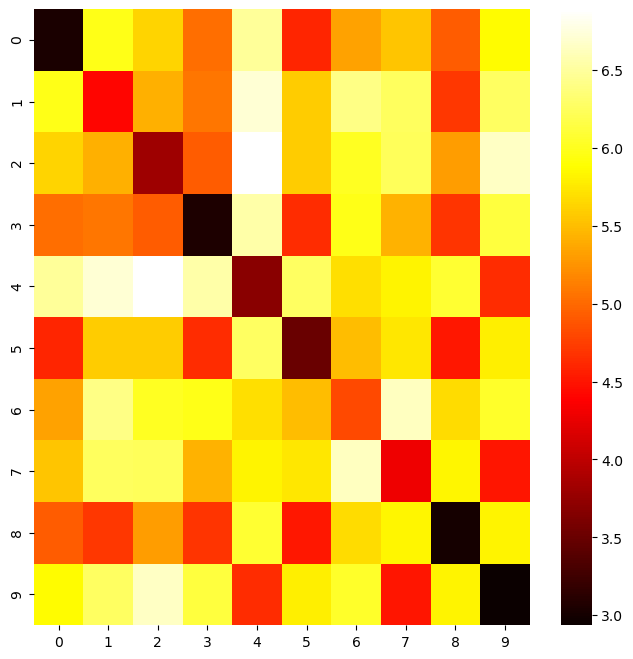} \\
    \caption{64 internal nodes.}
    \label{fig:loss_leaves}
  \end{subfigure}

  \caption{Mean distances between transition matrices $P_x$ for pairs of MNIST input samples represented by a distance matrix (the larger distance, the brighter color)
  . The rows and columns correspond to 0-9 digits.}
  \label{fig:mean_dist_M}
\end{figure}


\section{Conclusions}
In this work, we introduce Self-Organizing Neural Graphs (SONGs), a new type of decision graphs applicable in any deep learning pipeline. They optimize their structure during training by strengthening or weakening graph edges using gradient descent. Thanks to the graph structure, SONG can reuse the decision nodes and obtain state-of-the-art results with a significantly smaller number of nodes than existing methods. Moreover, the introduced general paradigm based on Markov processes allows for efficient training, and SONG converges to the binary acyclic directed graphs. Hence, we believe that our work opens a plethora of research pathways towards more effective applications of decision graphs in a deep learning setup.

\section*{Acknowledgments}
The work of J. Tabor and B. Zieliński is supported by the National Centre of Science (Poland) Grant No. 2021/41/B/ST6/01370, the work of Ł. Struski is supported by the National Centre of Science (Poland) Grant No. 2020/39/D/ST6/01332, and the work of T. Danel is supported by the National Centre of Science (Poland) Grant No. 2020/37/N/ST6/02728.
The research of M. \'Smieja was carried out within the research project ``Bio-inspired artificial neural network'' (grant no. POIR.04.04.00-00-14DE/18-00) within the Team-Net program of the Foundation for Polish Science co-financed by the European Union under the European Regional Development Fund.

{\small
\bibliographystyle{ieee_fullname}
\bibliography{ref}
}


\newpage

\appendix

\twocolumn[
\begin{center}
\textbf{\LARGE Supplementary Materials}
\vspace{0.5cm}
\end{center}
]

\section{SONGs visualization}

In this section, we provide the visualization of a SONG trained on MNIST where the nodes are represented by the learned filters and the ``average'' image passing through those nodes (see Figure~\ref{fig:mnist_Hinton_tree}).  Moreover, we provide additional examples of the graph structures obtained by training SONG on the MNIST and CIFAR10 datasets (see Figures~\ref{fig:mnist_graph_sm} and \ref{fig:cifar_graph_sm}) together with the consecutive steps of the Markov process (see Figures~\ref{fig:steps_1}, \ref{fig:steps_3}, \ref{fig:steps_2}, and \ref{fig:steps_4}).

Finally, we analyze the relationship between BCE loss and the probability of back and cross edges in the successive epochs of the training. We present the mean over multiple models and all test samples (as each test sample $x$ has its graph represented by matrix $P_x$). The number of back and cross edges is obtained in the following way. We first calculate all paths from the root with a probability higher than a particular threshold $0.0001$. Then we create a standard binary directed graph that contains all nodes and edges from those paths. Finally, we run the DFS algorithm for this graph (starting from the root) to obtain backed and cross edges.

\begin{figure*}[ht]
    \centering
    \includegraphics[width=0.75\textwidth]{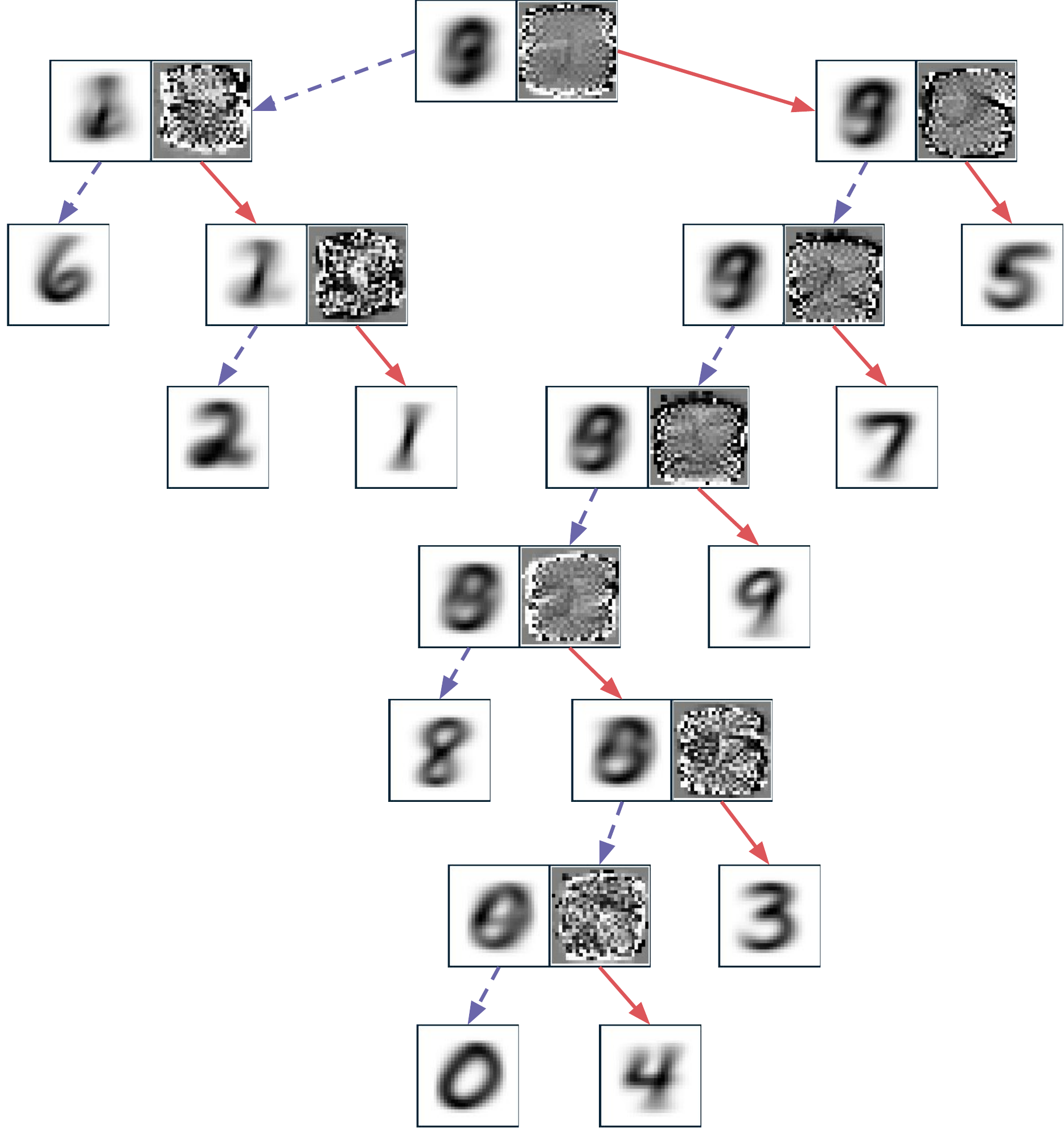}
    \caption{Visualization of a shallow SONG (SONG-S) trained on MNIST where the nodes are represented by the learned filters and the ``average'' image passing through those nodes (corresponding to the right and left side of each node, respectively). Notice that SONGs contain filters only in the inner nodes, as it is impossible to move out from the leaves.}
    \label{fig:mnist_Hinton_tree}
\end{figure*}

\begin{figure*}[ht]
    \centering
    \begin{subfigure}[b]{0.45\textwidth}
        \centering
        \includegraphics[width=\textwidth]{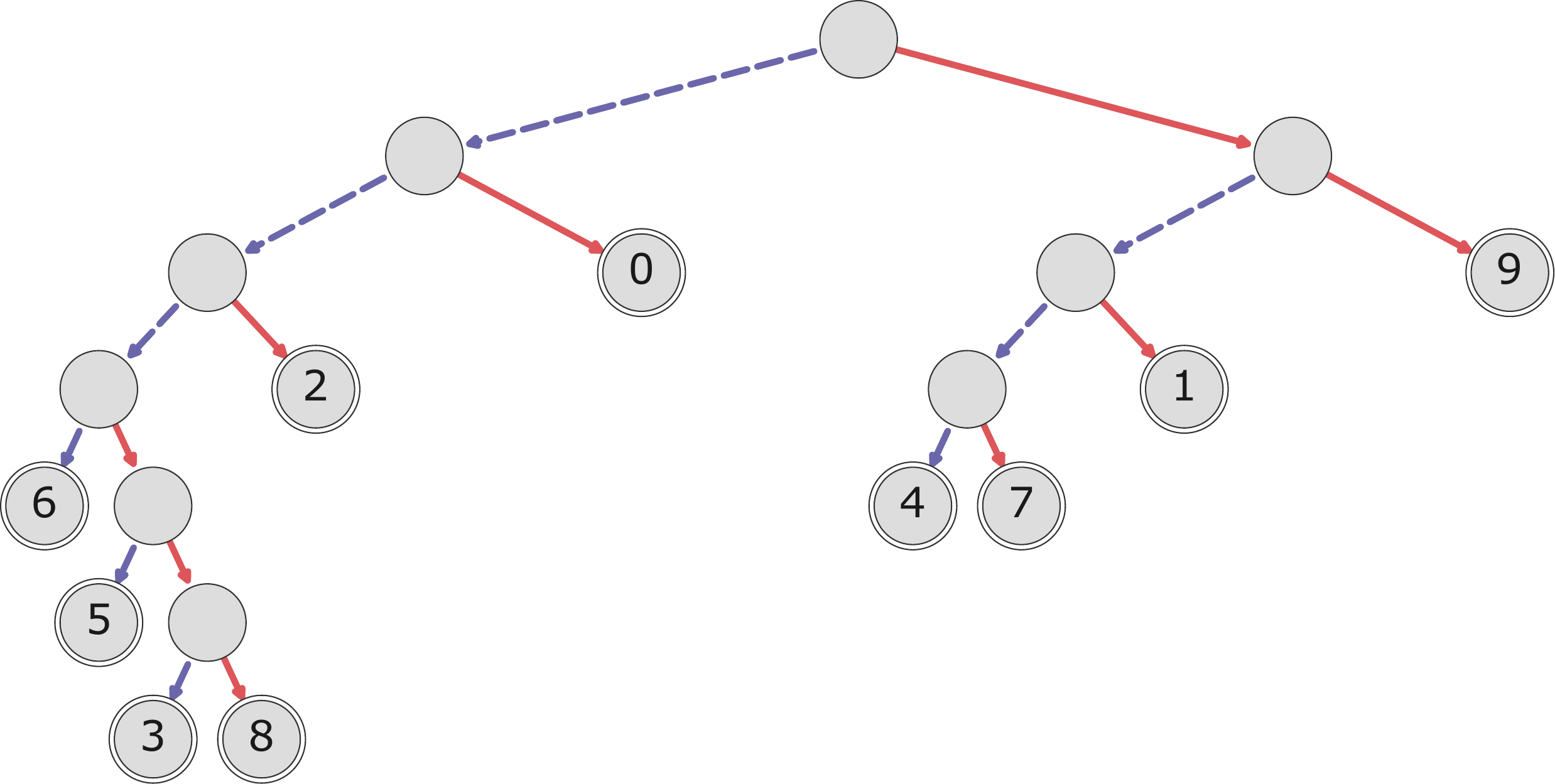}
        \caption{}
        \label{fig:mnist_graph_1}
    \end{subfigure}
    \qquad
    \begin{subfigure}[b]{0.45\textwidth}
        \centering
        \includegraphics[width=\textwidth]{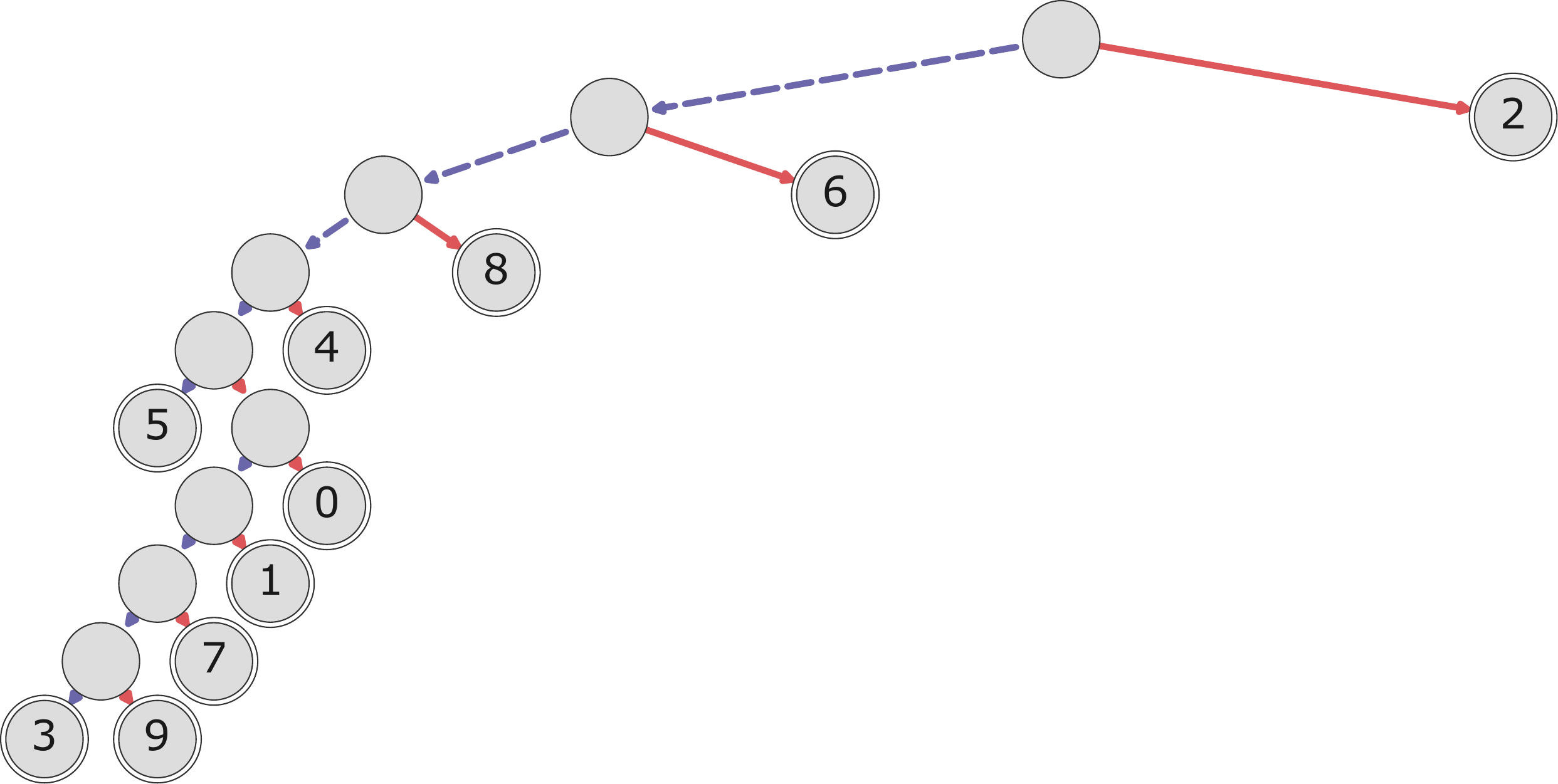}
        \caption{}
        \label{fig:mnist_graph_2}
    \end{subfigure}

    \caption{Examples of the graph structures obtained by training \our{} on the MNIST dataset. The root is the top-most node in each graph, and the leaves are denoted by double node borders. The numbers on the leaves are the MNIST classes. For each node $v_i$, we present two edges corresponding to the highest probability from two transition vectors $m_{\cdot{}i}^0$ and $m_{\cdot{}i}^1$ (represented as dashed blue and solid red arrows, respectively).}
    \label{fig:mnist_graph_sm}
\end{figure*}

\begin{figure*}[ht]
    \centering
    \begin{subfigure}[b]{0.49\textwidth}
        \centering
        \includegraphics[width=\textwidth]{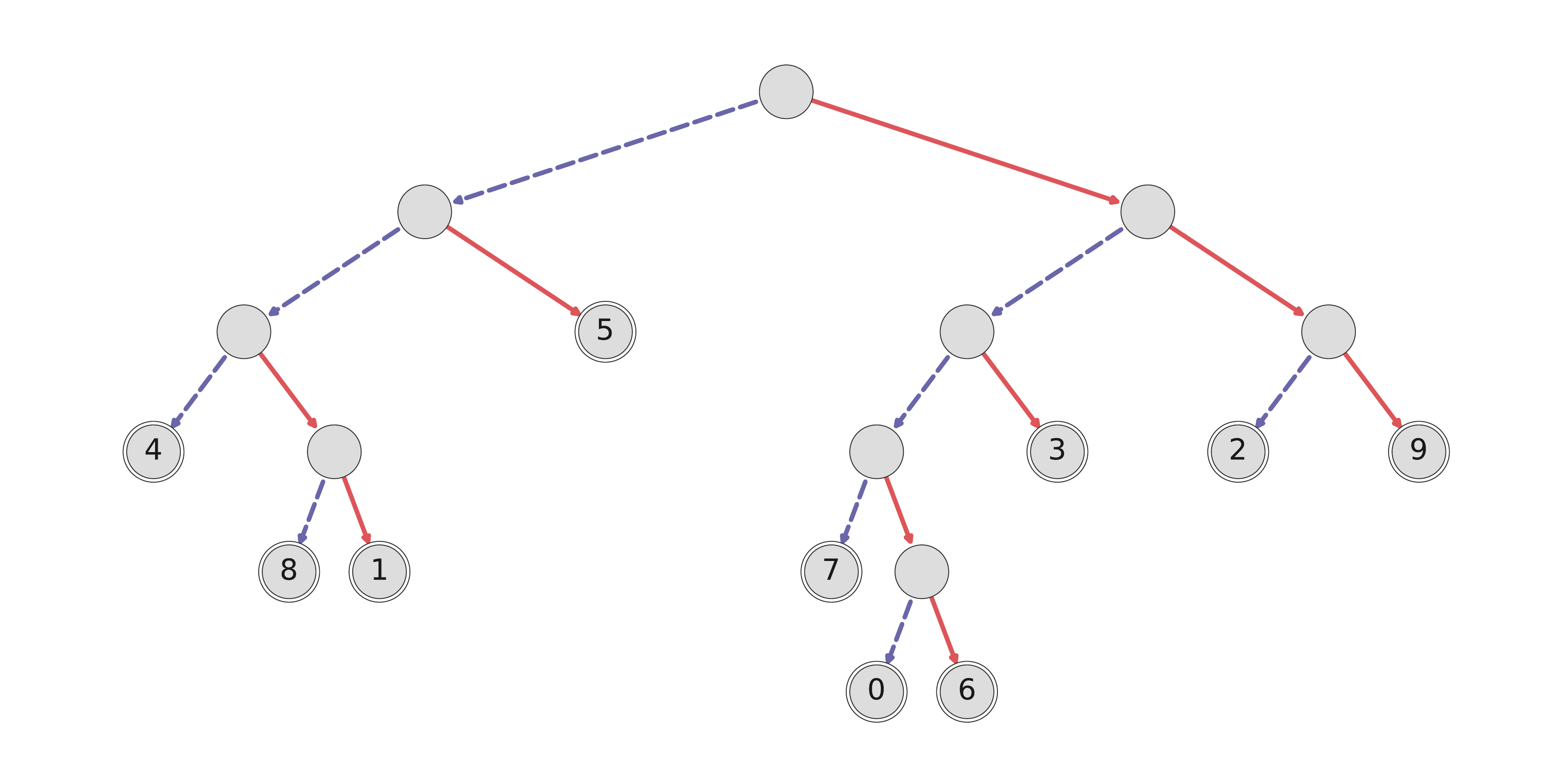}
        \caption{}
        \label{fig:mnist_graph_1}
    \end{subfigure}
    \hfill
    \begin{subfigure}[b]{0.49\textwidth}
        \centering
        \includegraphics[width=\textwidth]{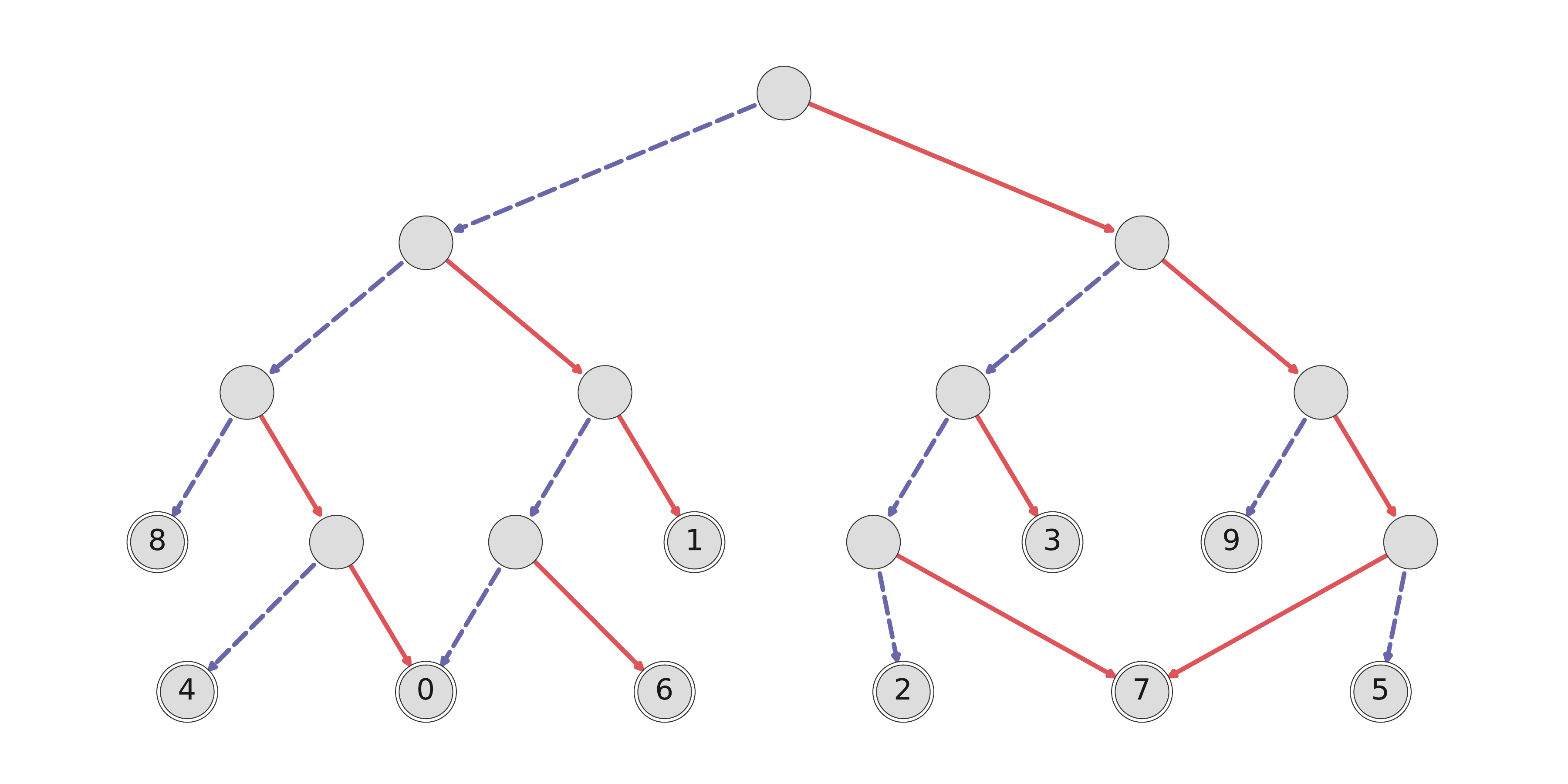}
        \caption{}
        \label{fig:mnist_graph_1}
    \end{subfigure}
    \begin{subfigure}[b]{0.49\textwidth}
        \centering
        \includegraphics[width=\textwidth]{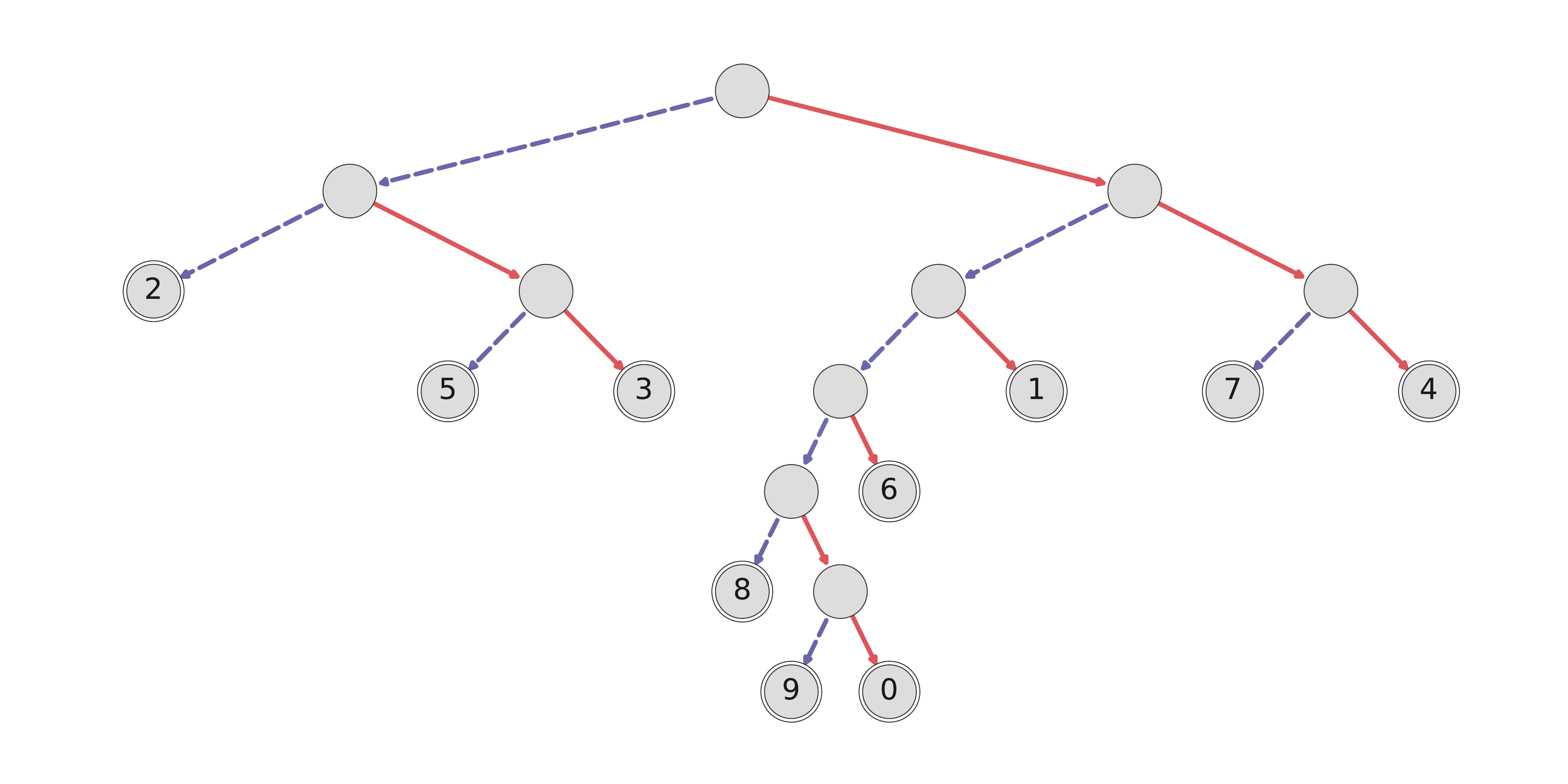}
        \caption{}
        \label{fig:mnist_graph_1}
    \end{subfigure}
    \hfill
    \begin{subfigure}[b]{0.49\textwidth}
        \centering
        \includegraphics[width=\textwidth]{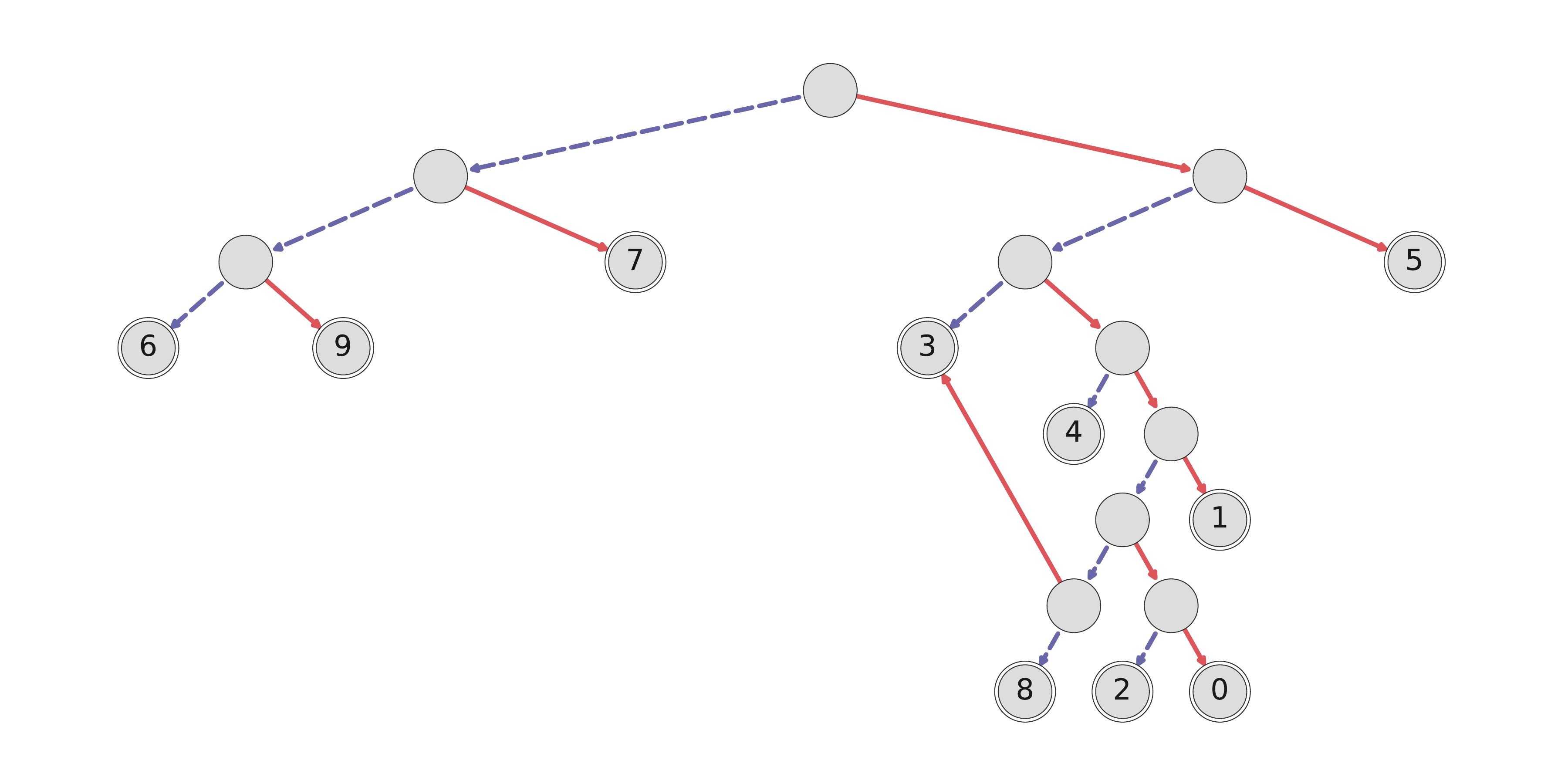}
        \caption{}
        \label{fig:mnist_graph_1}
    \end{subfigure}
    \hfill

    \caption{Examples of the graph structures obtained by training \our{} on the CIFAR10 dataset. The root is the top-most node in each graph, and the leaves are denoted by double node borders. The numbers on the leaves are the CIFAR10 classes. For each node $v_i$, we present two edges corresponding to the highest probability from two transition vectors $m_{\cdot{}i}^0$ and $m_{\cdot{}i}^1$ (represented as dashed blue and solid red arrows, respectively).}
    \label{fig:cifar_graph_sm}
\end{figure*}

\begin{figure*}[ht]
    \centering
    \includegraphics[width=0.49\textwidth]{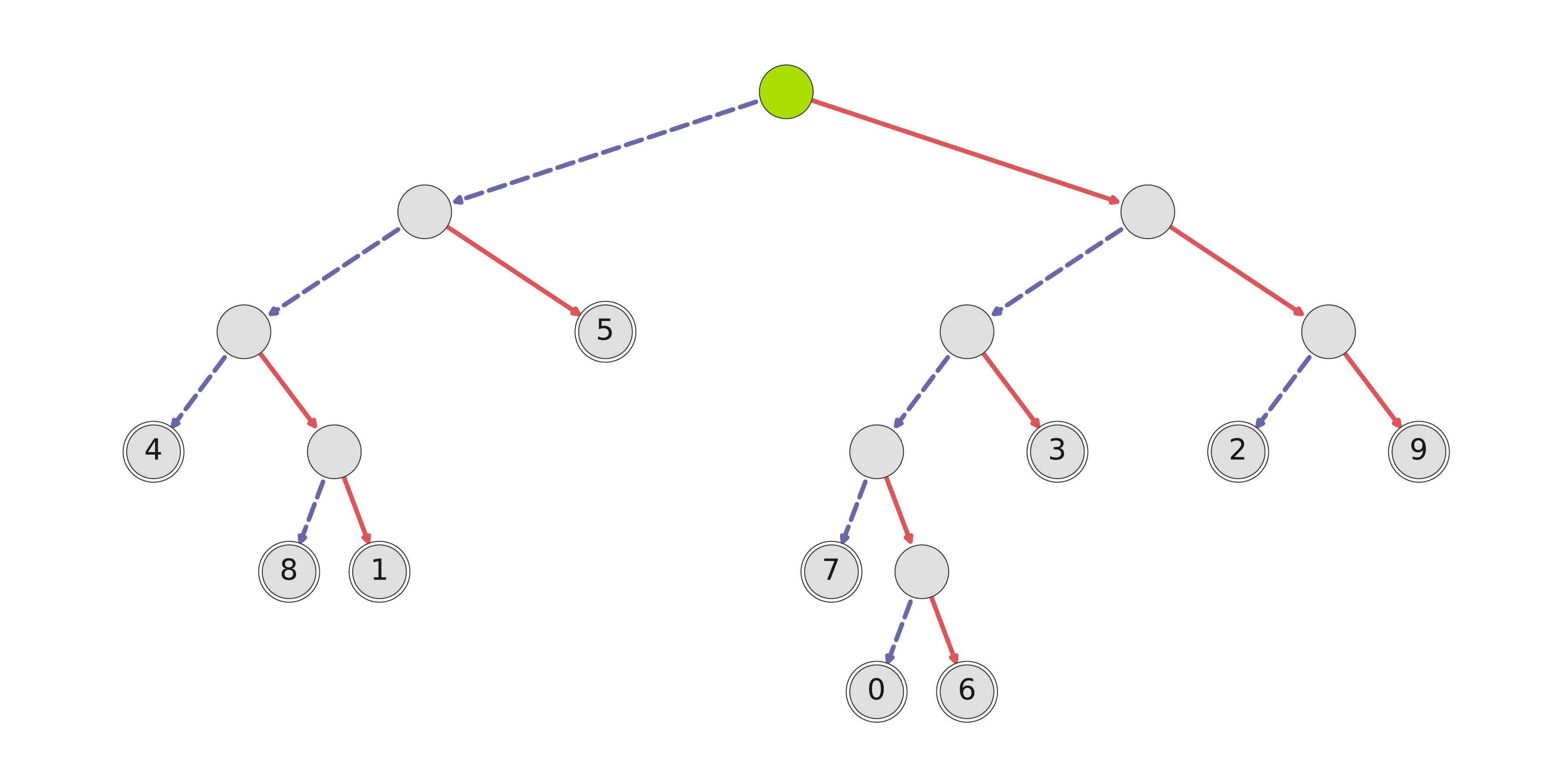}
    \includegraphics[width=0.49\textwidth]{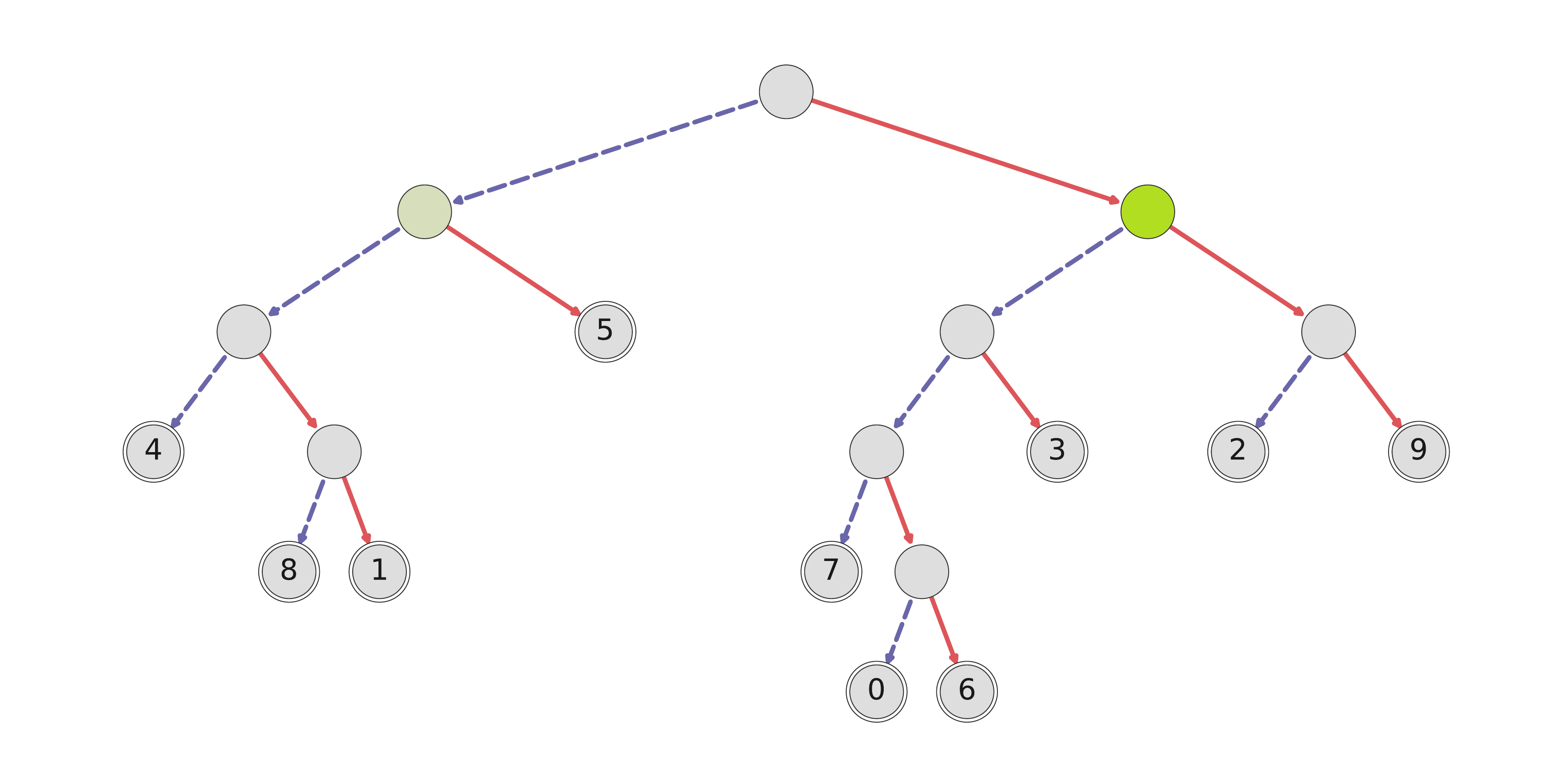}
    \includegraphics[width=0.49\textwidth]{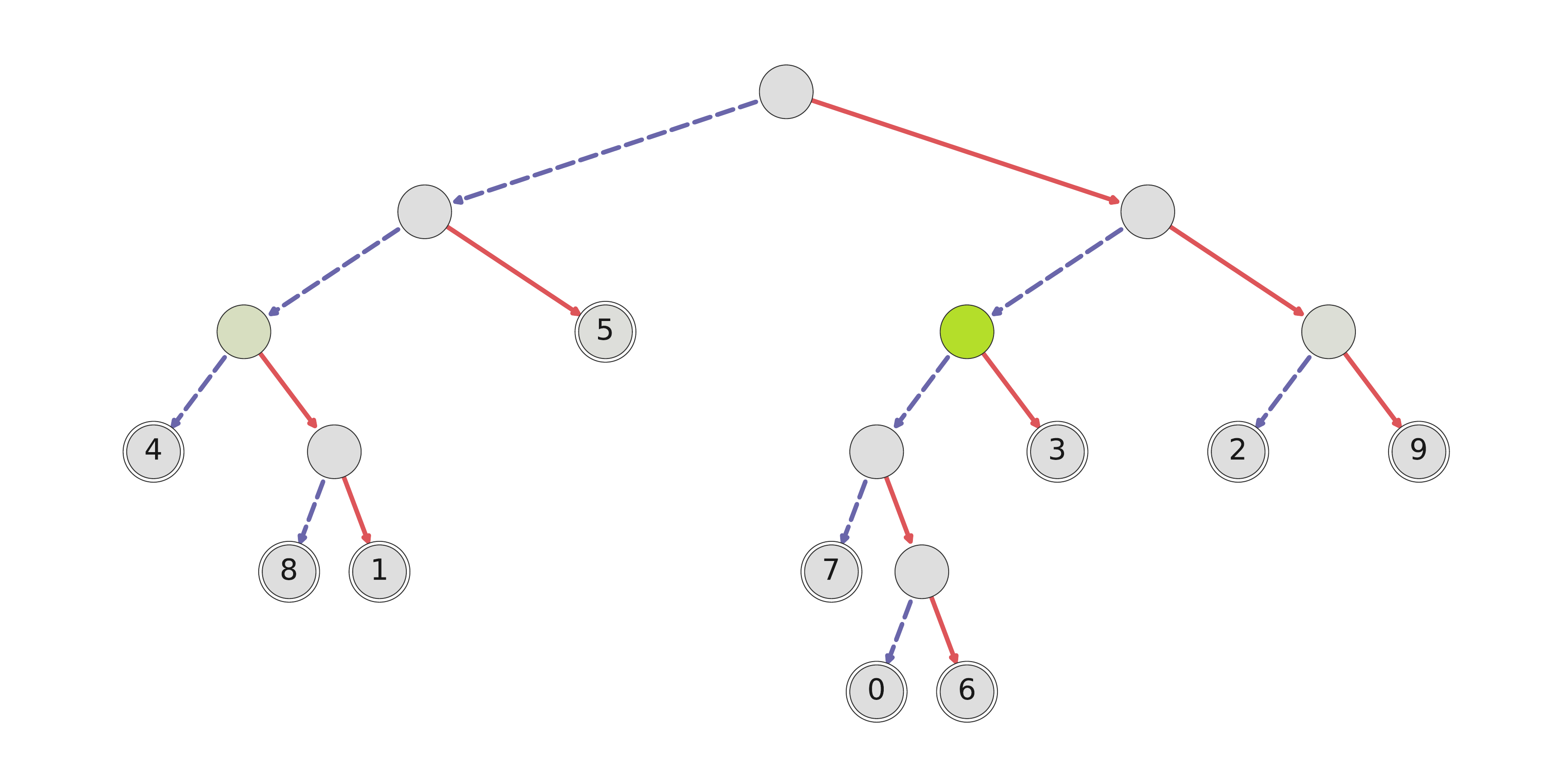}
    \includegraphics[width=0.49\textwidth]{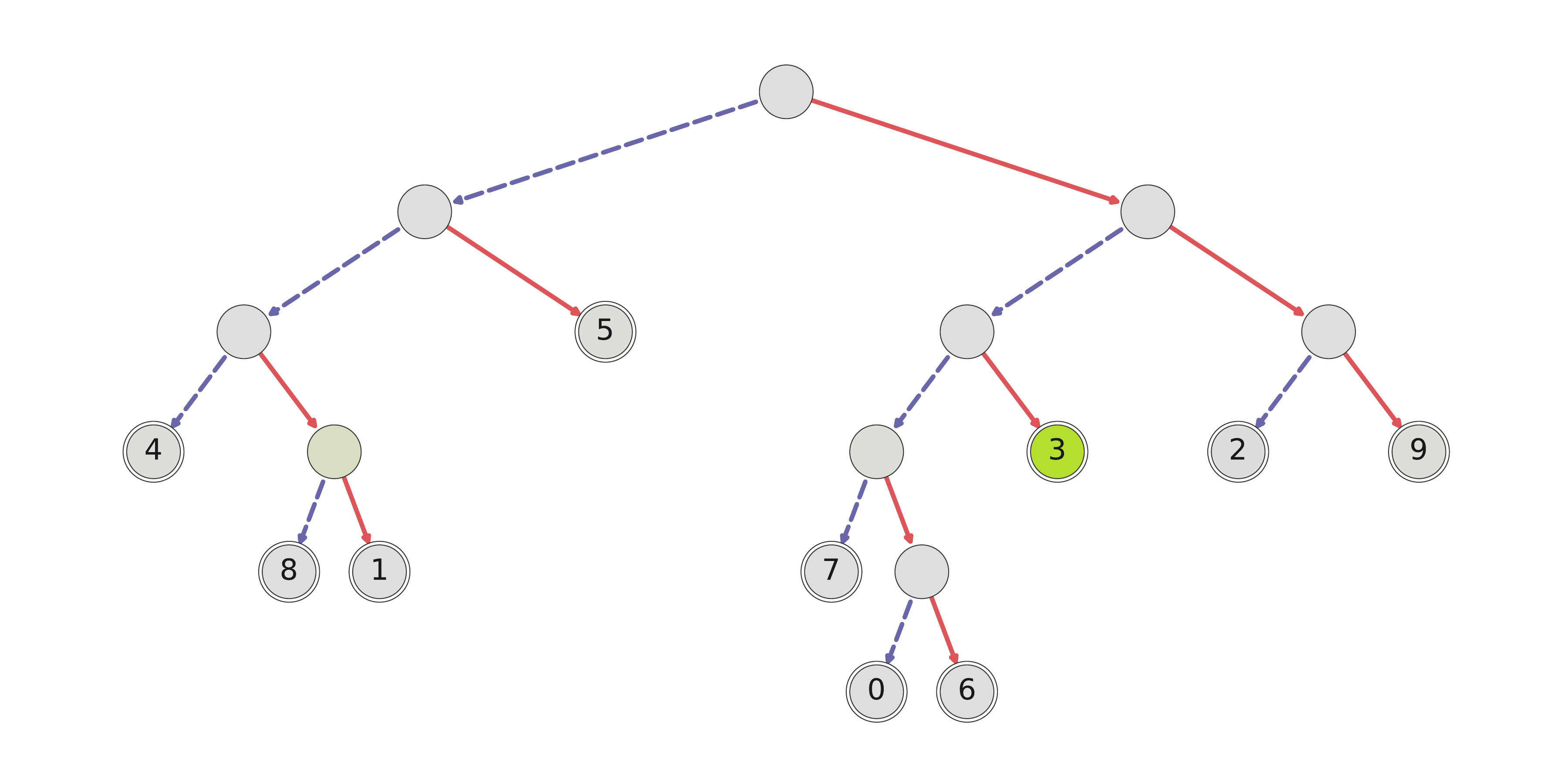}
    \caption{An input image passing through a SONG trained on CIFAR10. High saturation of the green color denotes high probability in the node. Each graph represent a consecutive step of the inference (from left to right, then top to bottom). For each node $v_i$, we present two edges corresponding to the highest probability from two transition vectors $m_{\cdot{}i}^0$ and $m_{\cdot{}i}^1$ (represented as dashed blue and solid red arrows, respectively).}
    \label{fig:steps_1}
\end{figure*}

\begin{figure*}[ht]
    \centering
    \includegraphics[width=0.49\textwidth]{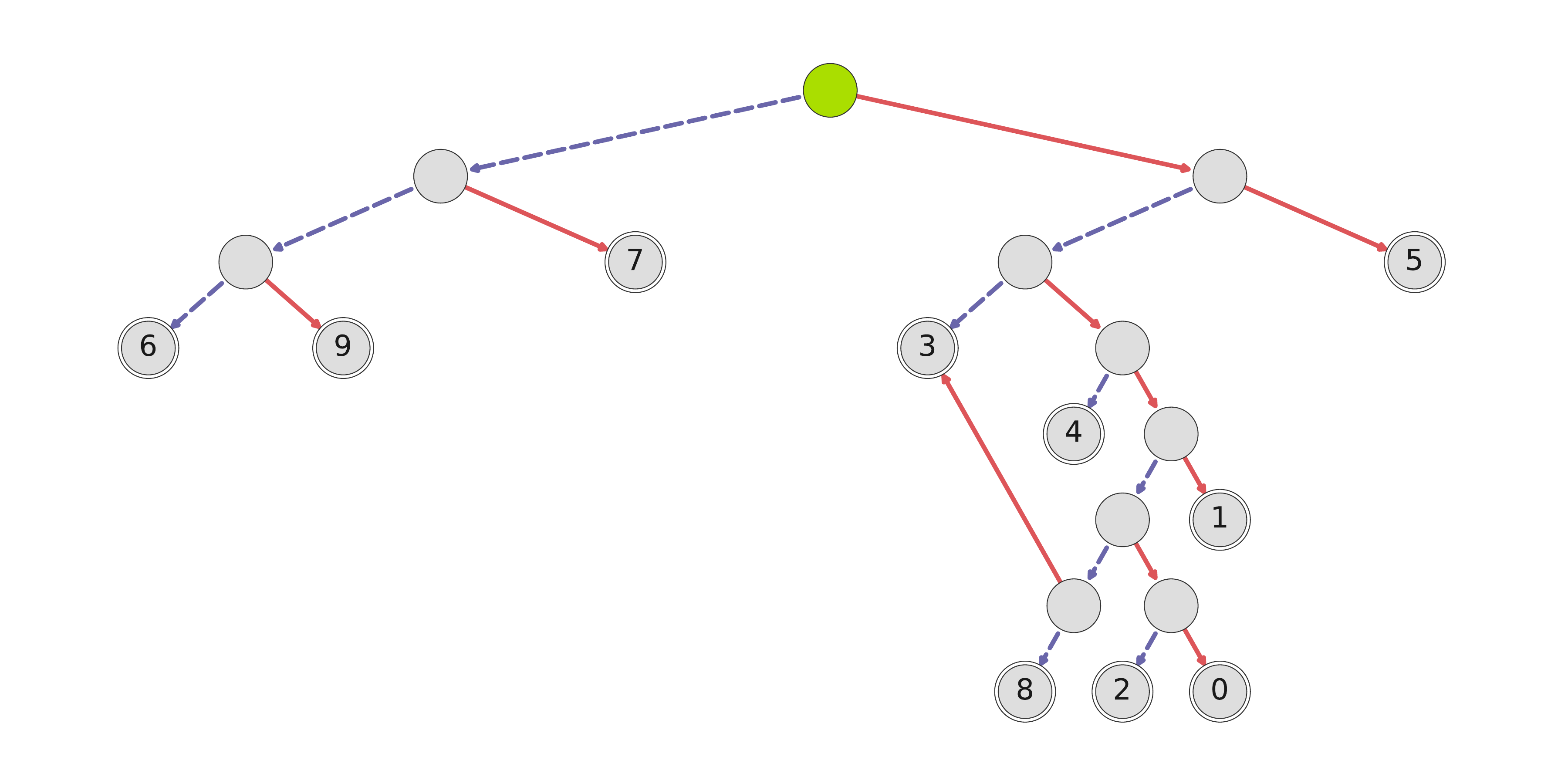}
    \includegraphics[width=0.49\textwidth]{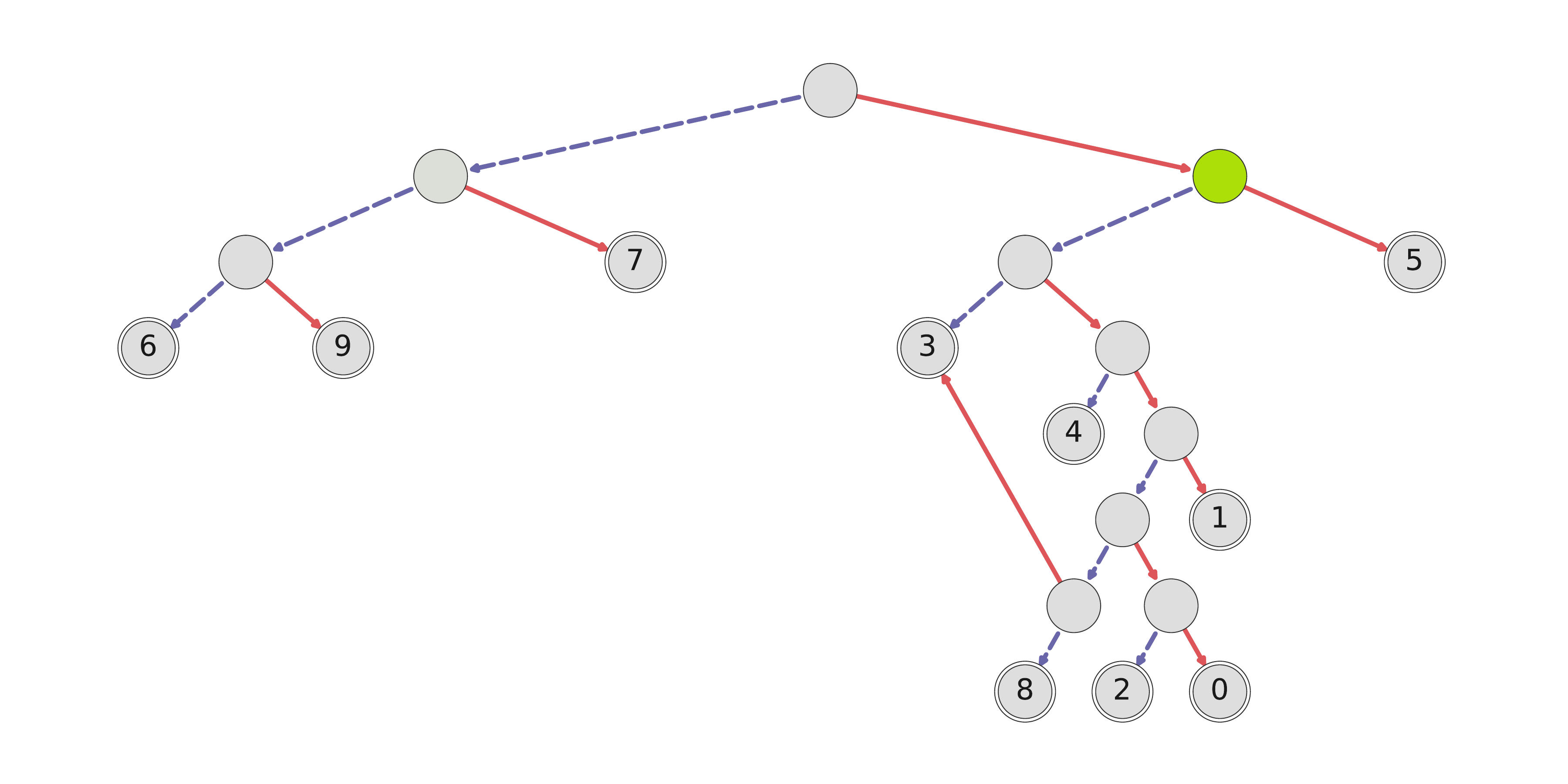}
    \includegraphics[width=0.49\textwidth]{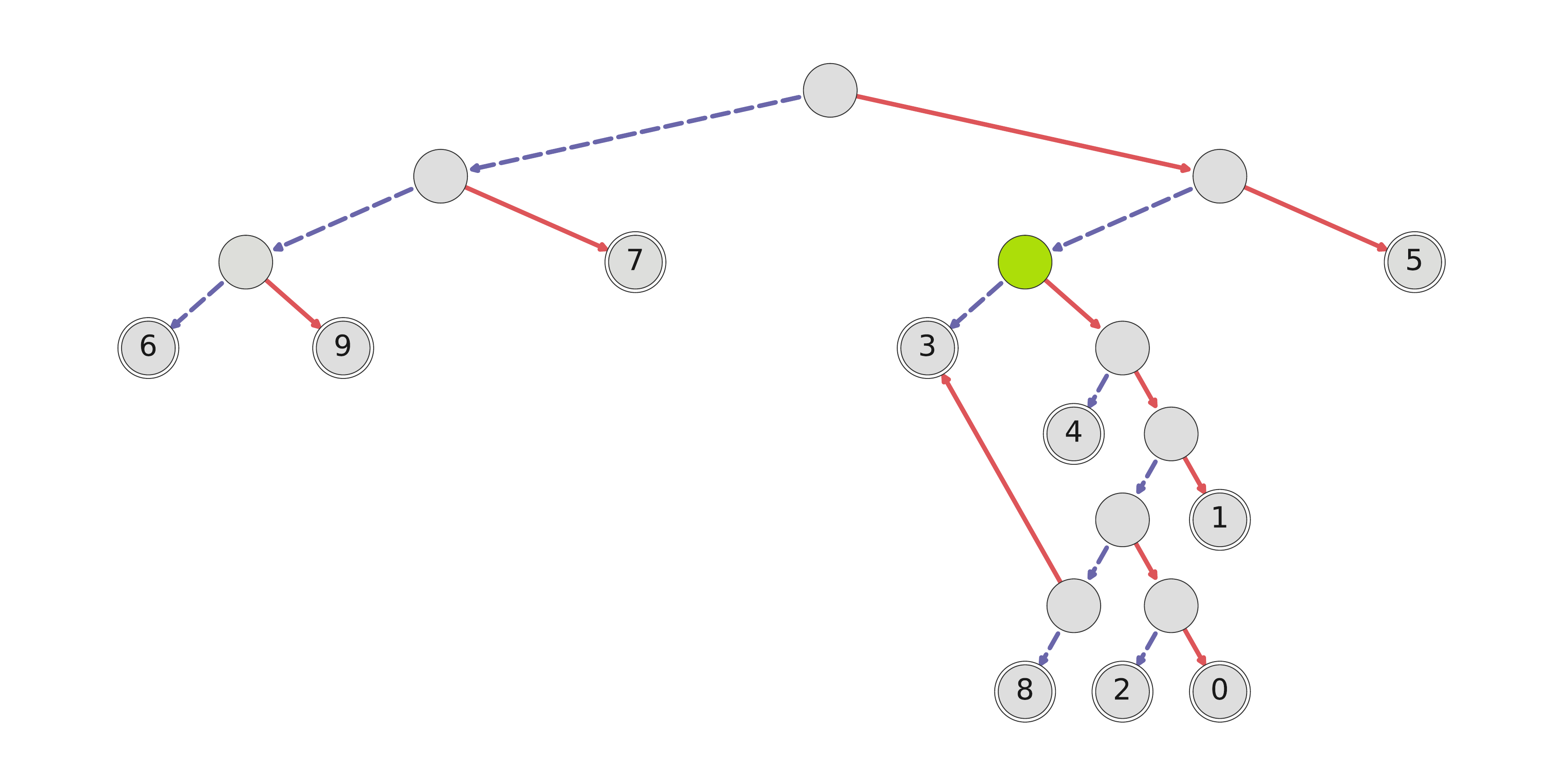}
    \includegraphics[width=0.49\textwidth]{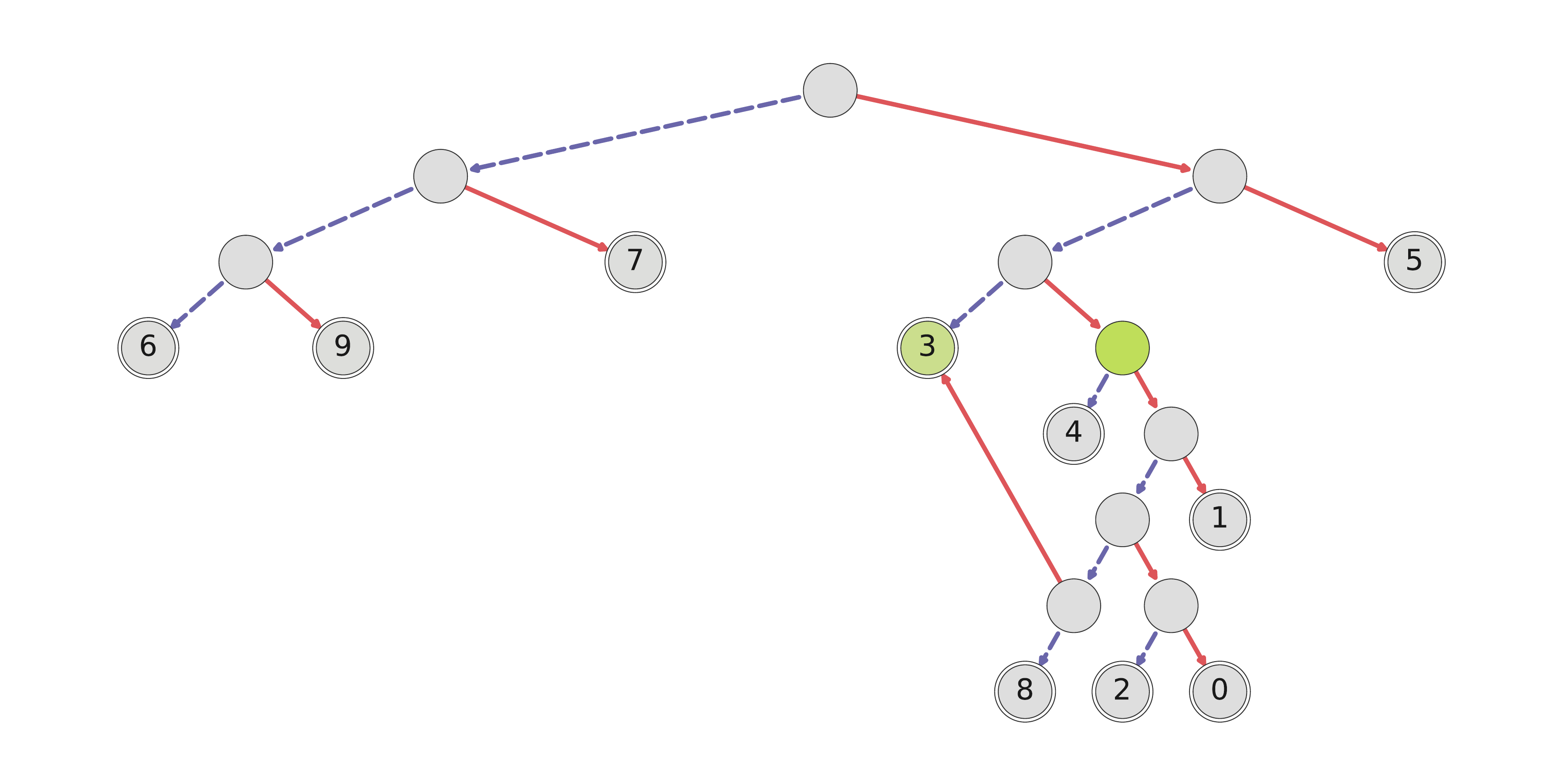}
    \includegraphics[width=0.49\textwidth]{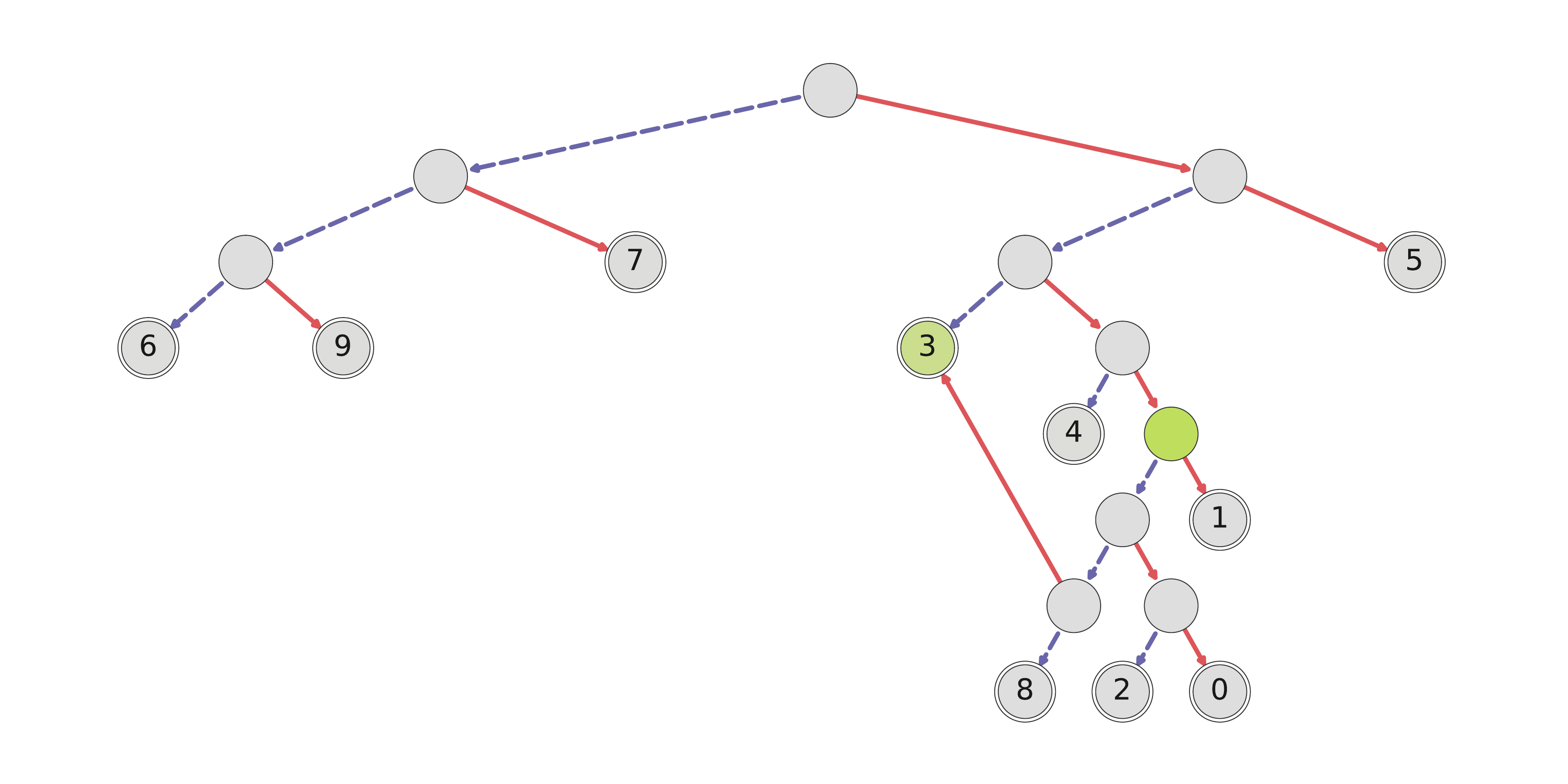}
    \includegraphics[width=0.49\textwidth]{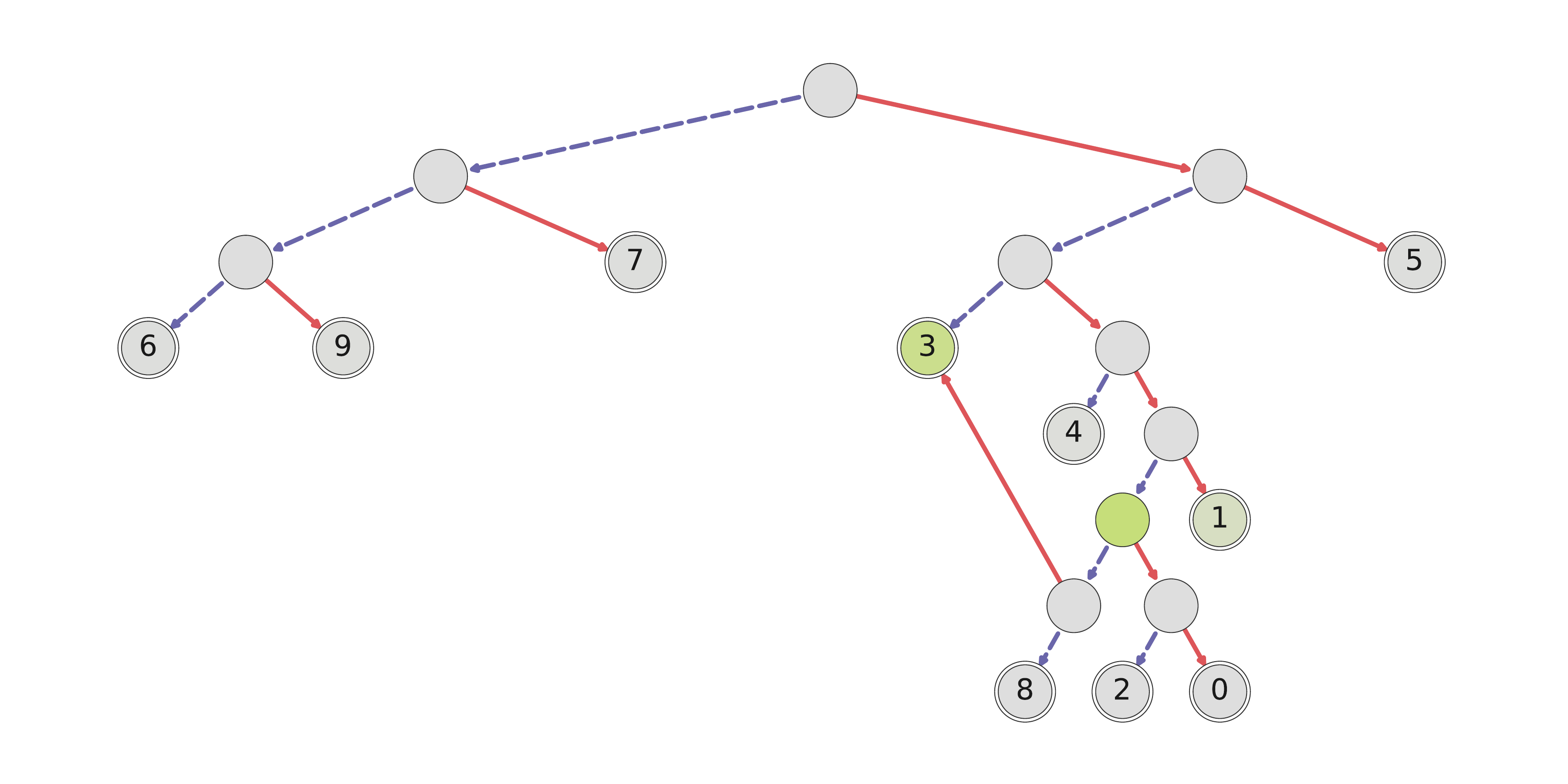}
    \includegraphics[width=0.49\textwidth]{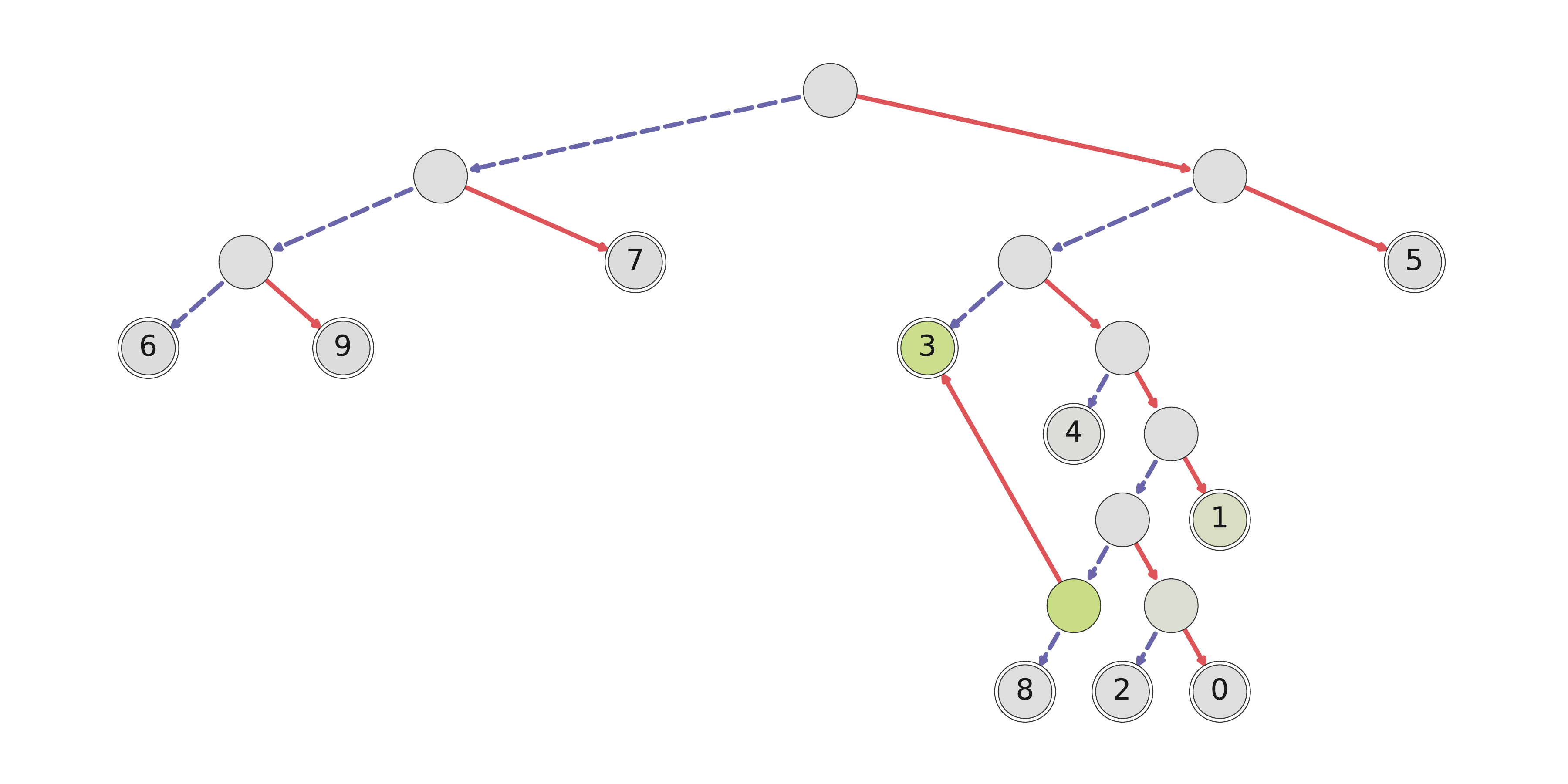}
    \includegraphics[width=0.49\textwidth]{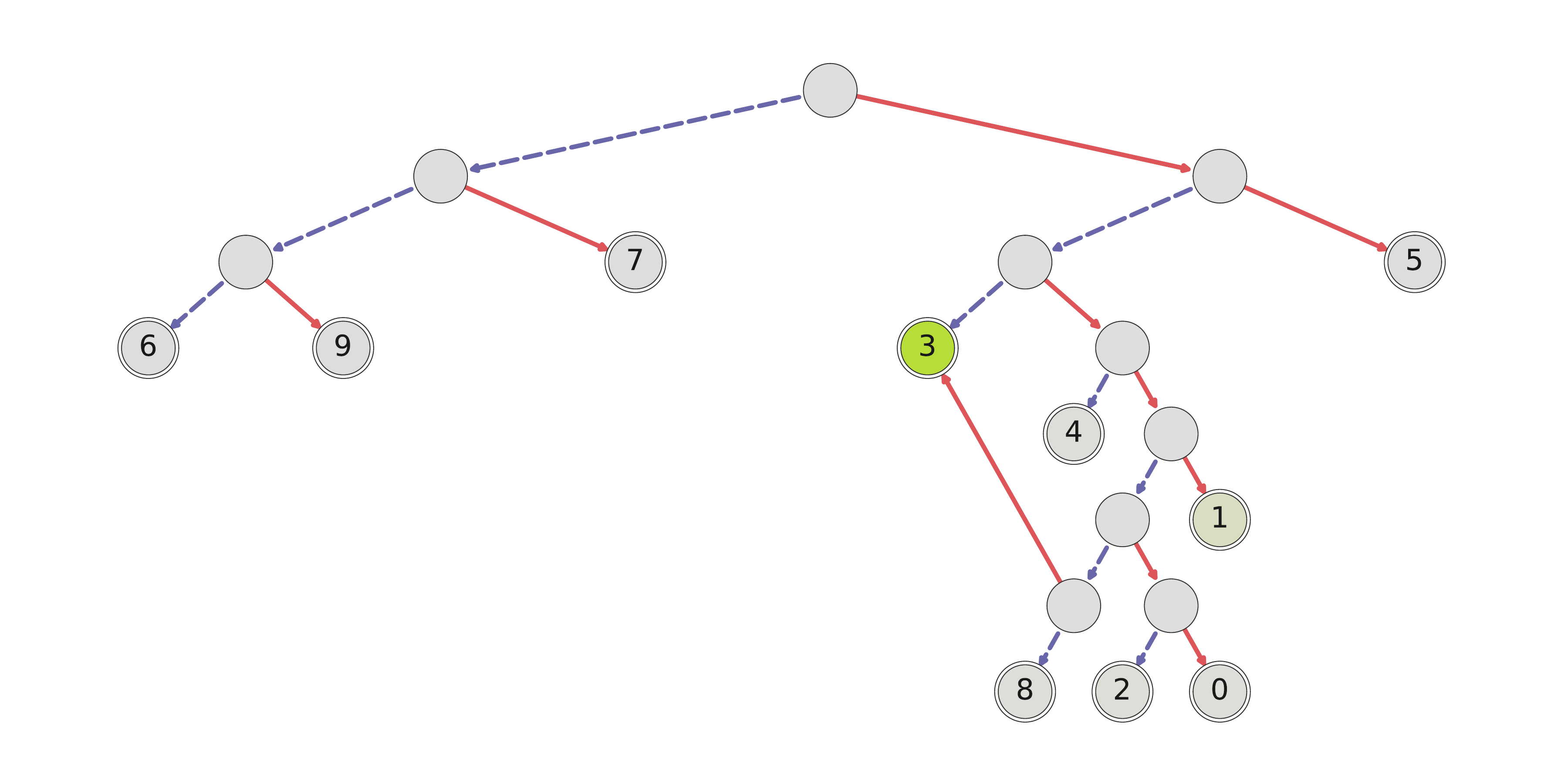}
    \caption{An input image passing through a SONG trained on CIFAR10. High saturation of the green color denotes high probability in the node. Each graph represent a consecutive step of the inference (from left to right, then top to bottom). For each node $v_i$, we present two edges corresponding to the highest probability from two transition vectors $m_{\cdot{}i}^0$ and $m_{\cdot{}i}^1$ (represented as dashed blue and solid red arrows, respectively).}
    \label{fig:steps_3}
\end{figure*}

\begin{figure*}[ht]
    \centering
    \includegraphics[width=0.49\textwidth]{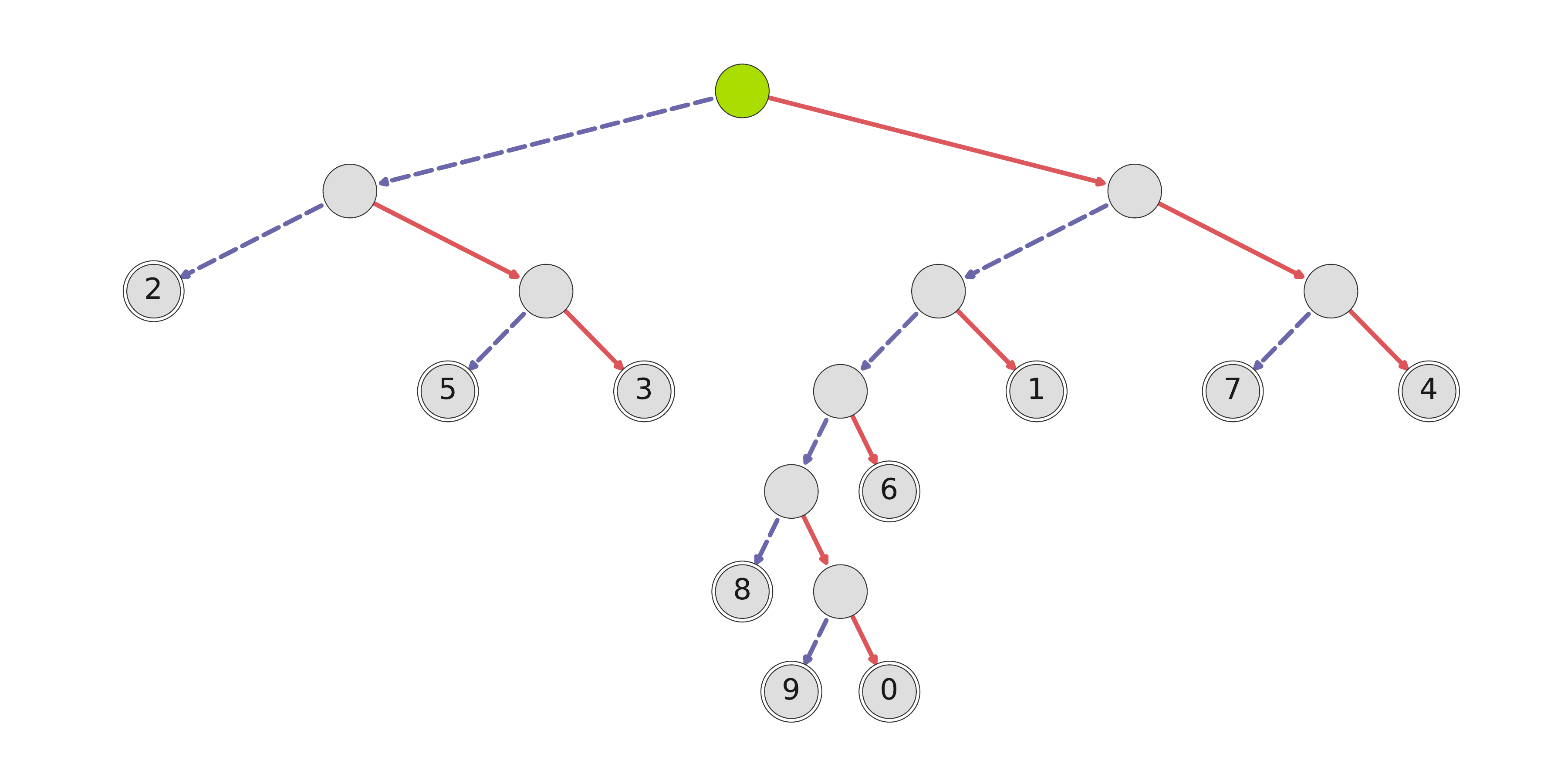}
    \includegraphics[width=0.49\textwidth]{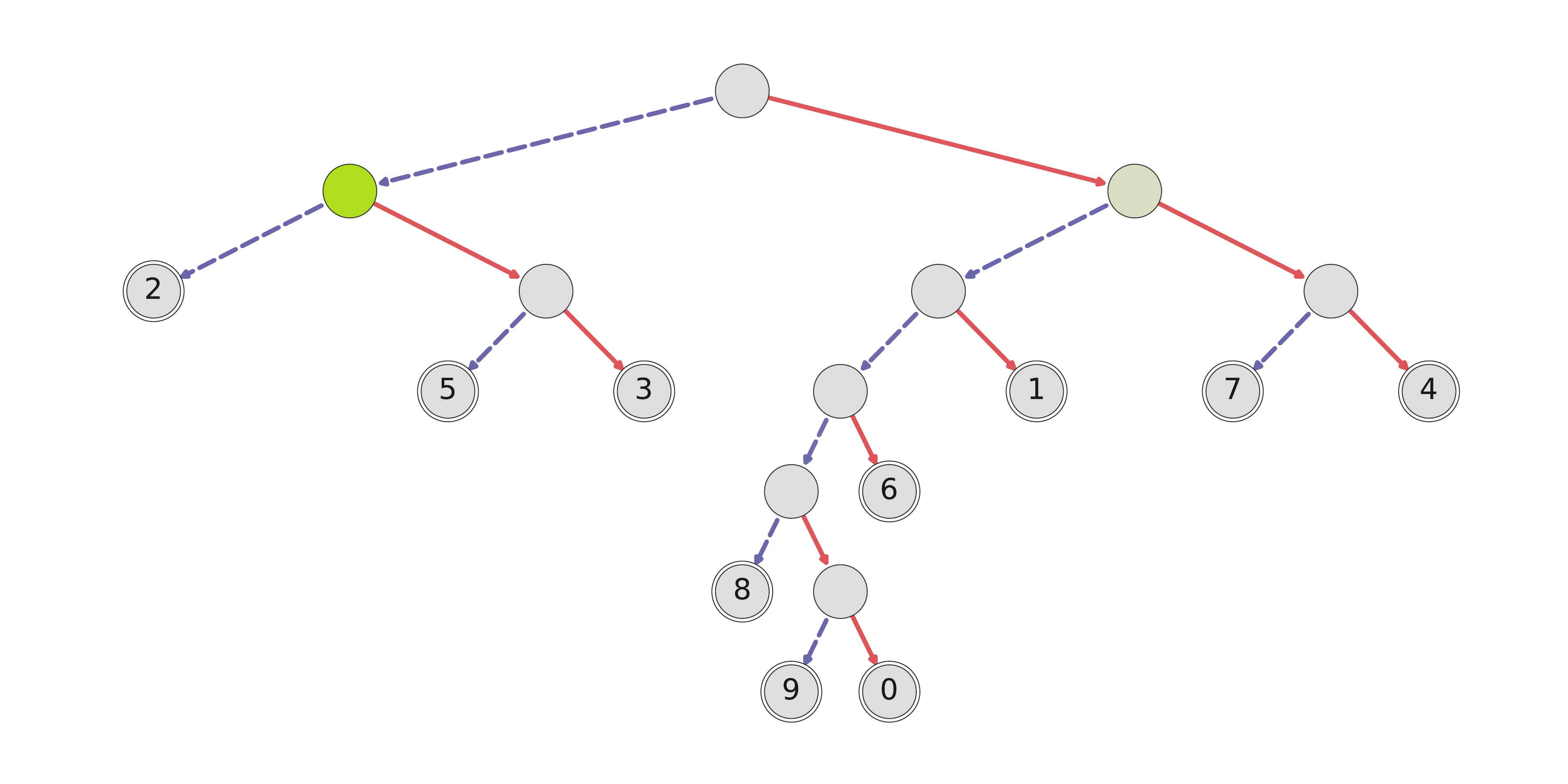}
    \includegraphics[width=0.49\textwidth]{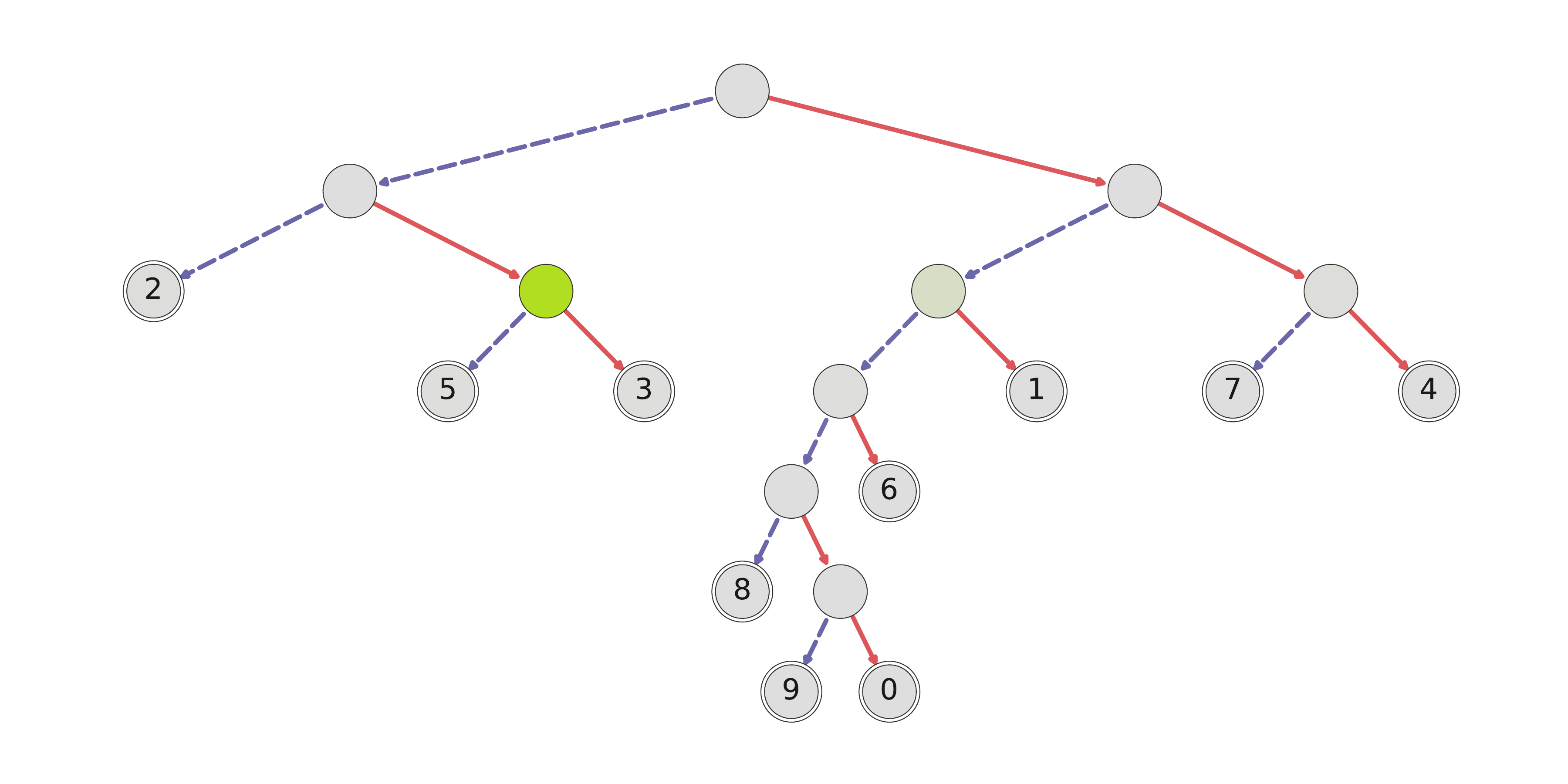}
    \includegraphics[width=0.49\textwidth]{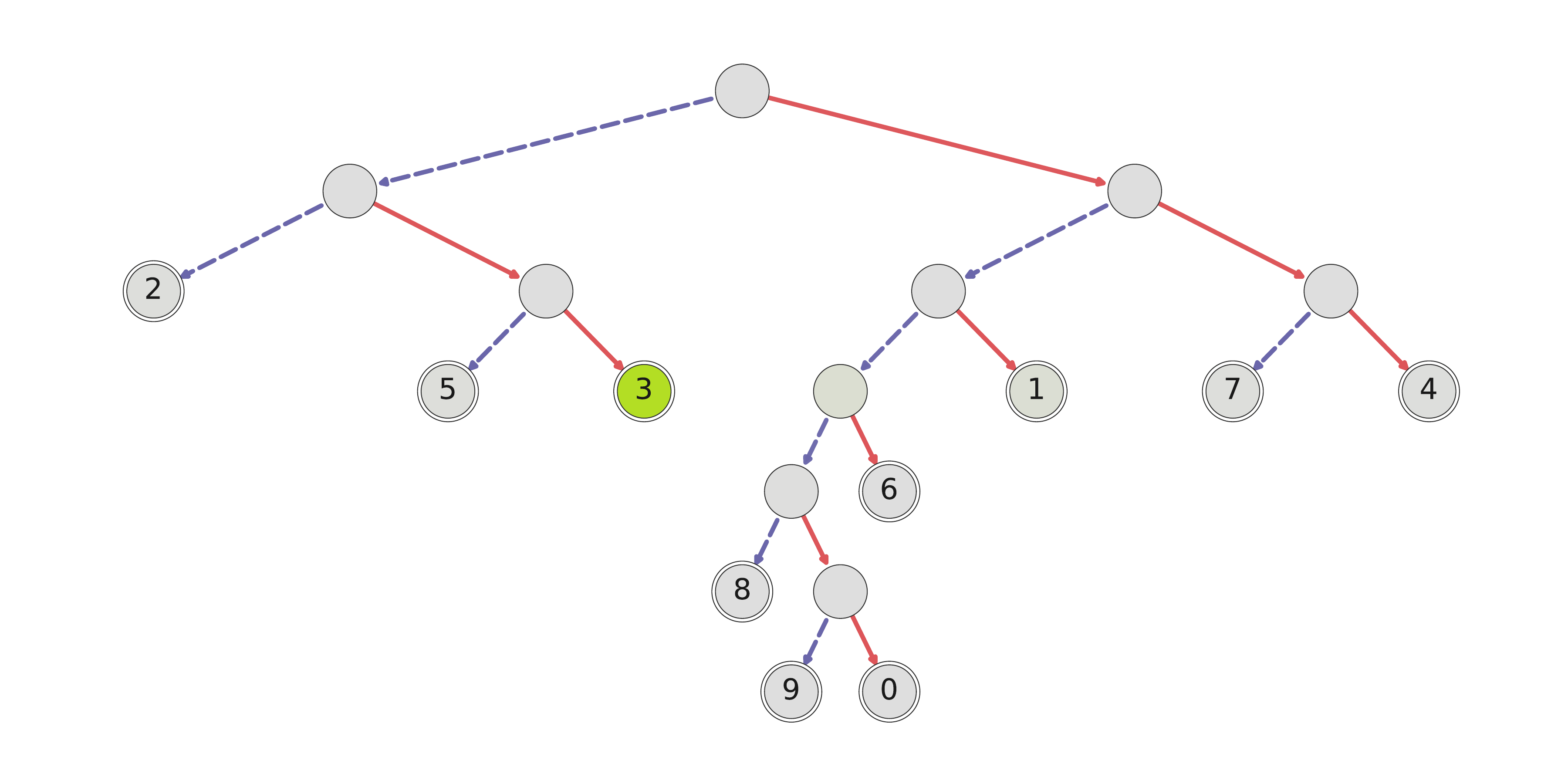}
    \caption{An input image passing through a SONG trained on CIFAR10. High saturation of the green color denotes high probability in the node. Each graph represent a consecutive step of the inference (from left to right, then top to bottom). For each node $v_i$, we present two edges corresponding to the highest probability from two transition vectors $m_{\cdot{}i}^0$ and $m_{\cdot{}i}^1$ (represented as dashed blue and solid red arrows, respectively).}
    \label{fig:steps_2}
\end{figure*}

\begin{figure*}[h]
    \centering
    \includegraphics[width=0.49\textwidth]{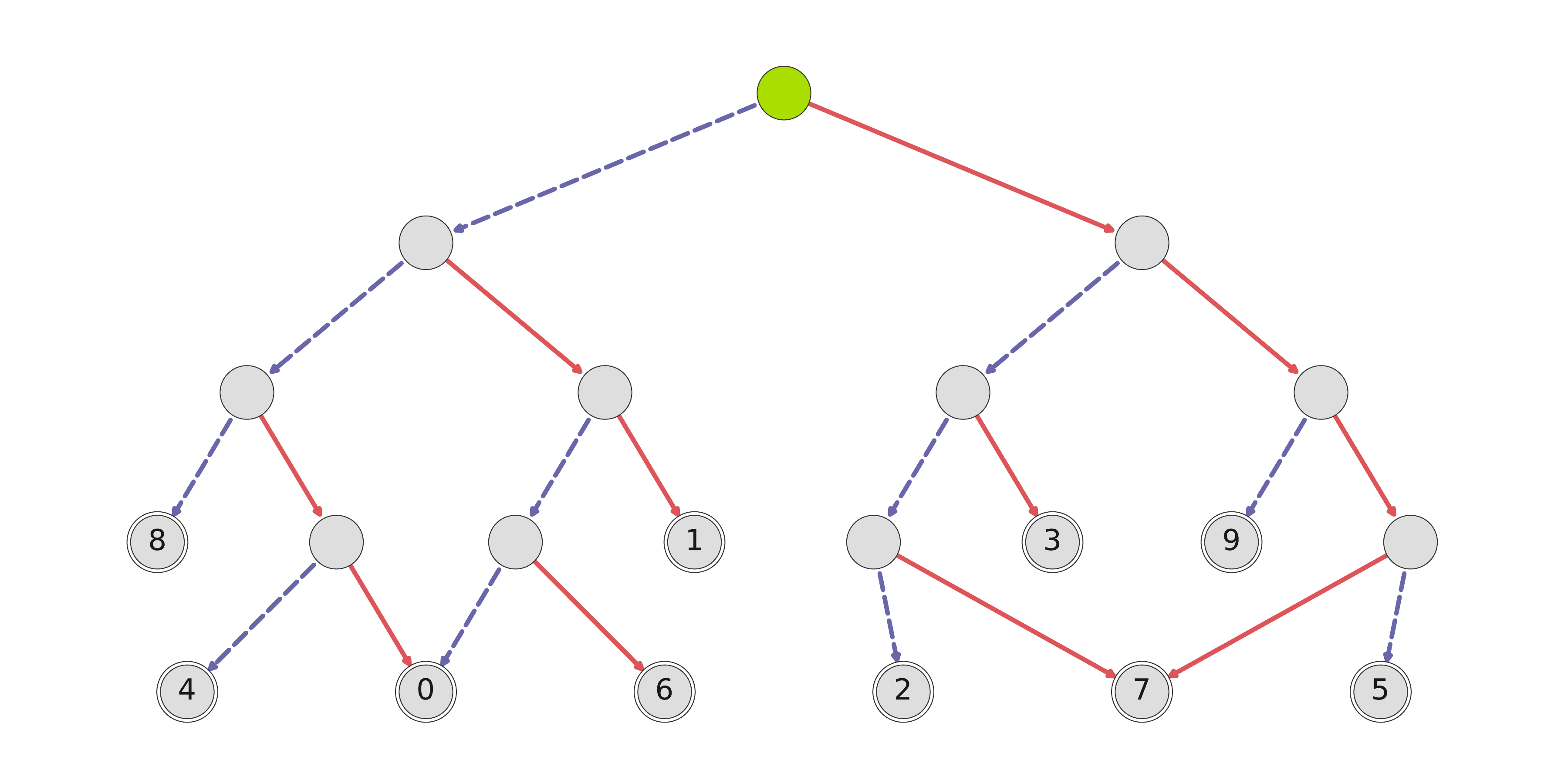}
    \includegraphics[width=0.49\textwidth]{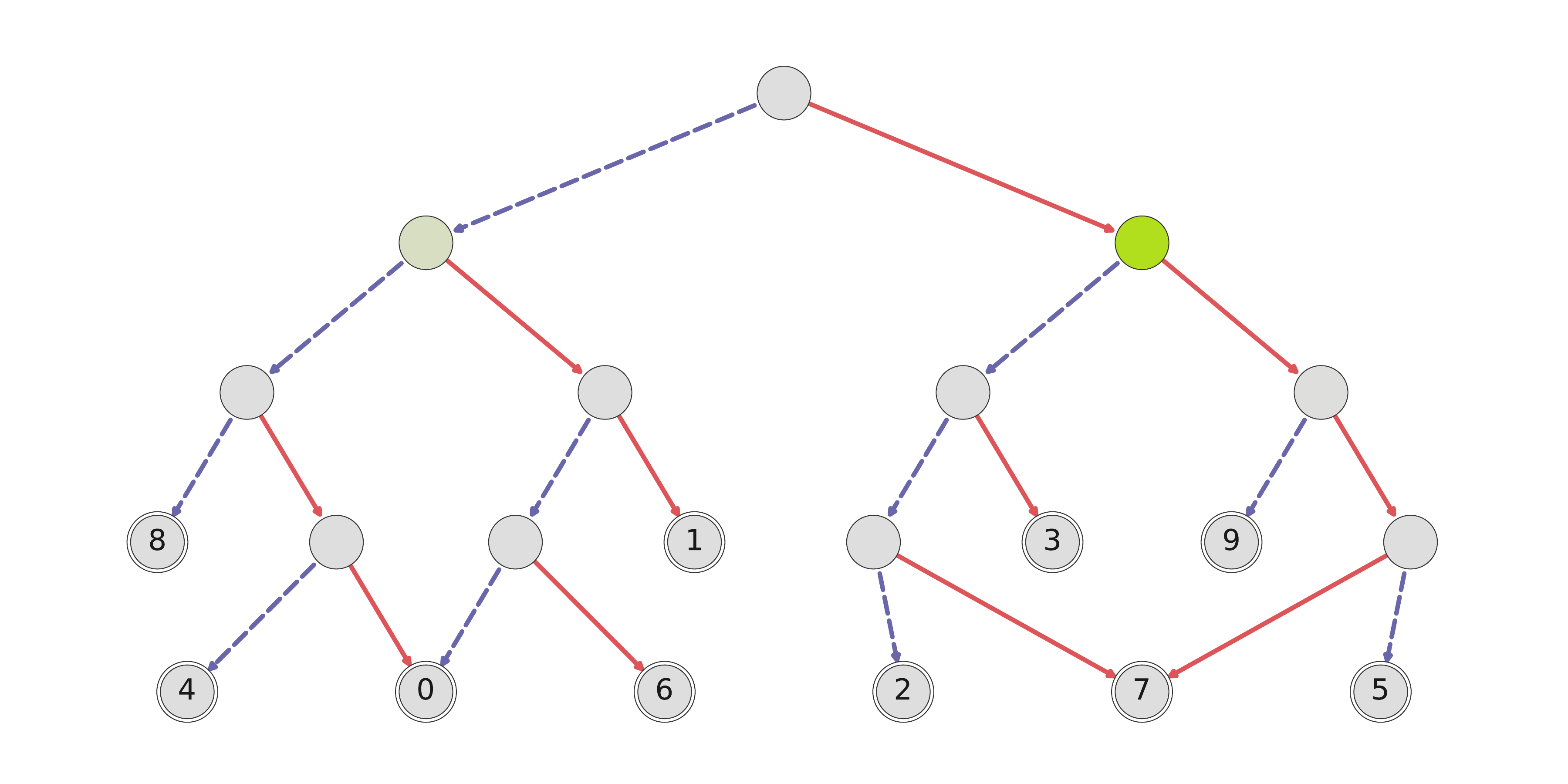}
    \includegraphics[width=0.49\textwidth]{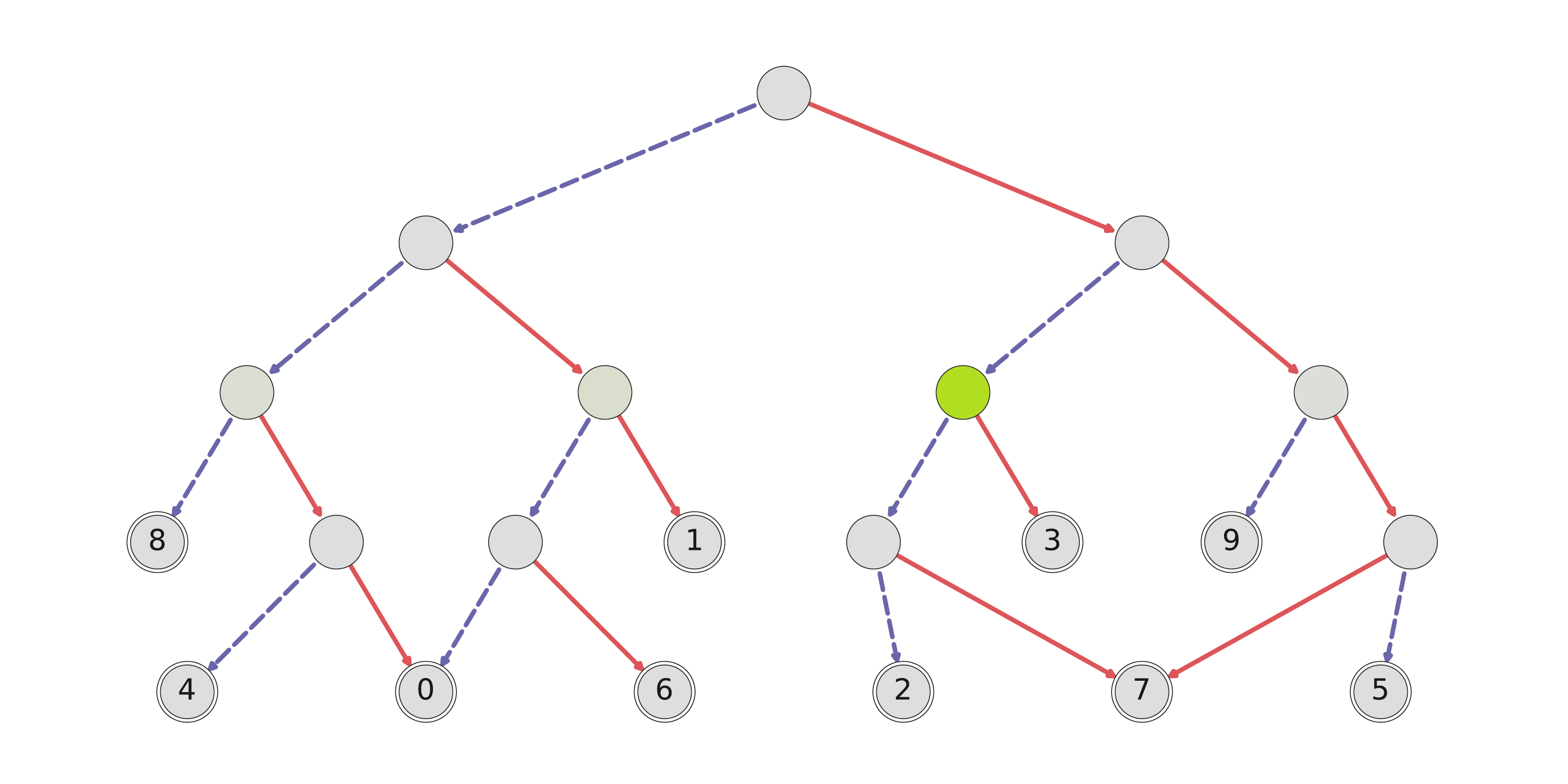}
    \includegraphics[width=0.49\textwidth]{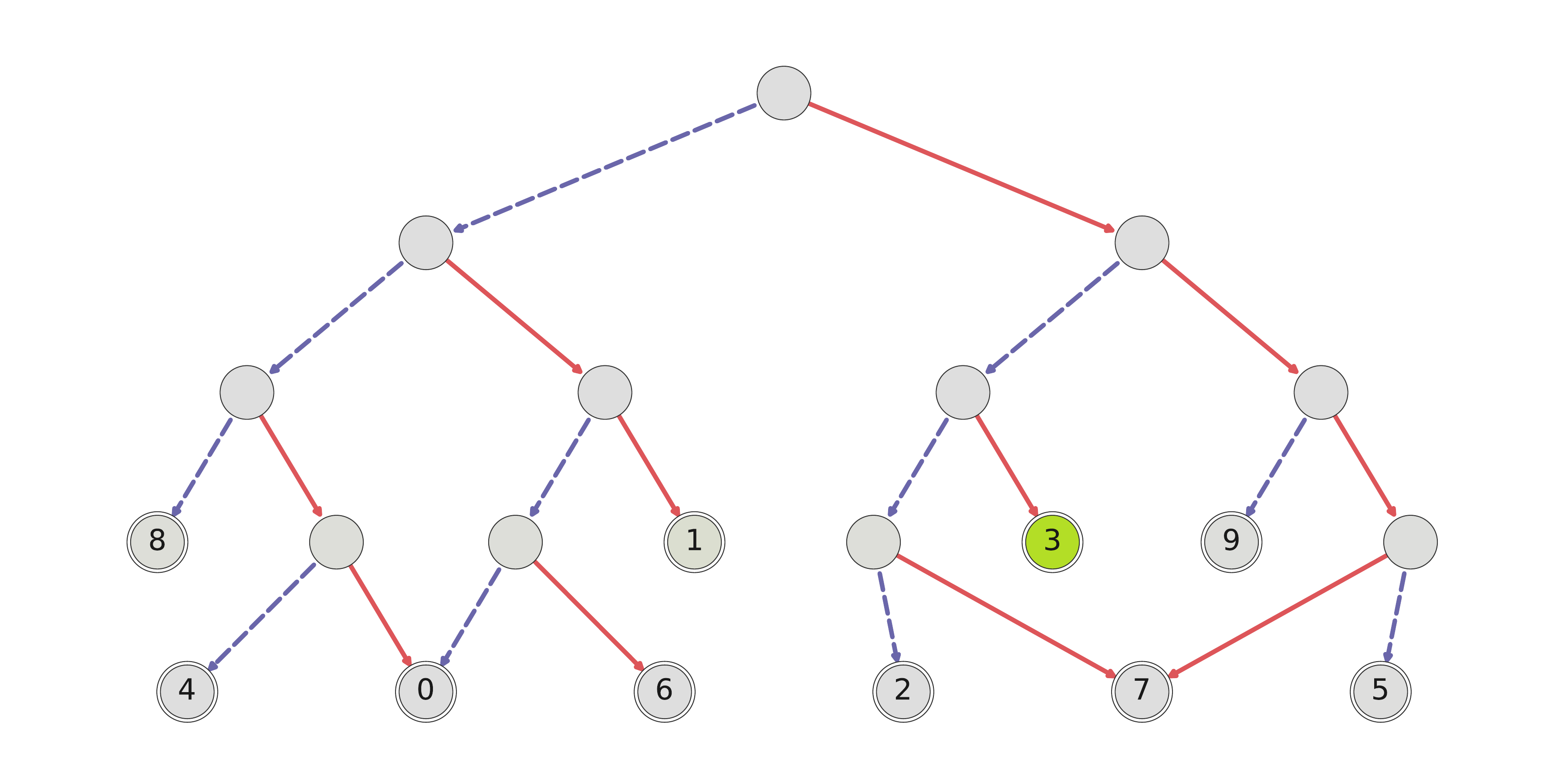}
    \caption{An input image passing through a SONG trained on CIFAR10. High saturation of the green color denotes high probability in the node. Each graph represent a consecutive step of the inference (from left to right, then top to bottom). For each node $v_i$, we present two edges corresponding to the highest probability from two transition vectors $m_{\cdot{}i}^0$ and $m_{\cdot{}i}^1$ (represented as dashed blue and solid red arrows, respectively).}
    \label{fig:steps_4}
\end{figure*}

\section{Theoretical analysis}
\label{sec:theoretical_appendix}

Let us consider SONG as the probabilistic model over trajectories. A trajectory of length $N$, starting at the root of \sbdg{} $G$, is defined as $T=(u_{i_t})_{t=1..N}$ with binary decision $d_t \in \{0,1\}, i_t \in I$, where $I$ denotes the set of node indexes. Thus, our trajectory starts at the root ($i_0=0$) and successively passes through nodes $u_{i_{t_1}},\ldots,u_{i_{t_N}}$. The position of trajectory after time $t$ is defined as $T(t)=u_{i_t}$ and the probability of trajectory $T$ is defined as
$$
\abovedisplayskip=0pt
\abovedisplayshortskip=0pt
\belowdisplayskip=0pt
\belowdisplayshortskip=1pt
\prob{T; G}=\prod_{t=1}^N (\sigma^{d_t}_{i_{t-1}} \cdot m^{d_t}_{i_ti_{t-1}}).
$$
Then the probability of reaching leaf $l$ after $N$ steps with a random trajectory $T$ equals
$
\prob{T(N)=l \, | \, T \sim G},
$
where $T \sim G$ denotes that we sample trajectories with respect to distribution given by $\prob{\cdot; G}$.

Next, we introduce a binarized graph $G$, where we binarize the connections from any pair of nodes. For a fixed $d\in\{0,1\}$, we denote $G[i,j;d]$ as the graph that makes a decision of moving from $u_i$ to $u_j$ with probability 1. 

In the following theorem, we show that if $G$ has no cycles, then we can decompose the probability of its trajectory into the mixture of such binarized graphs.

\begin{theorem} \label{th:1}
Let $G$ be a \sbdg{} where the probability of visiting twice an arbitrary node by a trajectory of length $N$ is zero. Moreover, $u_i$ be an internal node, fixed $d\in\{0,1\}$, and an arbitrary trajectory $T$ of length $N$. Then
\begin{equation} \label{eq:1}
\abovedisplayskip=0pt
\abovedisplayshortskip=0pt
\belowdisplayskip=0pt
\belowdisplayshortskip=1pt
\prob{T;G}=\sum_{j=1}^n m_{ji}^d \prob{T;G[i,j,d]}.
\end{equation}
\end{theorem}

\begin{proof}
Let $T=(u_{i_t})_{t=1..N}$ be a given trajectory and let us consider three cases of passing trough node $u_i$.
First case assumes that $T$ does not pass through $u_i$, i.e. $i \neq i_t$
for $t=1,\ldots,N$. Then, directly from the definition of the trajectory's probability
$$
\prob{T;G}=\prob{T;G[i,j,d]} \text{, for an arbitrary }j.
$$
This completes the proof of \eqref{eq:1} in this case. Hence, let us now consider the cases where the trajectory $T$ passes through node $u_i$.

Suppose the second case, when $T$ passes through $u_i$ more than once. In this case, we will show that both the left and right sides of \eqref{eq:1} are zero. Obviously, $\prob{T;G}=0$ follows directly from the assumption that the probability of visiting twice an arbitrary node by a trajectory of length $N$ in $G$ is zero. Assume, for an indirect proof that there exist $j$ such that $m_{ji}^d\prob{T;G[i,j,d]}>0$. Then $m_{ji}^d>0$. Moreover, if $T$ passes though $u_i$ and makes a decision $d$, then it has to move to $u_j$. In consequence,
$$
\prob{T;G}=m_{ji}^d\prob{T;G[i,j,d]}>0
$$
is a contradiction.

Let us consider the remaining, third case, when $T$ passes through $u_i$ only once, and makes a decision $d$. In other words, there exists a unique $t$ such that $i_t=i$ and $d_t=d$. Observe that if $j=i_{t+1}$ then we move from $u_i$ to $u_j$ and all the probabilities $\prob{T;G[i,l,d]}=0$, for $l \neq j$. Moreover, since $T$ visits $u_i$ only once, we get $\prob{T;G}=m_{ji}^d\prob{T;G[i,j,d]}$, which completes the proof.
\end{proof}

We now show the consequences of the above theorem for the SONG model. For this purpose, we assume that $X=(x_i)_{i=1..K}$ where each $x_i$ is associated with a label $y_i$. We also consider SONG $\G$ trained on $X$ for trajectories of length $N$. Thus for each pair $(x, y)$, we define the probability that a random trajectory of length $N$ reaches leaf corresponding to $y$ as $\prob{T(N)=y \, | T \sim \G_x}$.

In the following theorem, we show that if SONG is trained with zero CE or BCE loss, then no trajectory of length $N$ in $\G_x$ visits the same internal node twice with nonzero probability.

\begin{theorem}\label{th.acyclic}
Let us consider \our{} classifier with $N$ moves and $x$ being a data point with class $y$, such that $\mathrm{loss}\big(\prob{T(N)=y \, | T \sim \G_x}, y\big)=0$. 

Then no trajectory of length $N$ in $\G_x$ visits the same internal node twice with nonzero probability.
\end{theorem}

\begin{proof}
First observe, that directly from the fact that both CE and BCE are non-negative, $\mathrm{loss}\big(\prob{T(N)=y \, | T \sim \G_x}, y\big)=0$ iff 
$$
\prob{T(N)=y \, | T \sim \G_x}=1.
$$

Now suppose that there exists a trajectory $T$ with nonzero probability, which goes through a given internal node $u$ twice, i.e. $T(t_1)=T(t_2)=v \; \text{ for }\; t_1 < t_2$.
Observe that $T(t)$ is not a leaf for $t \in [t_1,t_2]$, since after reaching the leaf, we stay in it. Consider the trajectory $\tilde T$ given by
$$
\tilde T(t)=\begin{cases}
T(t) & \text{if }t \leq t_1, \\
T(t_1+s) & \parbox[t]{4.3cm}{if $t =t_1+l(t_2-t_1)+s, \\ \phantom{if} l\in\mathbb{N}, s \in \{0,..,t_2-t_1\}$}. 
\end{cases}
$$
In other words, this is a trajectory that forms a cycle after reaching $u$. Thus we does not end in a leaf with nonzero probability, which leads to a contradiction.
\end{proof}

The accuracy of $\G$ over set $X$ is defined as the probability of predicting the correct class
$$
\abovedisplayskip=0pt
\abovedisplayshortskip=0pt
\belowdisplayskip=0pt
\belowdisplayshortskip=1pt
\acc(\G;X)=\frac{1}{K} \sum_{i=1}^K
\prob{T(N)=y_i \, | T \sim \G_{x_i}}.
$$
As a direct consequence of Theorem \ref{th:1}, we formulate the following fact.

\begin{theorem} \label{th:b3}
Let $\G$ be a SONG. We assume that for every $x\in X$ no trajectory in $\G_x$ of length $N$ that visits twice the same internal node with nonzero probability. Let a node index $i \in \{1,\ldots,n\}$ and $d \in \{0,1\}$ be fixed. Then 
$$
\abovedisplayskip=0pt
\abovedisplayshortskip=0pt
\belowdisplayskip=0pt
\belowdisplayshortskip=1pt
\acc(\G;X)=\sum_{j=1}^n m_{ji}^d \acc(\G[i,j,d];X).
$$
\end{theorem}

\begin{proof}
By Theorem \ref{th:1}, for an arbitrary point $x \in X$ (with class $y$) and trajectory of length $N$, we have
$$
\abovedisplayskip=0pt
\abovedisplayshortskip=0pt
\belowdisplayskip=0pt
\belowdisplayshortskip=1pt
\prob{T;\G_x}=\sum_{j=1}^n m_{ji}^d \prob{T;\G_x[i,j,d]}.
$$
In consequence,
\begin{align*}
& \prob{T(N)=y \, | T \sim \G_x} = \\
& = \sum_{j=1}^n m_{ji}^d
\prob{T(N)=y \, | T \sim \G_x[i,j,d]}.
\end{align*}
Averaging the above probability over all points from $X$ and applying the definition of accuracy, we obtain the assertion of the theorem.
\end{proof}

Observe that the above theorem implies that if we discretize connections in the graph by applying formula \eqref{eq:jj} (below), then we do not decrease the accuracy of the model (statistically, we increase it):

\begin{theorem}
Let $\G_x$ be SONG generated for $x\in X$ with CE or BCE loss equals zero. Moreover, let node index $i \in I$ and $d \in \{0,1\}$ be fixed, and
\begin{equation} \label{eq:jj}
j=\argmax_{\tilde j} \acc(\G[i,\tilde j,d];X).
\end{equation}
Then 
$$
\abovedisplayskip=0pt
\abovedisplayshortskip=0pt
\belowdisplayskip=0pt
\belowdisplayshortskip=1pt
\acc(\G;X) \leq \acc(\G[i,j,d];X).
$$
\end{theorem}

\begin{proof}
From Theorem~\ref{th.acyclic} we obtain that $\G_x$ is SONG generated for $x\in X$ with no trajectory of length $N$ that visits twice the same point with nonzero probability. Theorem~\ref{th:b3} implies that if we discretize connections in the graph by applying formula \eqref{eq:jj}, then we do not decrease the accuracy of the model.
\end{proof}

\section{Ablation study on leaves regularization}

In Figure~\ref{fig:reg_ablation}, we present a comparison between SONG trained on MNIST dataset with (a) and without $L_{leaves}$ regularization (b). The accuracy and BCE loss reported at the final stage of training are similar for both models. However, there are significant differences between their convergence times. Most interestingly, models with regularization hold $L_{leaves}$ close to $0$ during the whole training, so the sum of probability in the leaves is close to $1$ all the time. On the other hand, the models without regularization have an increased value of $L_{leaves}$ between $50$ and $150$ epoch, meaning that the leaves are not reached for some of the input samples. Such behavior can be especially detrimental for larger datasets that require more training epochs to converge.

\begin{figure*}[t]
    \centering
    \includegraphics[width=.4\textwidth]{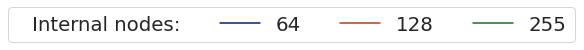} \\
    \begin{subfigure}[b]{.49\textwidth}
        \centering
        \includegraphics[width=\textwidth]{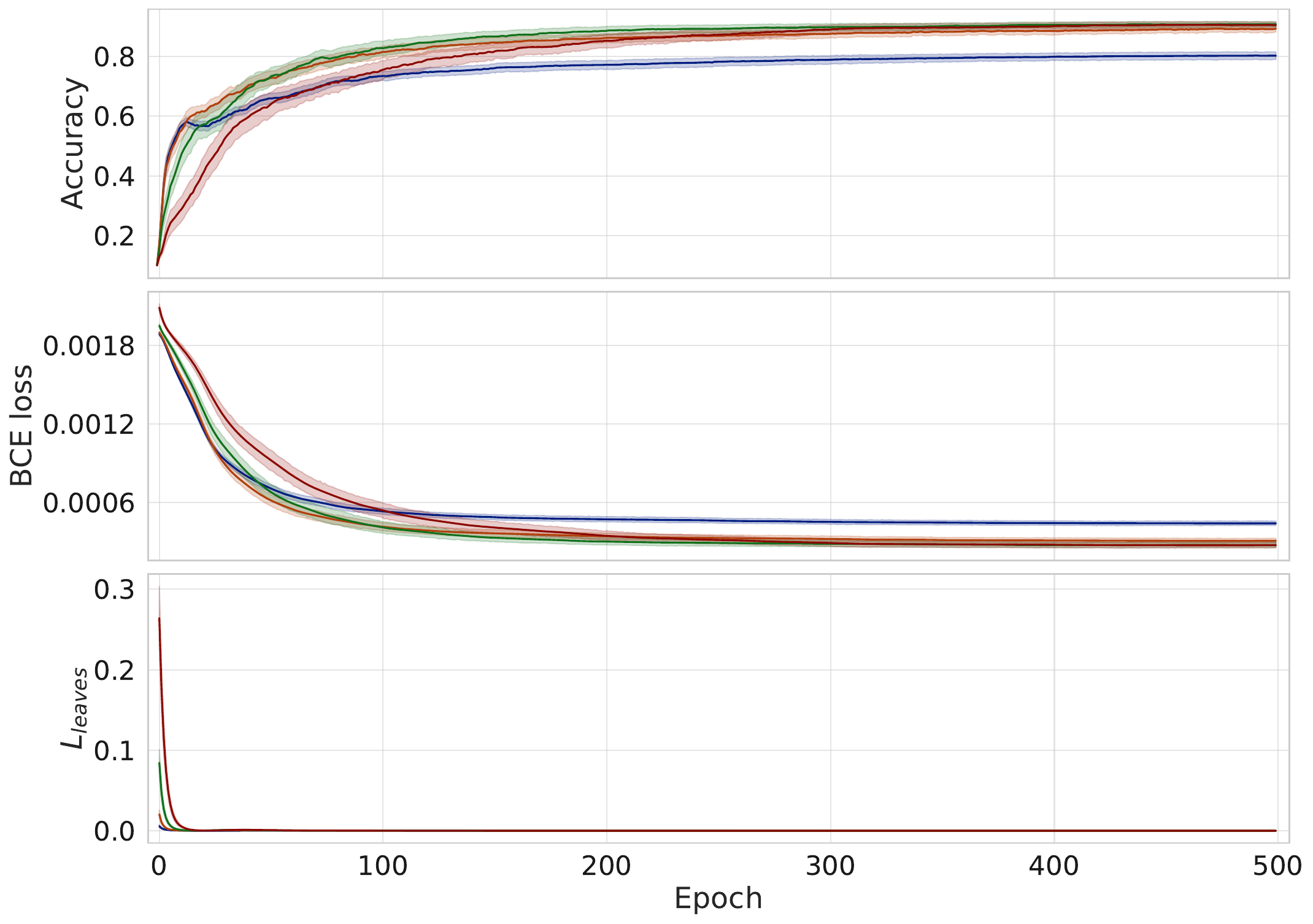}

        \caption{SONG trained with $L_{leaves}$ regularization.}
        \label{fig:loss}
    \end{subfigure}
    ~
    \begin{subfigure}[b]{.49\textwidth}
        \centering
        \includegraphics[width=\textwidth]{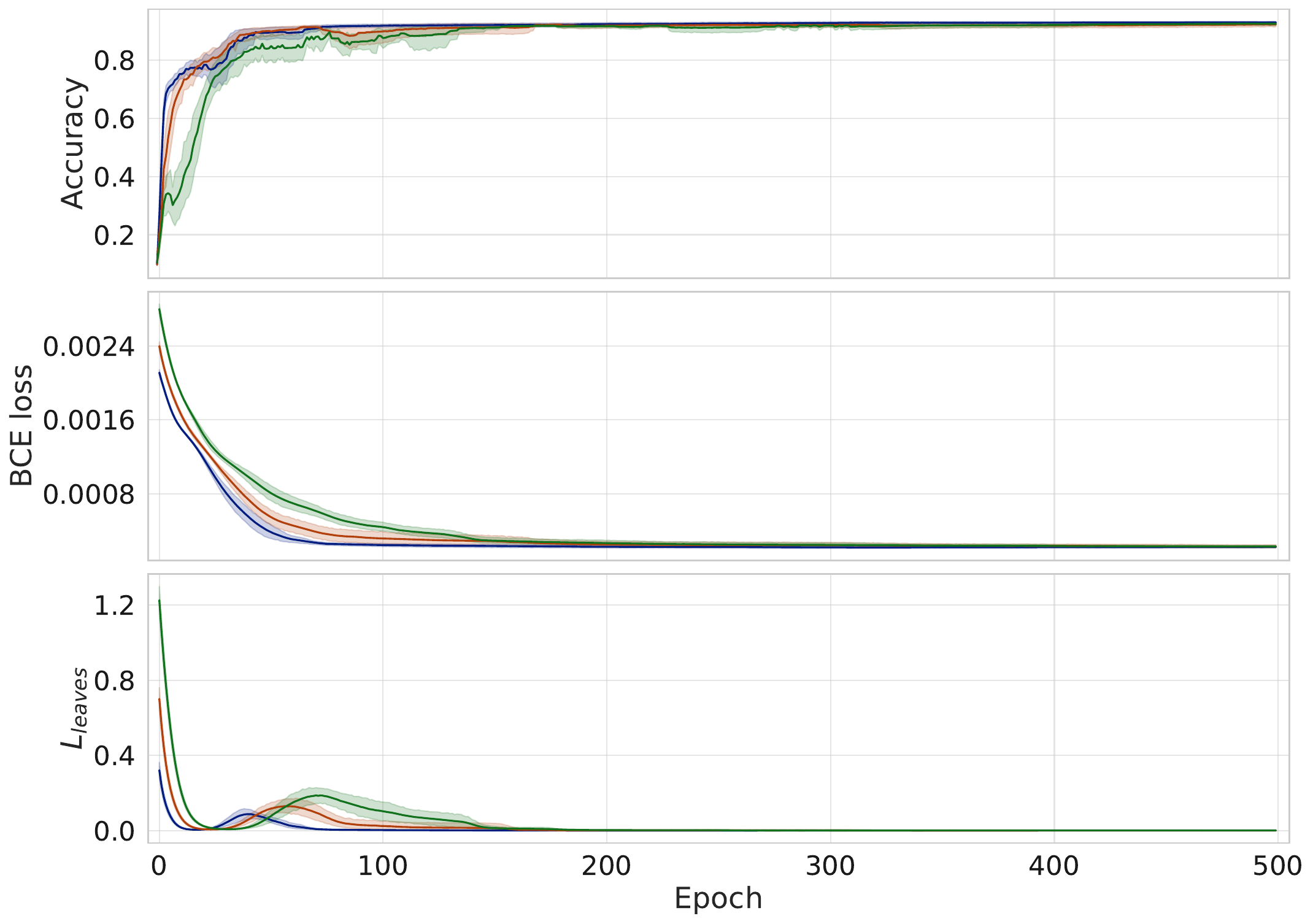}

        \caption{SONG trained without $L_{leaves}$ regularization.}
        \label{fig:loss_leaves}
    \end{subfigure}
    
    \caption{Accuracy, BCE loss, and $L_{leave}$ in the successive training epochs of SONG trained on the MNIST dataset. Each color represents a different number of internal nodes (64, 128, 255), and each line corresponds to mean and standard deviation over multiple training repetitions.}
    \label{fig:reg_ablation}
\end{figure*}

\section{Nodes and edges statistics}

Here, we show the nodes and edges statistics calculated for SONGs trained on the MNIST dataset (see Figure~\ref{fig:mnist_stat}, respectively). It is discussed in the article.

\begin{figure*}[tb]
    \centering
    \includegraphics[width=.9\textwidth]{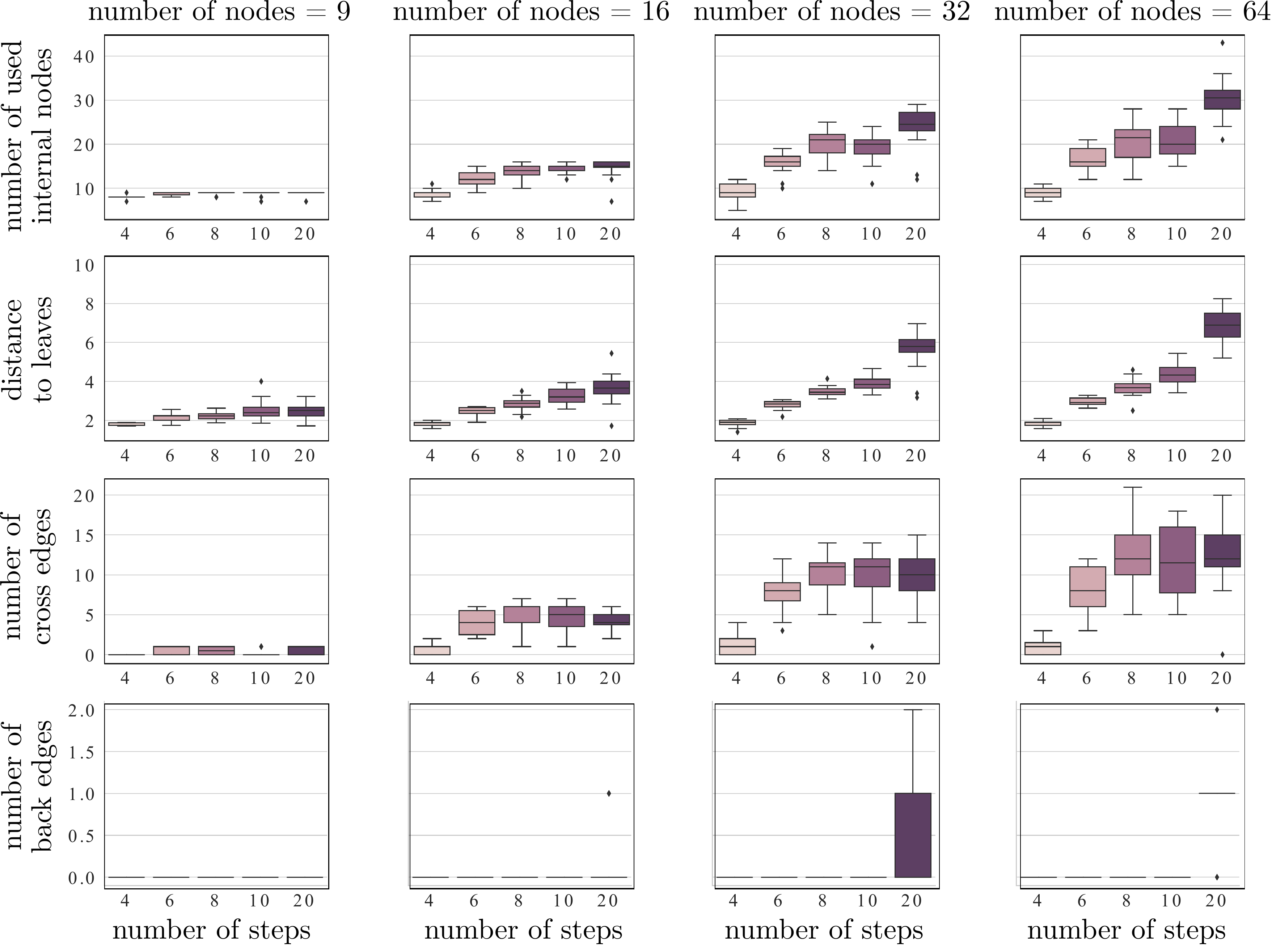}
    \caption{Nodes and edges statistics calculated for SONGs trained on the MNIST dataset. For each combination of the number of internal nodes and steps, 20 graphs are trained and used to plot the distributions of four statistics. One can observe a significant difference in SONG structure depending on those hyperparameters.}
    \label{fig:mnist_stat}
\end{figure*}

\section{Additional results}

Tables~
\ref{tab:fine}, and \ref{tab:shallow} show the relationship between the number of nodes and steps and prediction accuracy.

\begin{table}[ht]
    \centering
    \begin{subtable}{\linewidth}
        \centering
        \begin{tabular}{rrlll}
        \toprule
        nodes & steps & base & finetune & reset\\
        \midrule
        9 & 10 & 97.95 & 98.43 & 98.67 \\
        16 & 8 & 98.23 & \textbf{98.81} & 98.66 \\
        32 & 8 & 98.35 & 98.61 & 98.81 \\
        32 & 10 & 98.65 & 98.52 & 98.71 \\
        64 & 20 & 98.68 & 98.63 & 98.72 \\
        \bottomrule
        \end{tabular}
        \caption{MNIST.}
        \label{tab:fine_mnist}
    \end{subtable}\\
    \vspace{4mm}
    \begin{subtable}{\linewidth}
        \centering
        \begin{tabular}{rrlll}
        \toprule
        nodes & steps & base & finetune & reset\\
        \midrule
        9 & 10 & 94.94 & 94.98 & 95.26\\
        16 & 6 & 94.95 & 95.09 & 95.47\\
        32 & 6 & 94.95 & 95.12 & \textbf{95.62}\\
        64 & 10 & 94.94 & 95.03 & 95.41\\
        \bottomrule
      \end{tabular}
        \caption{CIFAR10.}
        \label{tab:fine_cifar}
  \end{subtable}
    \caption{Results obtained for selected models from Table~3 in the main paper (``base'') and their finetuned versions. We analyze two types of finetuning, either by using basis weights and finetune all the parameters of the network (``finetune'') or by taking the graph structure from the base model, reset other network parameters, and train the network from scratch (``reset''). One can observe that there is no obvious winning strategy, and it should be considered a hyperparameter. Notice also that we bold the performance reported in the main paper.}
    \label{tab:fine}
\end{table}

\begin{table}[ht]
    \centering
    \begin{subtable}{\linewidth}
        \centering
        \begin{tabular}{lrrrrrr}
        \toprule
        & \multicolumn{6}{c}{steps}\\
        \cmidrule{2-7}
        nodes & 5 & 10 & 20 & 30 & 40 & 50\\
        \midrule
        25 & 52.65 & 63.45 & 62.90 & 63.85 & 67.65 & 68.55 \\
        32 & 53.65 & 62.65 & 72.90 & 73.30 & 73.20 & 73.55 \\
        64 & 57.95 & 74.00 & 78.70 & 79.70 & 82.95 & 82.95 \\
        128 & 57.00 & 73.85 & 79.60 & 83.05 & 84.45 & 85.75 \\
        511 & 48.75 & 72.35 & 81.60 & 82.50 & 84.05 & \textbf{86.25} \\
        \bottomrule
        \end{tabular}
        \caption{Letter.}
        \label{tab:shallow_letter}
    \end{subtable}\\
    \vspace{4mm}
    \begin{subtable}{\linewidth}
        \centering
        \begin{tabular}{lrrrrr}
        \toprule
        & \multicolumn{3}{c}{steps}\\
        \cmidrule{2-4}
        nodes & 2 & 5 & 10 \\
        \midrule
        2 & 77.47 & 77.40 & 77.50 \\
        8 & 75.37 & 79.60 & 80.27 \\
        16 & 75.47 & 80.31 & 81.55  \\
        32 & 75.36 & 80.45 & 82.65 \\
        255 & 75.43 & 80.43 & \textbf{82.82} \\
        \bottomrule
        \end{tabular}
        \caption{Connect4.}
        \label{tab:shallow_connect4}
    \end{subtable}\\
    \vspace{4mm}
    \addtolength{\tabcolsep}{-3pt}
    \begin{subtable}{\linewidth}\small
        \centering
        \begin{tabular}{lrrrrrrrr}
        \toprule
        & \multicolumn{8}{c}{steps}\\ \cmidrule{2-9}
        nodes & 4 & 6 & 8 & 10 & 20 & 30 & 40 & 50 \\
        \midrule
        9 & 87.58 & 88.68 & 88.52 & 88.93 & 89.36 & 90.48 & 90.36 & 90.40 \\
        16 & 90.74 & 91.73 & 93.06 & 93.09 & 93.42 & 92.97 & 93.39 & 93.37 \\
        32 & 88.80 & 91.47 & 93.22 & 93.56 & 94.38 & 93.67 & 93.72 & 93.56 \\
        64 & 86.35 & 92.77 & 93.33 & 93.41 & 94.66 & 94.29 & 94.86 & 94.55\\
        128 & 90.10 & 93.11 & 93.65 & 94.15 & 94.58 & 94.80 & 94.99 & 94.97 \\
        255 & 90.05 & 93.11 & 93.80 & 93.88 & 94.28 & 94.75 & 95.43 & \textbf{95.74} \\
        \bottomrule
      \end{tabular}
        \caption{MNIST.}
        \label{tab:shallow_mnist}
    \end{subtable}
    \addtolength{\tabcolsep}{1pt}    
    \caption{SONG as a shallow model (SONG-S). One can observe that the performance increases with the increasing number of nodes and steps for all datasets. We bold the performance reported in the main paper.}
    \label{tab:shallow}
\end{table}

\section{Transition matrices}

In Figures~\ref{fig:M2}-\ref{fig:M6}, we present sample matrices $M^0$ and $M^1$ before and after training. One can observe that at the beginning, there are weak connections between all nodes. However, trained matrices are almost binary and usually contain one value close to $1$ in each column, and all other values are close to $0$.

\begin{figure*}
    \centering
    \begin{subfigure}[b]{0.45\textwidth}
        \centering
        \includegraphics[width=\textwidth]{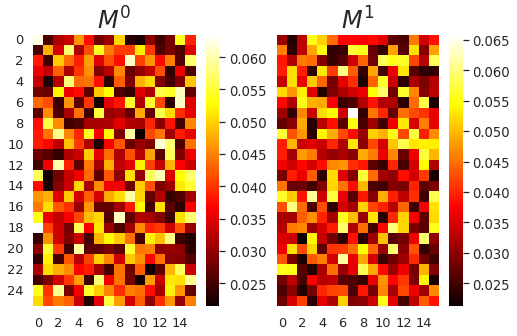}
        \caption{Initial values of $M^0$, $M^1$.}
    \end{subfigure}
    \qquad
    \begin{subfigure}[b]{0.45\textwidth}
        \centering
        \includegraphics[width=.97\textwidth]{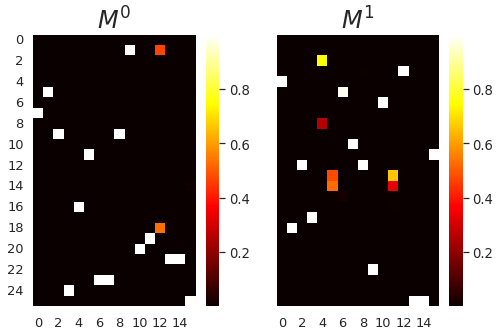}
        \caption{Trained values of $M^0$, $M^1$.}
    \end{subfigure}
    
    \caption{Sample matrices $M^0$ and $M^1$ of the SONG before and after training on the MNIST dataset with $16$ internal nodes.}
    \label{fig:M2}
\end{figure*}

\begin{figure*}
    \centering
    \begin{subfigure}[b]{0.45\textwidth}
        \centering
        \includegraphics[width=\textwidth]{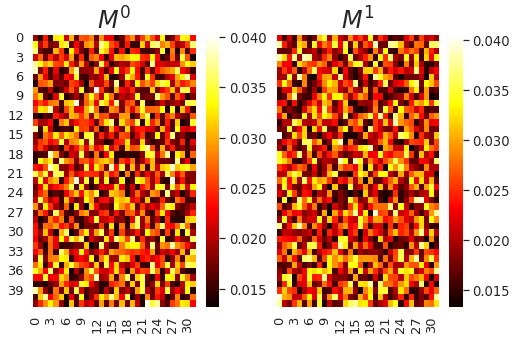}
        \caption{Initial values of $M^0$, $M^1$.}
    \end{subfigure}
    \qquad
    \begin{subfigure}[b]{0.45\textwidth}
        \centering
        \includegraphics[width=.97\textwidth]{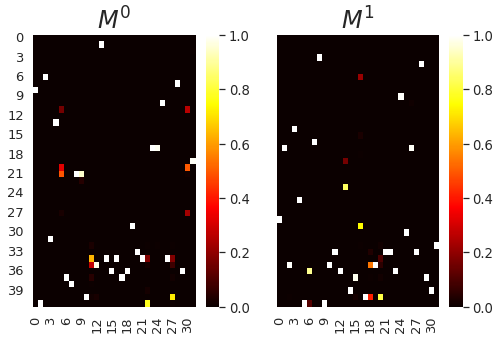}
        \caption{Trained values of $M^0$, $M^1$.}
    \end{subfigure}
    
    \caption{Sample matrices $M^0$ and $M^1$ of the SONG before and after training on the MNIST dataset with $32$ internal nodes.}
    \label{fig:M3}
\end{figure*}

\begin{figure*}
    \centering
    \begin{subfigure}[b]{0.45\textwidth}
        \centering
        \includegraphics[width=\textwidth]{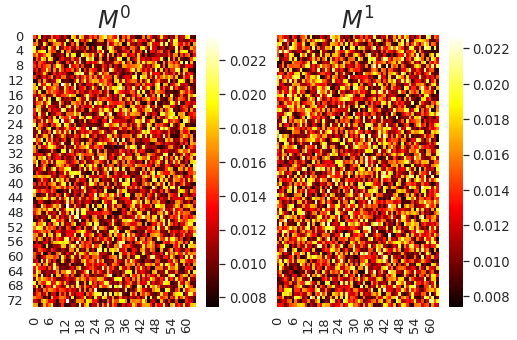}
        \caption{Initial values of $M^0$, $M^1$.}
    \end{subfigure}
    \qquad
    \begin{subfigure}[b]{0.45\textwidth}
        \centering
        \includegraphics[width=.97\textwidth]{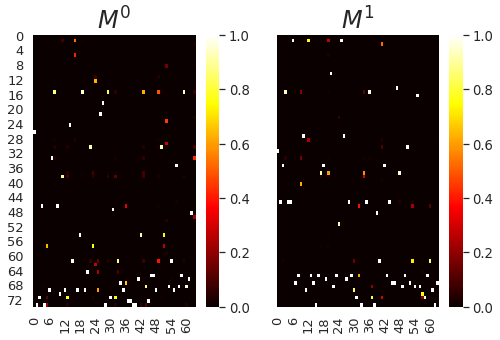}
        \caption{Trained values of $M^0$, $M^1$.}
    \end{subfigure}
    
    \caption{Sample matrices $M^0$ and $M^1$ of the SONG before and after training on the MNIST dataset with $64$ internal nodes.}
    \label{fig:M4}
\end{figure*}

\begin{figure*}
    \centering
    \begin{subfigure}[b]{0.45\textwidth}
        \centering
        \includegraphics[width=\textwidth]{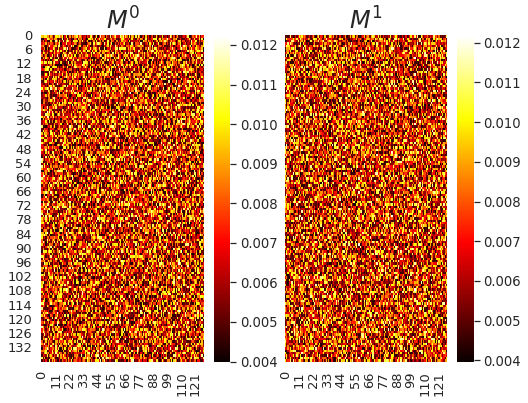}
        \caption{Initial values of $M^0$, $M^1$.}
    \end{subfigure}
    \qquad
    \begin{subfigure}[b]{0.45\textwidth}
        \centering
        \includegraphics[width=.97\textwidth]{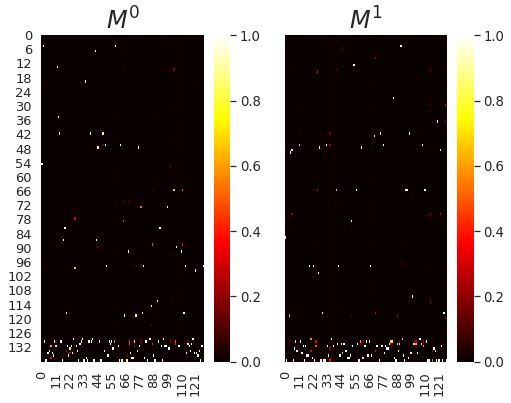}
        \caption{Trained values of $M^0$, $M^1$.}
    \end{subfigure}
    
    \caption{Sample matrices $M^0$ and $M^1$ of the SONG before and after training on the MNIST dataset with $128$ internal nodes.}
    \label{fig:M5}
\end{figure*}

\begin{figure*}
    \centering
    \begin{subfigure}[b]{0.46\textwidth}
        \centering
        \includegraphics[width=\textwidth]{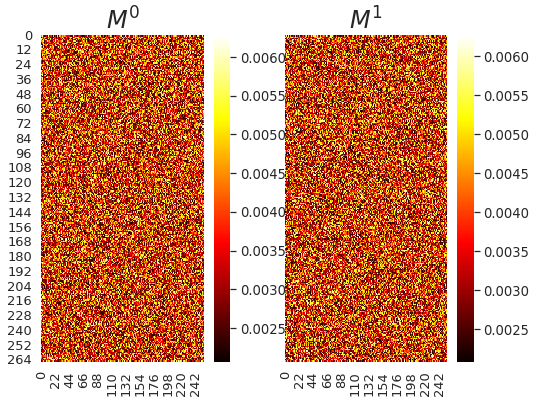}
        \caption{Initial values of $M^0$, $M^1$.}
    \end{subfigure}
    \qquad
    \begin{subfigure}[b]{0.45\textwidth}
        \centering
        \includegraphics[width=.97\textwidth]{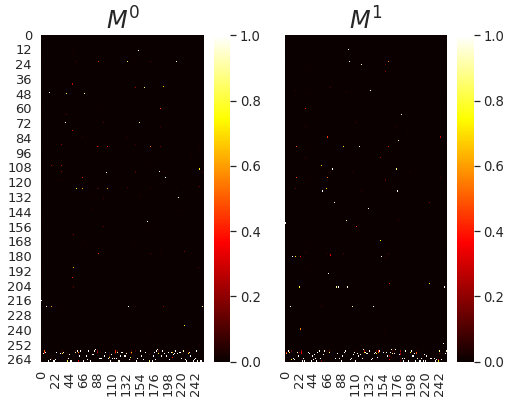}
        \caption{Trained values of $M^0$, $M^1$.}
    \end{subfigure}
    
    \caption{Sample matrices $M^0$ and $M^1$ of the SONG before and after training on the MNIST dataset with $256$ internal nodes.}
    \label{fig:M6}
\end{figure*}

\section{Experimental setup}
\label{sec:exp_setup}

We used the following datasets in our experiments:
\begin{itemize}
    \item Letter (\url{https://archive.ics.uci.edu/ml/datasets/Letter+Recognition}),
    \item Connect4 (\url{http://archive.ics.uci.edu/ml/datasets/connect-4}),
    \item MNIST  (published under CC BY-SA 3.0 license),
    \item CIFAR 10 \& CIFAR 100 (published under MIT license),
    \item TinyImageNet (\url{https://www.kaggle.com/c/tiny-imagenet/data}).
\end{itemize}

Moreover, we consider two types of setups, deep (SONG) and shallow (SONG-S). In SONG, we build neural networks that contain two successive parts, CNN and a graph. For the MNIST dataset, the CNN is built from two convolution layers with $8$ and $16$ filters of size $5\times 5$, each followed by ReLU and $2\times 2$ max pooling. Finally, a linear layer returns representation vectors of dimension 50. For other datasets (CIFAR10, CIFAR100, and TinyImageNet), we use model ResNet18 without the last linear layer. At the same time, for SONG-S, we only flatten the input sample to a one-dimensional vector.

For SONGs, we apply a similar experimental setup as in the state-of-the-art methods to have comparable results. More precisely, we take the previously trained ResNet18 network, remove its last layer, and use the remaining part as a CNN part. For the MNIST data, we train the first part directly using Binary Cross Entropy (BCE) loss. For the remaining datasets, we take a model from \url{github.com/alvinwan/neural-backed-decision-trees} (published under MIT license) trained with Cross Entropy (CE) and finetune it using BCE loss. During training the SONG, weights of CNNs are frozen. Moreover, the following hyper-parameters are considered in the grid-search:
\begin{itemize}
    \item For MNIST and CIFAR10:
    \begin{itemize}
        \item the number of nodes: 9, 16, 32, 64,
        \item the number of steps: 4, 6, 8, 10, 20.
    \end{itemize}
    \item For CIFAR100:
    \begin{itemize}
        \item the number of nodes: 99, 256, 512,
        \item the number of steps: 7, 12, 20, 40.
    \end{itemize}
    \item For TinyImagenet200:
    \begin{itemize}
        \item the number of nodes: 512,
        \item the number of steps: 20, 40.
    \end{itemize}
\end{itemize}
Additionally, we consider a batch size 64 or 128 and the learning rate $0.001$ for all datasets. Finally, when it comes to initialization, $\Mleft$, $\Mright$, and biases in nodes are initialized from a uniform distribution on the interval $[0, 1]$, and the remaining parameters (filters in the nodes) use the Kaiming initialization.

For SONG-S, the following hyper-parameters are considered in the grid-search:
\begin{itemize}
    \item For Letter dataset:
    \begin{itemize}
        \item the number of nodes: 25, 32, 64, 128, 511,
        \item the number of steps: 5, 10, 20, 30, 40, 50.
    \end{itemize}
    \item For Connect4 dataset:
    \begin{itemize}
        \item the number of nodes: 2, 8, 16, 32, 255,
        \item the number of steps: 2, 5, 10.
    \end{itemize}
    \item For MNIST dataset:
    \begin{itemize}
        \item the number of nodes: 9, 16, 32, 64, 128, 256,
        \item the number of steps: 4, 6, 8, 10, 20, 30, 40, 50.
    \end{itemize}
\end{itemize}
The remaining hyper-parameters are similar to the SONG setup.



\section{Computation time and resources}

We have run our experiments on Nvidia V100 32GB GPUs of our internal cluster. For deep setup, we trained 50, 50, 25, and 10 models for MNIST, CIFAR10, CIFAR100, and TinyImageNet, respectively. Each model required around 2, 2, 6, and 10 hours, respectively. For the shallow setup, we trained 60, 30, and 96 models for Letter, Connect4, and MNIST, respectively. In this case, each model required around 5, 2, and 2 hours, respectively.

\end{document}